\documentclass[11pt]{article}
\usepackage{bbm}
\usepackage{eqnarray,amsmath,amsfonts,amsthm,mathrsfs}
\usepackage{color}
\usepackage{bm}
\usepackage{amssymb}
\usepackage[dvips]{graphicx}
\usepackage{epsfig}
\usepackage{epsf}
\usepackage{float}
\usepackage{subfigure}
\usepackage{amsfonts,amsmath,amsthm,amssymb,graphicx,float,fancyhdr,multirow,hyperref}
\usepackage{booktabs,longtable,authblk}
\usepackage{mathrsfs,hhline}

\usepackage{fancyhdr}
\usepackage[top=2.5cm, bottom=2.5cm, left=3cm, right=3cm]{geometry}
\usepackage{graphicx}
\newtheorem{theorem}{Theorem}
\newtheorem{assumption}{Assumption}
\newtheorem{corollary}{Corollary}
\newtheorem{definition}{Definition}

\newtheorem{lemma}{Lemma}
\newtheorem{proposition}{Proposition}
\newtheorem{remark}{Remark}
\newtheorem{example}{Example}

\allowdisplaybreaks[4]

\def\begeqn{\begin{equation}}
\def\endeqn{\end{equation}}
\def\begth{\begin{theorem}}
\def\endth{\end{theorem}}
\def\begprop{\begin{proposition}}
\def\endprop{\end{proposition}}
\def\begcor{\begin{corollary}}
\def\endcor{\end{corollary}}
\def\begdef{\begin{definition}}
\def\enddef{\end{definition}}
\def\beglemm{\begin{lemma}}
\def\endlemm{\end{lemma}}
\def\begexm{\begin{example}}
\def\endexm{\end{example}}
\def\begrem{\begin{remark}}
\def\endrem{\end{remark}}
\def\begassum{\begin{assumption}}
\def\endassum{\end{assumption}}

\numberwithin{equation}{section}

\newcommand{\EE}{\mathbb E}
\newcommand{\PP}{\mathbb P}
\newcommand{\II}{\mathbb I}
\newcommand{\cN}{\mathcal N}
\newcommand{\cT}{\mathcal T}
\newcommand{\cL}{\mathcal L}
\newcommand{\cU}{\mathcal U}
\newcommand{\bX}{\textbf X}
\title{Statistical Optimality of Divide and Conquer Kernel-based Functional Linear Regression$^\dag$\footnotetext{\dag~The work described in this paper is supported by the National Natural Science Foundation of China [Grants Nos.12171039 and 12061160462] and Shanghai Science and Technology Program [Project No. 21JC1400600]. Email addresses: 20210180088@fudan.edu.cn (J. Liu), leishi@fudan.edu.cn (L. Shi).}}

\author{Jiading Liu}
\author{Lei Shi}
\affil{School of Mathematical Sciences and Shanghai Key Laboratory for
	Contemporary Applied Mathematics, Fudan University, Shanghai 200433,
	P.\ R.\ China.}

\date{}

\begin{document}
	\maketitle
\begin{abstract}
	Previous analysis of regularized functional linear regression in a reproducing kernel Hilbert space (RKHS) typically requires the target function to be contained in this kernel space. This paper studies the convergence performance of divide-and-conquer estimators in the scenario that the target function does not necessarily reside in the underlying RKHS. As a decomposition-based scalable approach, the divide-and-conquer estimators of functional linear regression can substantially reduce the algorithmic complexities in time and memory. We develop an integral operator approach to establish sharp finite sample upper bounds for prediction with divide-and-conquer estimators under various regularity conditions of explanatory variables and target function. We also prove the asymptotic optimality of the derived rates by building the mini-max lower bounds. Finally, we consider the convergence of noiseless estimators and show that the rates can be arbitrarily fast under mild conditions. 
\end{abstract}


\section{Introduction}\label{section: introduction}

Functional data analysis (FDA) has been an intense recent study, achieving remarkable success in a wide range of fields, including, among many others, chemometrics, linguistics, medicine and economics (see, e.g., \cite{ramsay2005functional,wang2016functional}). Under an FDA framework, the explanatory variable is usually a random function. We consider the following functional linear regression model to characterize the functional nature of explanatory variables. Let $Y$ be a scalar response, and $X$ be a random element taking values in ${\cal L}^2(\mathcal{T})$. Throughout the paper, we use ${\cal L}^2(\mathcal{T})$ to denote the Hilbert space of square integrable functions defined over a domain $\mathcal{T}\subseteq \mathbb{R}^D$ for some integer $D\geq 1$. In the functional linear regression model, the dependence of $Y$ and $X$ is expressed as 
\begin{equation}\label{LFRmodel}
	Y=	\int_{\mathcal{T}} \beta_0(t)X(t)dt + \epsilon,
\end{equation} where $\beta_0\in {\cal L}^2(\mathcal{T})$ is the slope function and $\epsilon$ is a random noise independent of $X$ with zero mean and bounded variance. The goal of functional linear regression is to construct an estimator $\hat{\beta}$ to approximate $\beta_0$ based on training samples of $(X,Y)$. The performance of an estimator can be measured by the prediction risk, given by
\begin{equation}\label{predicitonrisk}
	\mathcal{R}({\hat{\beta}}):=\mathbb{E}\left[\left(Y-\int_{\mathcal{T}} \hat{\beta}(t)X(t)dt\right)^2\right],
\end{equation} or equivalently, the excess prediction risk $\mathcal{R}(\hat{\beta})-\mathcal{R}(\beta_0)$.

The research on model \eqref{LFRmodel} can be traced back to the 1990s (see, for example, \cite{hastie1993statistical,marx1996generalized, cardot1999functional}). Subsequently,  a vast amount of literature has emerged to study the prediction and estimation problems under this model. A flourishing line of research is based on the functional principal component analysis (FPCA), leveraging spectral expansions of
the covariance kernel of $X$ and its empirical counterpart to estimate the slope function (see, e.g., \cite{yao2005functional,cai2006prediction,hall2007methodology,ramsay2005functional}). A necessary condition for the success
of the FPCA-based approaches is that the slope function $\beta_0$ can be efficiently represented by the leading functional principal components, which, however, fails to hold in many applications. To address this issue, another  influential line of research utilizes kernel-based estimators to approximate the target $\beta_0$ in a suitable reproducing kernel Hilbert space (RKHS) (see, e.g.,\cite{yuan2010reproducing,cai2012minimax}). More concretely, given a training sample set $S:=\{(X_i,Y_i)\}_{i=1}^N$ consisting of $N$ independent copies of $(X,Y)$, one can employ an RKHS $({\cal H}_K,\|\cdot\|_K)$ induced by a reproducing kernel $K:\mathcal{T}\times \mathcal{T}\to \mathbb{R}$ to estimate $\beta_0$ through the regularized least squares (RLS) estimators defined by  
\begin{equation}\label{totalestimator}
	\hat{\beta}_{S,\lambda}:=\mathop{\mathrm{argmin}}_{\beta \in {\cal H}_K}\left\{\frac{1}{N}\sum_{i=1}^N \left(Y_i-\int_{\mathcal{T}} \beta(t)X_i(t)dt\right)^2+\lambda \|\beta\|^2_K\right\}.
\end{equation} Here we choose a tuning parameter $\lambda>0$ to balance fidelity to the data and complexity of the estimators (measured by its squared ${\cal H}_K$ norm). According to the Representer Theorem proved in \cite{yuan2010reproducing}, $\hat{\beta}_{S,\lambda}$ can be uniquely expressed as $\hat{\beta}_{S,\lambda}(\cdot)=\sum_{i=1}^N c_i\int_{\mathcal{T}} K(\cdot,t)X_i(t)dt$ with $(c_1,\cdots,c_N)^T=\left(\lambda N \mathbb{I}_N + \mathbb{K}_{\bm X} \right)^{-1} {\bm Y}$, where $\mathbb{I}_N$ is the identity matrix on $\mathbb{R}^N$, $\mathbb{K}_{\bm X}\in \mathbb{R}^{N\times N}$ is the kernel matrix evaluated on ${\bm X}:=\{X_1,\cdots,X_N\}$ with the $(i,j)-$entrance $[\mathbb{K}_{\bm X}]_{i,j}=\int_{\mathcal{T}}X_i(s)K(s,t)X_j(t)dsdt$, and ${\bm Y}:=(Y_1,\cdots,Y_N)^T$. Under the assumption that the slope function $\beta_0$ belongs to the RKHS ${\cal H}_K$, it is shown in \cite{cai2012minimax} that the excess prediction risk of $\hat{\beta}_{S,\lambda}$ can achieve the mini-max optimal convergence rates.

In this paper, we aim to further advance the line of research on the kernel-based approach designed for functional linear regression model \eqref{LFRmodel}. Specially, we will study the convergence behavior of divided-and-conquer RLS estimators without requiring the unknown slope function $\beta_0$ to be contained in the RKHS ${\cal H}_K$. As a generalization of classical kernel ridge regression (see, e.g., \cite{murphy2012machine}), algorithm \eqref{totalestimator} suffers from the same complexity issue that seriously limits its performance when dealing with massive data. To make the computational problem more tractable for large-scale sample sets, we implement algorithm \eqref{totalestimator} via the divide-and-conquer approach. We randomly partition the entire sample set $S$ into $m$ disjoint equal-sized subsets $S_1,\cdots,S_m$. On each $S_j$, a local estimator $\hat{\beta}_{S_j,\lambda}$ is obtained according to algorithm \eqref{totalestimator}, i.e.,
\begin{equation*}
	\hat{\beta}_{S_j,\lambda}(\cdot)=\sum_{i:(X_i,Y_i)\in S_j} c_i\int_{\mathcal{T}} K(\cdot,t)X_i(t)dt\mbox{ where } (c_i)_{\{i:(X_i,Y_i)\in S_j\}}=(\lambda |S_j| \mathbb{I}_{|S_j|} + \mathbb{K}_{{\bm X}_j})^{-1} {\bm Y}_j.
\end{equation*}Here $|S_j|$ denotes the cardinality of $S_j$, ${\bm X}_j$ is the set of $X$'s sample in $S_j$, and  ${\bm Y}_j\in \mathbb{R}^{|S_j|}$ is a vector composed of $Y$'s sample in $S_j$. Divide-and-conquer RLS estimator is then computed by simply averaging $\{\hat{\beta}_{S_j,\lambda}\}_{j=1}^m$, which is given by
\begin{equation}\label{finalestimator}
	\overline{\beta}_{S,\lambda}:=\frac{1}{m}\sum_{j=1}^m \hat{\beta}_{S_j,\lambda}.
\end{equation} This approach is appealing due to its easy exercisable partitions. By partitioning the sample set into $m$ subsets of equal size and executing algorithm \eqref{totalestimator} on each subset concurrently, one can approximately diminish the computational complexities in terms of time and memory to $\frac{1}{m^2}$ of the initial requirements. In the context of regression analysis for massive data, divide-and-conquer kernel ridge regression and its variants have been extensively studied in statistics and machine learning communities (see, for instance, \cite{zhang2015divide,guo2017learning,lin2017distributed,dumpert2018universal,mucke2018parallelizing,hamm2021adaptive,sun2021optimal,hamm2022intrinsic,kohler2022total}). In the present paper, we evaluate the prediction performance of averaged estimator $\overline{\beta}_{S,\lambda}$ in \eqref{finalestimator} via its excess prediction risk:
\begin{equation}\label{equation: excess risk of beta(S,lambda)}
	\mathcal{R}(\overline{\beta}_{S,\lambda})- \mathcal{R}(\beta_0)
\end{equation} 
in a more general setting which allows $\beta_0\notin {\cal H}_K$. In supervised learning problem, if the target function resides outside the hypothesis space, e.g., the underlying RKHS in kernel-based regression, this scenario is often referred to as model misspecification (see, e.g., \cite{rao1971some,bach2008consistency}). More recently, convergence behaviors of kernel ridge regression in model misspecification scenarios have been investigated in many works (see, e.g., \cite{fischer2020sobolev,lin2020optimal,sun2021optimal}), which show asymptotically mini-max optimal rates in a variety of situations. In practice, canonical choices of ${\cal H}_K$ in \eqref{totalestimator} are the Sobolev spaces of smoothness $s$ (see \cite{yuan2010reproducing} and the references therein). Though such an RKHS is dense in ${\cal L}^2(\mathcal{T})$, the assumption that $\beta_0$ lies precisely in it is too restrictive in many real applications, as this assumption requires the derivatives of $\beta_0$ up to order $s-1$ are absolute continuous and its $s-$th derivative belongs to ${\cal L}^2(\mathcal{T})$. This raises the question of whether the global RLS estimator \eqref{totalestimator} and its averaged version \eqref{finalestimator} can still maintain excellent prediction performances in the model misspecification scenario $\beta_0\notin {\cal H}_K$. We positively answer this question by establishing a tight convergence analysis with an integral operator technique. Furthermore, we also consider the noiseless circumstance when the model (\ref{LFRmodel}) has no additive noise. The noiseless condition means no ambiguity of the response $Y$ given the explanatory variable $X$; in other words, the response $Y$ is determined uniquely by the input $X$. The noiseless linear model has been widely adopted in many areas, including image classification and sound recognition (see, e.g., \cite{jun2019kernel}). The convergence of estimators in a noiseless model is very important but has not been considered till the very recent papers (see, for example, \cite{jun2019kernel,berthier2020tight,sun2021optimal}).

The main contribution of this paper is to present new finite sample bounds on the prediction risk \eqref{equation: excess risk of beta(S,lambda)} concerning various regularity conditions. These conditions characterize the complexity of the prediction problem in functional linear regression model \eqref{LFRmodel}, which is measured through regularities of the explanatory variable $X$, optimum $\beta_0$, and their images under the kernel operators. See Section \ref{section: preliminaries} and Section \ref{section: main results} for precise definitions and statements.  Our analysis of convergence incorporates these regularity conditions into the integral operator techniques, substantially generalizing previously published bounds, which only consider the case $\beta_0\in \mathcal{H}_K$, to the model misspecification scenario $\beta_0\notin \mathcal{H}_K$ and the divide-and-conquer estimators. For prediction using the noisy model, the established convergence
is tight as in most cases we prove upper and lower bounds on the performance of estimators that almost match.  For prediction using the noiseless model, we prove that the estimators can converge with arbitrarily
fast polynomial rates if the reproducing kernel or the covariance kernel is sufficiently smooth. Thus the estimators show some adaptivity to the complexity of the prediction problem. Besides, our analysis only requires the the kernel function to be square integrable, eliminating the uniformly bounded or even continuity assumptions required in previous literature, which is more in line with the practical application scenarios of functional data analysis.

The rest of this paper is organized as follows. We start in Section \ref{section: preliminaries} with an introduction to notations, general assumptions, and some preliminary results. In Section \ref{section: main results}, we describe the regularity conditions and present main theorems and their corollaries. In Section \ref{section: comparison}, we give further comments on these regularity conditions and main results and compare them with other related contributions. All proofs can be found in Section \ref{section: convergence analysis} and the Appendix.

\section{Preliminaries}\label{section: preliminaries}

In this section we will provide basic notations and some preliminary results necessary for the further statement. We first recall some basic notations in operator theory (e.g., see \cite{conway2000course}). Consider a linear operator $A: \mathcal{H} \rightarrow \mathcal{H}'$, where both $(\mathcal{H}, \langle\cdot,\cdot\rangle_{\mathcal{H}})$ and $(\mathcal{H}', \langle\cdot,\cdot\rangle_{\mathcal{H}'})$ represent Hilbert spaces, equipped with their respective norms $\|\cdot\|_{\mathcal{H}}$ and $\|\cdot\|_{\mathcal{H}'}$. The collection of bounded linear operators from $\mathcal{H}$ to $\mathcal{H}'$ forms a Banach space when considered under operator norm $\|A\|_{\mathcal{H}, \mathcal{H}'} = \sup_{\|f\|_{\mathcal{H}}=1} \|Af\|_{\mathcal{H}'}$, symbolized as $\mathscr{B}(\mathcal{H}, \mathcal{H}')$ or $\mathscr{B}(\mathcal{H})$ in cases where $\mathcal{H} = \mathcal{H}'$. In scenarios where $\mathcal{H}$ and $\mathcal{H}'$ are implicitly understood, the subscript is omitted, simplifying the operator norm notation to $\|\cdot\|$. The adjoint of $A$, denoted by $A^*$, satisfies the equality $\langle Af, f'\rangle_{\mathcal{H}} = \langle f, A^*f'\rangle_{\mathcal{H}'}$ for all $f \in \mathcal{H}$ and $f' \in \mathcal{H}'$. If $A$ is an element of $\mathscr{B}(\mathcal{H}, \mathcal{H}')$, then its adjoint $A^*$ belongs to $\mathscr{B}(\mathcal{H}', \mathcal{H})$, and the norms of $A$ and $A^*$ are equivalent. An operator $A \in \mathscr{B}(\mathcal{H})$ is identified as self-adjoint if $A^* = A$ and as nonnegative if it is self-adjoint and satisfies $\langle Af, f\rangle_{\mathcal{H}} \geq 0$ for every $f \in \mathcal{H}$. The operator norm of a nonnegative operator $A\in \mathscr{B}({\cal H})$ has an equivalent expression:
\begin{equation}\label{equation: equivalent definition of nonnegative operator}
	\|A\|=\sup_{x\in \mathcal{H},\|x\|_{\mathcal{H}}=1}\langle Ax,x\rangle_{\mathcal{H}}.
\end{equation}
For $f\in {\cal H}$ and $f'\in {\cal H'}$, we introduce a rank-one operator $f\otimes f': {\cal H} \to {\cal H'}$ defined by $f \otimes f' (h)= \langle f, h \rangle_{\cal H} f', \forall h \in {\cal H}$. Consider $A$ to be a compact and nonnegative operator within $\mathscr{B}({\cal H})$. According to the Spectral Theorem, it is guaranteed that an orthonormal basis $\{e_k\}_{k\geq 1}$, composed of $A$'s eigenfunctions, exists within ${\cal H}$. This basis allows $A$ to be expressed as $A=\sum_{k\geq 1} \lambda_k e_k \otimes e_k$, where $\lambda_k$ denotes the non-negative eigenvalues of $A$ in a descending sequence. These eigenvalues (with geometric multiplicities) may either form a finite set or approach zero as $k$ increases indefinitely. Moreover, for any $r>0$, we define the $r-$th power of $A$ as $A^r=\sum_{k \geq 1} \lambda^r_k e_k \otimes e_k,$ which is itself a nonnegative compact operator on ${\cal H}$. An operator $A$ belonging to $\mathscr{B}(\mathcal{H}, \mathcal{H}')$ is identified as a Hilbert-Schmidt operator if, for a given orthonormal basis $\{e_k\}_{k\geq 1}$ of $\mathcal{H}$, the series $\sum_{k\geq 1}\|Ae_k\|^2_{\mathcal{H}'}$ converges. The collection of Hilbert-Schmidt operators constitutes a Hilbert space itself, equipped with the inner product defined by $\langle A, B\rangle_{HS} = \sum_{k\geq 1} \langle Ae_k, Be_k\rangle_{\mathcal{H}'}$, with $\|\cdot\|_{HS}$ representing the associated norm. In particular, a Hilbert-Schmidt operator $A$ is compact and we have the following inequality to relate its two different norms:
\begin{equation}\label{relationship between L2 and Linfty norm}
	\|A\| \leq \|A\|_{HS}.
\end{equation}  An operator $A \in \mathscr{B}({\cal H},{\cal H}')$ is of trace class if $\sum_{k\geq 1}\langle \left(A^{*}A\right)^{1/2}e_k,e_k \rangle_{{\cal H}}<\infty$ for some (any) orthonormal basis $\{e_k\}_{k\geq 1}$ of ${\cal H}$. All trace class operators constitute a Banach space endowed with the norm $trace(A):= \sum_{k\geq 1}\langle \left(A^{*}A\right)^{1/2}e_k,e_k \rangle_{{\cal H}}$. It is obviously for any nonnegative operator $A\in \mathscr{B}({\cal H})$,
\begin{equation}\label{equation: trace class}
	trace(A)=\sum_{k\geq1}\langle Ae_k,e_k \rangle_{\cal H}.
\end{equation}
In the following, we fix a reproducing kernel Hilbert space (RKHS), denoted as ${\cal H}_K$, consisting of functions $f:\mathcal{T} \to \mathbb{R}$, where each evaluation functional on this space is bounded. Consequently, there exists a distinct symmetric nonnegative definite kernel function $K:{\cal T}\times {\cal T} \to \mathbb{R}$, known as the reproducing kernel, intrinsically linked to ${\cal H}_K$. Let $K_t:\mathcal{T}\to \mathbb{R}$ defined by $K_t(\cdot)=K(\cdot,t)$ for $t\in {\cal T}$ and denote by $\langle \cdot, \cdot \rangle_K$ the inner product of ${\cal H}_K$ which induces the norm $\|\cdot\|_K$. Then $K_t \in {\cal H}_K$ and the reproducing property 
\begin{equation*}
	f(t)= \langle f, K_t \rangle_K
\end{equation*} holds for all $t\in {\cal T}$ and $f \in {\cal H}_K$. Furthermore, it is well known in the literature (refer to \cite{aronszajn1950theory}) that any symmetric positive definite kernel $K$ distinctly characterizes an RKHS, ${\cal H}_K$, for which $K$ serves as the reproducing kernel. Throughout the paper, we assume that $K$ is measurable on $\mathcal{T}\times \mathcal{T}$ such that 
\begin{equation*} 
	\int_{\mathcal{T}} \int_{\mathcal{T}} K^2(t,t') dt dt' <\infty.
\end{equation*} Recall that ${\cal L}^2(\cal{T})$ is the Hilbert space of functions from ${\cal T}$ to $\mathbb{R}$ square-integrable with respect to Lebesgue measure. Denote by $\|\cdot\|_{{\cal L}^2}$ the corresponding norm of ${\cal L}^2(\cal{T})$ induced by the inner product $\langle f, g\rangle_{{\cal L}^2}=\int_{{\cal T}} f(t) g(t) dt$. The reproducing kernel $K$ induces an integral operator $L_K:\mathcal{L}^{2}(\mathcal{T})\to \mathcal{L}^2(\mathcal{T})$, given by, for $f\in {\cal L}^2(\cal{T})$ and $t\in {\cal T}$,
\begin{equation*}
	L_K(f)(t)=\int_{\cal T} K(s,t)f(s)ds,
\end{equation*} which is a nonnegative, compact operator on ${\cal L}^2(\cal{T})$. Then $L^{1/2}_K$ is well-defined and compact, and $L^{1/2}_K$ is an isomorphism from $\overline{{\cal H}_K}$, the closure of ${\cal H}_K$ in ${\cal L}^2(\cal{T})$, to ${\cal H}_K$, i.e., for each $f\in \overline{{\cal H}_K}$, $L^{1/2}_K f \in {\cal H}_K$ and
\begin{equation}\label{normrelation2}
	\|f\|_{{\cal{L}}^2}=\|L^{1/2}_K f\|_K.
\end{equation} Since we are mainly interested in the model misspecification scenario $\beta_o\notin {\cal H}_K$, we will always assume ${\cal H}_K$ is dense in ${\cal L}^2(\cal T)$, i.e., ${\cal L}^2({\cal T})=\overline{{\cal H}_K}$. 

Besides the reproducing kernel $K$, another important kernel in our setting is the covariance kernel. Without loss of generality, we let the explanatory variable $X$ satisfy $\mathbb{E}[X]=0$ and further assume $\mathbb{E}\left[\|X\|^2_{{\cal L}^2}\right]<\infty$. Then the covariance kernel $C:{\cal T}\times {\cal T} \to \mathbb{R}$, given by $C(s,t):=\mathbb{E}\left[X(s)X(t)\right],\forall s,t\in {\cal T}$, can define a compact nonnegative integral operator $L_C:\mathcal{L}^{2}(\mathcal{T})\to \mathcal{L}^2(\mathcal{T})$ via
\begin{equation*}
	L_C(f)(t)=\int_\mathcal{T} C(s,t)f(s)ds,\quad\forall f\in {\cal L}^2(\mathcal{T})\mbox{ and }\forall t\in\mathcal{T}.
\end{equation*}

We next use $L_K$ and $L_C$ to give an expression of estimator $\hat{\beta}_{S,\lambda}$ in \eqref{totalestimator}. Recall that ${\cal L}^2({\cal T})=\overline{{\cal H}_K}$ and norms relation \eqref{normrelation2}. We can express $\hat{\beta}_{S,\lambda}$ as $\hat{\beta}_{S,\lambda}=L^{1/2}_K \hat{f}_{S,\lambda}$ with
\begin{equation*}
	\hat{f}_{S,\lambda}=\mathop{\mathrm{argmin}}_{f \in {\cal L}^2(\cal T)}\left\{\frac{1}{N}\sum_{i=1}^{N} \left(Y_i-\left\langle L^{1/2}_K f, X_i \right\rangle_{{\cal L}^2}\right)^2+\lambda \|f\|^2_{{\cal L}^2}\right\}.
\end{equation*} Following the proof of Theorem 6.2.1 in \cite{hsing2015theoretical}, we can solve $\hat{f}_{S,\lambda}$ explicitly and obtain the following proposition. 

\begin{proposition}\label{prop1}
	The estimator $\hat{\beta}_{S,\lambda}$ in (\ref{totalestimator}) can be expressed as
	$\hat{\beta}_{S,\lambda}=L^{1/2}_K \hat{f}_{S,\lambda}$ with
	\begin{equation}\label{fhat}
		\hat{f}_{S,\lambda} = \left(\lambda I+ T_{\bf X}\right)^{-1}\frac{1}{|S|}\sum_{(X_i,Y_i)\in S}L_K^{1/2}X_i Y_i,
	\end{equation} where $I$ denotes the identity operator on ${\cal L}^2(\mathcal{T})$, $|S|=N$ is the cardinality of $S=\{(X_i,Y_i)\}_{i=1}^N$, and $T_{\bf X}: {\cal L}^2({\cal T}) \to {\cal L}^2(\cal T)$ is an empirical operator with ${\bf X}=\{X_1,\cdots,X_N\}$ defined by
	\begin{equation}\label{Tx}
		T_{\bf X} = \frac{1}{|S|}\sum_{X_i \in {\bf X} }L_K^{1/2}X_i\otimes L_K^{1/2}X_i.
	\end{equation}
\end{proposition}

Recall that $S=\cup_{j=1}^m S_j$ with $S_j \cap S_k = \emptyset$ for $j\neq k$ and $|S_j|=\frac{N}{m}$. One can define the empirical operators $T_{{\bf X}_j}$ with ${\bf X}_j=\{X_i:(X_i,Y_i)\in S_j\}$ according to \eqref{Tx} and compute the local estimator $\hat{f}_{S_j,\lambda}$  as \eqref{fhat}, i.e., 
\begin{equation*}
	T_{{\bf X}_j} = \frac{1}{|S_j|}\sum_{X_i \in {\bf X}_j }L_K^{1/2}X_i\otimes L_K^{1/2}X_i
\end{equation*} and 
\begin{equation*}
	\hat{f}_{S_j,\lambda} = (\lambda I+ T_{{\bf X}_j})^{-1}\frac{1}{|S_j|}\sum_{(X_i,Y_i)\in S_j}L_K^{1/2}X_i Y_i.
\end{equation*} Then the averaged estimator $\overline{\beta}_{S,\lambda}$ in \eqref{finalestimator}  is given by $\overline{\beta}_{S,\lambda}=L^{1/2}_K\overline{f}_{S,\lambda}$ with $\overline{f}_{S,\lambda}:=\frac{1}{m}\sum_{j=1}^m \hat{f}_{S_j,\lambda}$.

To derive the upper bounds of excess prediction error, for any estimator $\hat{\beta}\in {\cal L}^2({\cal T})$, we rewrite $\mathcal{R(\hat{\beta})}-\mathcal{R}(\beta_0)$ as
\begin{equation}\label{excesspredictionerror}
	\mathcal{R(\hat{\beta})}-\mathcal{R}(\beta_0) = \mathbb{E}\left[\left\langle X,\hat{\beta}-\beta_0\right\rangle_{{\cal L}^2}^2\right] = \left\|L_C^{1/2}\left(\hat{\beta}-\beta_0\right)\right\|_{{\cal L}^2}^2.
\end{equation} Notice $T_{\bf X}$ and $\frac{1}{|S|}\sum_{(X_i,Y_i)\in S}L_K^{1/2}X_i Y_i$ are empirical versions of $L^{1/2}_KL_CL^{1/2}_K$ and $L^{1/2}_KL_C\beta_0$, respectively. We thus introduce intermediate function $f_\lambda:= \left(\lambda I+L_K^{1/2} L_C L_K^{1/2}\right)^{-1}L_K^{1/2}L_C\beta_0$ which can be expected to approximate $\hat{f}_{S,\lambda}$ and its averaged version $\overline{f}_{S,\lambda}$. According to \eqref{excesspredictionerror}, we then split $\mathcal{R(\overline{\beta}_{S,\lambda})}-\mathcal{R}(\beta_0)$ into two parts:
\begin{align}\label{decomposition 2}
	\mathcal{R(\overline{\beta}_{S,\lambda})}-\mathcal{R}(\beta_0)&=\left\|L_C^{1/2}\left(L^{1/2}_K\overline{f}_{S,\lambda}-L^{1/2}_Kf_{\lambda}+L^{1/2}_Kf_{\lambda}-\beta_0\right)\right\|_{{\cal L}^2}^2\nonumber\\
	&\leq 2\mathscr{S}(S,\lambda) + 2 \mathscr{A}(\lambda),
\end{align} where $\mathscr{S}(S,\lambda):=\left\|L^{1/2}_CL^{1/2}_K\overline{f}_{S,\lambda}-L^{1/2}_CL^{1/2}_Kf_{\lambda}\right\|^2_{{\cal L}^2}$ and $\mathscr{A}(\lambda):=\left\|L^{1/2}_CL^{1/2}_Kf_{\lambda}-L^{1/2}_C\beta_0\right\|^2_{{\cal L}^2}$. 

In Section \ref{section: main results}, we will describe the assumptions that are used to estimate $\mathscr{S}(S,\lambda)$ and $\mathscr{A}(\lambda)$, and then state the main results of this paper. Before that, we give a further characterization of the operators which are crucial in our estimation. For simplicity, let $T:=L_K^{1/2} L_C L_K^{1/2}$ and $T_*:=L_C^{1/2}L_K L_C^{1/2}$. Note that 
\begin{equation*}
	T= L_K^{1/2}L_C^{1/2}\left(L_K^{1/2}L_C^{1/2}\right)^{*} \mbox{ and } T_* = \left(L_K^{1/2}L_C^{1/2}\right)^{*}L_K^{1/2}L_C^{1/2}.
\end{equation*} 
Due to the compactness of $L^{1/2}_KL^{1/2}_C$, both $T$ and $T_*$ are compact and nonnegative operators on ${\cal L}^2({\cal T})$. Singular value decomposition of $L^{1/2}_KL^{1/2}_C$ (see, e.g., Theorem 4.3.1 in \cite{hsing2015theoretical}) leads to the following expansions:
\begin{equation}\label{singular value decomposition}
	\begin{aligned}
		L_K^{1/2}L_C^{1/2} &= \sum_{k\geq 1}\sqrt{\mu_k} \varphi_k \otimes \phi_k,\\
		L_C^{1/2}L_K^{1/2} &= \sum_{k\geq 1}\sqrt{\mu_k} \phi_k \otimes \varphi_k,\\
		T&= \sum_{k\geq 1}\mu_k \phi_k \otimes \phi_k,\\
		T_*&= \sum_{k\geq 1} \mu_k \varphi_k \otimes \varphi_k,
	\end{aligned}
\end{equation} 
where $\{\mu_k\}_{k\geq 1}$ is a non-negative, non-increasing and summable sequence, $\{\phi_k\}_{k\geq 1}$ and $\{\varphi_k\}_{k\geq 1}$ are two orthonormal bases of $L^2(\mathcal{T})$. Actually, for any $\mu_k>0$, $\sqrt{\mu_k}$ is the singular values of $L^{1/2}_KL^{1/2}_C$ and the corresponding left and right singular vectors are given by $\varphi_k$ and $\phi_k$, which are the eigenvectors (with the same eigenvalue $\mu_k$) of $T$ and $T_*$. In particular, the system $\{\mu_k,\phi_k, \varphi_k\}_{k\geq 1}$ plays an important role in describing regularities of the explanatory variable $X$ and the slope function $\beta_0$ which we will explain in details in Section \ref{section: main results}.

\section{Main Results}\label{section: main results}
In this section, we will present our main theoretical results on the upper and lower bounds of excess prediction risk of divide-and-conquer estimator \eqref{finalestimator} for the functional linear regression model \eqref{LFRmodel}. These main results are based on several key assumptions, including the regularity conditions of the slope function and the functional explanatory variable. We begin by establishing a min-max lower bound for excess prediction risks.

\subsection{Mini-max Convergence Lower Bound}
In this subsection, we first introduce assumptions on the slope function $\beta_0$ and the random noise $\epsilon$, then based on the two assumptions, we establish a min-max lower bound for excess prediction risks. We begin with the regularity assumption on the slope function $\beta_0$ which is expressed in terms of covariance operator $L_C$ and operator $T_*$ given in \eqref{singular value decomposition}.
\begin{assumption}[regularity condition of slope function]\label{assumption1}
	The slope function $\beta_0$ in functional linear regression model \eqref{LFRmodel} satisfies
	\begin{equation}\label{regularity condition}
		L_C^{1/2}\beta_0 = T_*^\theta(\gamma_0) \mbox{ with $0<\theta\leq 1/2$ and $\gamma_0\in \mathcal{L}^2(\cal T)$}.
	\end{equation}
	
\end{assumption} This assumption implies that $L^{1/2}_C\beta_0$ belongs to the range space of $T^{\theta}_*$ expressed as
$$\mathrm{ran}T^{\theta}_*:=\left\{f \in {\cal L}^2(\mathcal{T}): \sum_{k \geq 1} \frac{\langle f, \varphi_k\rangle^2_{{\cal L}^2}}{\mu_k^{2\theta}}<\infty \right\},$$  where $\{\mu_k,\varphi_k\}_{k\geq 1}$ is the eigensystem of $T_*$. Then $\mathrm{ran}T^{\theta_1}_* \subseteq \mathrm{ran}T^{\theta_2}_*$ whenever $\theta_1 \geq \theta_2$. The regularity of functions within the range of $T^{\theta}_*$ is determined by the rate at which their expansion coefficients decrease, utilizing the set $\{\varphi_k\}_{k \geq 1}$. The stipulation in condition \eqref{regularity condition} signifies that the square of the inner product $\langle L^{1/2}_C\beta_0, \varphi_k\rangle$ in the ${\cal L}^2$ space diminishes more swiftly than the eigenvalues of $T_*$ raised to the $2\theta$ power. A larger value of $\theta$ thus correlates with more rapid attenuation rates, signifying higher regularities of $L^{1/2}_C\beta_0$. In particular, for $\theta=0$ we have $\mathrm{ran}T^0_*={\cal L}^2(\cal T)$ implying $\beta_0\in {\cal L}^2({\cal T})$ and $\beta_0\in {\cal H}_K$ ensures regularity condition \eqref{regularity condition} is satisfied with $\theta=1/2$ as $\mathrm{ran}T^{1/2}_*=\mathrm{ran}L^{1/2}_CL^{1/2}_K$. From this point of view,  condition \eqref{regularity condition} allows $\beta_0 \notin \mathcal{H}_K$ which extends the previous regularity assumption in \cite{yuan2010reproducing,cai2012minimax}. This condition is also known as H\"{o}lder-type source condition involving the operator $T_*$, which is a classical smoothness assumption in the theory of inverse problems. Similar conditions defined by the operator $L_K$ are widely used in the literature of learning theory (see, e.g., \cite{bauer2007regularization,caponnetto2007optimal,smale2007learning,blanchard2018optimal}). We will provide more discussions on Assumption \ref{assumption1} in Section \ref{section: comparison}.

Throughout of the paper, we assume the following noise condition.

\begin{assumption}[noise condition]\label{assumption2}
	The random noise $\epsilon$ in functional linear regression model \eqref{LFRmodel} is independent of $X$ satisfying $\mathbb{E}[\epsilon]=0$ and $\mathbb{E}[\epsilon^2]\leq \sigma^2$.
\end{assumption}

Now under Assumption \ref{assumption1} and \ref{assumption2}, we propose the following theorem which establishes a mini-max lower bound for excess prediction risks. To this end, we also need to assume that $\{\mu_k\}_{k\geq 1}$, i.e., the eigenvalues of $T_*$ (and $T$), satisfy a polynomially decaying condition. For two positive sequences $\{a_k\}_{k\geq 1}$ and $\{b_k\}_{k\geq 1}$, we say $a_k \lesssim b_k$ holds if there exits a constant $c>0$ independent of $k$ such that $a_k \leq c b_k, \forall k\geq 1$. In addition, $a_k \asymp b_k$ if and only if  $a_k \lesssim b_k$ and $b_k \lesssim a_k$. For the sake of simplicity, we write $L^{1/2}_C\beta_0 \in \mathrm{ran}T^{\theta}_*$ if $\beta_0$ satisfies the regularity condition \eqref{regularity condition}.

\begin{theorem}[mini-max convergence lower bound]\label{theorem:lower bound}
	Under Assumption \ref{assumption1} with $0<\theta\leq 1/2$ and Assumption \ref{assumption2} with $\sigma>0$, suppose that $\{\mu_k\}_{k\geq 1}$ satisfy $\mu_k\asymp k^{-1/p}$ for some $0< p \leq 1$. Then excess prediction risks satisfy
	\begin{equation}
		\lim_{\gamma \to 0}\mathop{\inf\lim}_{N\to \infty} \inf_{\hat{\beta}_{S}} \sup_{\beta_0} \mathbb{P} \left\{\mathcal{R}(\hat{\beta}_S)-\mathcal{R}(\beta_0)\geq \gamma N^{-\frac{2\theta}{2\theta+p}}\right\} = 1,
	\end{equation} where the supremum is taken over all $\beta_0 \in \mathcal{L}^2(\cal T)$ satisfying $L^{1/2}_C\beta_0 \in \mathrm{ran}T^{\theta}_*$  and the infimum is taken over all possible predictors $\hat{\beta}_S \in \mathcal{L}^2(\cal T)$ based on the training sample set $S=\{(X_i,Y_i)\}_{i=1}^N$.
\end{theorem}
In Theorem \ref{theorem:lower bound} and subsequent statements, the case $p=1$ corresponds to the case in which we only require $\{\mu_k\}_{k\geq 1}$ to be summable. The lower bound for $p=1$ is also referred to as the capacity-independent optimum in some studies (for instance, see \cite{yao2007early}). This means that the bound is optimal in the mini-max sense without the necessity for a capacity hypothesis, i.e., without the decaying condition of the eigenvalues $\{\mu_k\}_{k\geq 1}$.

\subsection{Convergence Upper Bounds in Noised Case}
In this subsection, we will establish three different upper bounds on the excess prediction risk of divide-and-conquer estimator \eqref{finalestimator} under Assumption \ref{assumption1} and \ref{assumption2}. These upper bounds are based on three different regularity assumptions on the explanatory variable $X$, respectively. We first consider the upper bound on the convergence rate of the excess prediction risk \eqref{excesspredictionerror} and show that the convergence rate of the lower bound established in Theorem \ref{theorem:lower bound} can be achieved by the divide-and-conquer estimator \eqref{finalestimator}. The following assumption on the moment condition of the explanatory variable $X$ plays a crucial role in establishing the upper bound of the convergence rate of \eqref{equation: excess risk of beta(S,lambda)}.
\begin{assumption}[regularity condition of explanatory  variable \uppercase\expandafter{\romannumeral1}]\label{momentcondition1}
	There exists a constant $c_1>0$, such that for any $f\in \cL^2(\cT)$, 
	\begin{equation}\label{momentcondition2}
		\EE\left[\left<X,f\right>^4_{\cL^2}\right]\leq c_1\left[\EE \left<X,f\right>^2_{\cL^2}\right]^2.
	\end{equation}
\end{assumption} Assumption \ref{momentcondition1} has been introduced in \cite{cai2012minimax,yuan2010reproducing}. Condition \eqref{momentcondition2} articulates that the kurtosis of linear functionals applied to $X$ remains constrained, a criterion particularly satisfied with $c_1=3$ in scenarios where $X$ is modeled by a Gaussian process. For the convenience of further statements, define the effective dimension as
\begin{equation}\label{effectivedimension}
	\mathcal{N}(\lambda):= \sum_{k\geq 1}\frac{\mu_k}{\lambda + \mu_k},
\end{equation} where $\lambda >0$ and $\{\mu_k\}_{k\geq 1}$ are non-negative eigenvalues of $T$ (with geometric multiplicities) arranged in decreasing order. The effective dimension is widely used in the convergence analysis of kernel ridge regression ( see, for instance, \cite{caponnetto2007optimal,fischer2020sobolev,lin2017distributed,zhang2015divide}). Now under a polynomially decaying condition of eigenvalues $\{\mu_k\}_{k\geq 1}$, we can show in the following theorem that the convergence rate of the lower bound in Theorem \ref{theorem:lower bound} can be obtained by the divide-and-conquer RLS estimator $\overline{\beta}_{S,\lambda}=\frac{1}{m}\sum_{j=1}^m \hat{\beta}_{S_j,\lambda}$ with $S=\cup_{j=1}^m S_j=\{(X_i,Y_i)\}_{i=1}^N$ and $|S_j|=\frac{N}{m}$. We employ $o(\alpha_N)$ to denote a little-o sequence of $\{a_N\}_{N\geq 1}$ if $\lim_{N\to \infty} o(a_N)/a_N=0$.
\begin{theorem}[convergence upper bound \uppercase\expandafter{\romannumeral1}]\label{theorem:basic upper bound}
	Under Assumption \ref{assumption1} with $0<\theta\leq 1/2$, Assumption \ref{assumption2} with $\sigma>0$ and Assumption \ref{momentcondition1} with $c_1>0$, suppose that $\{\mu_k\}_{k\geq 1}$ satisfy $\mu_k\lesssim k^{-1/p}$ for some $0< p \leq 1$.
	\begin{enumerate}
		\item For $p/2<\theta\leq 1/2$, there holds
		\begin{equation}\label{extension of Yuan and Cai}
			\lim_{\Gamma\to\infty}\mathop{\sup\lim}_{N\to \infty}\sup_{\beta_0}\mathbb{P}\left\{\mathcal{R}(\overline{\beta}_{S,\lambda})-\mathcal{R}(\beta_0)\geq \Gamma N^{-\frac{2\theta}{2\theta+p}}\right\}=0
		\end{equation} provided that  $m \leq o \left(N^{\frac{2\theta - p}{4\theta + 2p}}\right)$ and $\lambda = N^{-\frac{1}{2\theta+p}}$.
		\item For $0<\theta\leq p/2$, there holds
		\begin{equation}\label{upperbound22}
			\lim_{\Gamma\to\infty}\mathop{\sup\lim}_{N\to \infty}\sup_{\beta_0}\mathbb{P}\left\{\mathcal{R}(\overline{\beta}_{S,\lambda})-\mathcal{R}(\beta_0)\geq \Gamma N^{-\frac{\theta}{p}}(\log N)^{\frac{3\theta r}{p}}\right\}=0
		\end{equation} provided that $m \leq (\log N)^r$ for some $r> 0$ and  $\lambda =N^{-\frac{1}{2p}}(\log N)^{\frac{3r}{2p}}$, and
		\begin{equation}\label{upperbound23}
			\lim_{\Gamma\to\infty}\mathop{\sup\lim}_{N\to \infty}\sup_{\beta_0}\mathbb{P}\left\{\mathcal{R}(\overline{\beta}_{S,\lambda})-\mathcal{R}(\beta_0)\geq \Gamma N^{-\frac{(1-r)\theta}{p}}\log N\right\}=0
		\end{equation} provided that $m \leq N^r$ for some $0\leq r <1$ and  $\lambda =  N^{-\frac{1-r}{2p}}(\log N)^{\frac{1}{2\theta}}$.
	\end{enumerate}
	Here the supremum is taken over all $\beta_0 \in \mathcal{L}^2(\cal T)$ satisfying $L^{1/2}_C\beta_0 \in \mathrm{ran}T^{\theta}_*$ with $0<\theta\leq 1/2$.
\end{theorem}

Actually, we will show that if the eigenvalue decay satisfies a polynomial upper bound of order $1/p$ with $0<p< 1$ and if the regularity parameter $\theta$ satisfies $0<\theta\leq p/2$, there holds
\begin{equation*}
	\lim_{\Gamma\to\infty}\mathop{\sup\lim}_{N\to \infty}\sup_{\beta_0}\mathbb{P}\left\{\mathcal{R}(\overline{\beta}_{S,\lambda})-\mathcal{R}(\beta_0)\geq \Gamma \lambda^{2\theta}\right\}=0
\end{equation*} provided that $m^2\lambda^{-2p}\leq o(N)$. From Theorem \ref{theorem:basic upper bound}, the bound \eqref{extension of Yuan and Cai} implies when $\theta\in(p/2,1/2]$, the excess prediction risk of $\overline{\beta}_{S,\lambda}$ attains the convergence rate of the lower
bound given by Theorem \ref{theorem:lower bound} and is therefore rate-optimal. Additionally, if $\theta=p/2$, from \eqref{upperbound22}, taking $m\leq (\log N)^r$ and  $\lambda =N^{-\frac{1}{2p}}(\log N)^{\frac{3r}{2p}}$ with some $r>0$ yields
\begin{equation*}
	\lim_{\Gamma\to\infty}\mathop{\sup\lim}_{N\to \infty}\sup_{\beta_0}\mathbb{P}\left\{\mathcal{R}(\overline{\beta}_{S,\lambda})-\mathcal{R}(\beta_0)\geq \Gamma N^{-\frac{1}{2}}(\log N)^{\frac{3r}{2}}\right\}=0.
\end{equation*} This convergence rate is also optimal up to a logarithmic factor. The bound \eqref{extension of Yuan and Cai} generalizes previous results of \cite{cai2012minimax}, which only considered the case $\beta_0\in \mathcal{H}_K$, to the model misspecification scenario $\beta_0\notin \mathcal{H}_K$ and the divide-and-conquer estimators. Actually, when $\theta=1/2$, taking $m=1$ and $\lambda=N^{-\frac{1}{2\theta+p}}$, we recovery Theorem 2 of \cite{cai2012minimax} which establishes minimax upper bound for the estimator $\hat{\beta}_{S,\lambda}$ in \eqref{totalestimator} when $\beta_0\in \mathcal{H}_K$.

We next introduce a higher-order moment condition on $X$ such that one can establish the strong convergence in expectation. To this end, given a reproducing kernel $K$, we shall introduce various regularities of the explanatory variable $X$ defined through its image under $L^{1/2}_K$. Recall that $X$ is a random element taking values in ${\cal L}^2(\mathcal{T})$ with $\mathbb{E}[X]=0$ and $\mathbb{E}\left[\|X\|^2_{{\cal L}^2}\right]<\infty$, and $\{\mu_k,\phi_k\}_{k\geq 1}$ is the eigensystem of $T$ given by \eqref{singular value decomposition}. Consider the principal component decomposition of $L^{1/2}_KX$ with respect to $T$ (see for details, \cite{ash2014topics}), which is expressed as 
\begin{equation}\label{PCdecomposition}
	L^{1/2}_K X = \sum_{k\geq 1} \sqrt{\mu_k} \xi_k \phi_k
\end{equation} where the $\xi_k$'s are zero-mean, uncorrelated real-valued random variables with $\mathbb{E}[\xi^2_k]=1$. We assume the following moment condition to characterize the regularity of $L^{1/2}_KX$.

\begin{assumption}[regularity condition of explanatory variable \uppercase\expandafter{\romannumeral2}]\label{momentcondition3}
	For some integer $\ell\geq 2$, there exists a constant $\rho<\infty$ such that $\{\xi_k\}_{k\geq 1}$ in decomposition \eqref{PCdecomposition} satisfy $\sup_{k\geq 1}\mathbb{E}[\xi_k^{4\ell}] \leq \rho^{4\ell}$. Moreover, there exists a constant $c_2 > 0$ such that 
	\begin{equation}\label{momentcondition31}
		\mathbb{E}\left[\langle X,f\rangle^8_{{\cal L}^2}\right]\leq c_2^2 \left[\EE \langle X,f\rangle_{{\cal L}^2}^2\right]^4, \quad \forall f\in {\cal L}^2(\mathcal{T}).
	\end{equation} 
\end{assumption} Since $\EE\left[\xi_k^2\right]=1$, we always have $\rho\geq 1$. When $X$ is a Gaussian random element taking value in ${\cal L}^2(\mathcal{T})$,  Assumption \ref{momentcondition3} is satisfied for any integer $\ell\geq 2$. In fact, given an integer $\ell\geq 2$, the linear functionals of a Gaussian random element $X$ satisfy
\begin{equation*}
	\mathbb{E}\left[\langle X,f\rangle^{4\ell}_{{\cal L}^2}\right]\leq (4\ell-1)!! \left[\EE \langle X,f\rangle_{{\cal L}^2}^2\right]^{2\ell}, \quad \forall f\in {\cal L}^2(\mathcal{T}).
\end{equation*} Then taking $f=L^{1/2}_K \phi_k$ implies Assumption \ref{momentcondition3} with $\rho = [(4\ell-1)!!]^{\frac{1}{4\ell}}$ and $c_2^2=105$ (by letting $\ell=2$). We need condition \eqref{momentcondition31} to bound $\EE\left[\langle X,\beta_0-L_K^{1/2}f_\lambda \rangle_{\cL^2}^4 \right]$ in the model misspecification scenario $\beta_0\notin \mathcal{H}_K$, which is crucial in the estimation of $$\mathscr{S}(S,\lambda)=\left\|L^{1/2}_CL^{1/2}_K\overline{f}_{S,\lambda}-L^{1/2}_CL^{1/2}_Kf_{\lambda}\right\|^2_{{\cal L}^2}.$$  Now we can establish the following upper bounds of \eqref{excesspredictionerror} in expectation under Assumption \ref{assumption1}, \ref{assumption2} and \ref{momentcondition3}.

\begin{theorem}[convergence upper bound \uppercase\expandafter{\romannumeral2}]\label{theorem:stronger upper bound}
	Suppose that Assumption \ref{assumption1} is satisfied with $0<\theta\leq 1/2$ and $\gamma_0\in \mathcal{L}^2(\cal T)$. Under Assumption \ref{assumption2} with $\sigma>0$ and Assumption \ref{momentcondition3} with some integer $\ell\geq 2$ and $c_2>0$, take $\lambda \leq 1$, then if $2\leq\ell<8$, there holds
	\begin{align}\label{equation: theorem stronger 1}
		&\EE\left[\left(\mathcal{R(\overline{\beta}_{S,\lambda})}-\mathcal{R}(\beta_0)\right)\right]\nonumber\\
		&\leq 2\lambda^{2\theta}\|\gamma_0\|_{\cL^2}^2 + 16\frac{\cN(\lambda)}{N}(c_2\lambda^{2\theta}\|\gamma_0\|^2_{\cL^2}+\sigma^2) + 8c_2\frac{m}{N}\cN(\lambda)\lambda^{2\theta}\|\gamma_0\|_{\cL^2}^2\nonumber\\
		&\quad + b_1(\ell)\lambda^{\frac{\ell-8}{4}}\left[1+ \left(\frac{m\cN^2(\lambda)}{N}\right)^{\frac{\ell}{8}}+\lambda^{-\frac{\ell}{4}}\left(\frac{m\cN(\lambda)}{N}\right)^{\frac{\ell}{4}}\right]\left(\frac{m\cN^2(\lambda)}{N}\right)^{\frac{\ell}{4}}\frac{4+2m}{N}\left(1+\lambda^{2\theta}\cN(\lambda)\right)\nonumber\\
		&\quad + b_2(\ell)\lambda^{\frac{\ell-8}{4}}\left[1+ \left(\frac{m\cN^2(\lambda)}{N}\right)^{\frac{\ell}{8}}+\lambda^{-\frac{\ell}{4}}\left(\frac{m\cN(\lambda)}{N}\right)^{\frac{\ell}{4}}\right]\left(\frac{m\cN^2(\lambda)}{N}\right)^{\frac{\ell}{4}}\frac{4\sigma^2}{N}\cN(\lambda).
	\end{align}
	If $\ell \geq 8$, there holds
	\begin{align}\label{equation: theorem stronger 2}
		&\EE\left[\left(\mathcal{R(\overline{\beta}_{S,\lambda})}-\mathcal{R}(\beta_0)\right)\right]\nonumber\\
		&\leq 2\lambda^{2\theta}\|\gamma_0\|_{\cL^2}^2 + 16\frac{\cN(\lambda)}{N}(c_2\lambda^{2\theta}\|\gamma_0\|^2_{\cL^2}+\sigma^2) + 8c_2\frac{m}{N}\cN(\lambda)\lambda^{2\theta}\|\gamma_0\|_{\cL^2}^2\nonumber\\
		&\quad + b_1(\ell)\left[1+ \frac{m\cN^2(\lambda)}{N}+\frac{1}{\lambda^{2}}\left(\frac{m\cN(\lambda)}{N}\right)^{2}\right]\left(\frac{m\cN^2(\lambda)}{N}\right)^{\frac{\ell}{4}}\frac{4+2m}{N}\left(1+\lambda^{2\theta}\cN(\lambda)\right)\nonumber\\
		&\quad + b_2(\ell)\left[1+ \frac{m\cN^2(\lambda)}{N}+\frac{1}{\lambda^{2}}\left(\frac{m\cN(\lambda)}{N}\right)^{2}\right]\left(\frac{m\cN^2(\lambda)}{N}\right)^{\frac{\ell}{4}}\frac{4\sigma^2}{N}\cN(\lambda).
	\end{align}
	Here $b_1(\ell)$ and $b_2(\ell)$ are constants only depending on $\ell$ and will be specified in the proof.
\end{theorem}

Recall that $a_N \lesssim b_N$ means that there exits a constant $c>0$ independent of $N$ such that $a_N \leq c b_N, \forall N\geq 1$. We obtain the following claims by using Theorem \ref{theorem:stronger upper bound}.

\begin{corollary}\label{corollary:corollary stronger}
	Under the assumptions of Theorem \ref{theorem:stronger upper bound}, suppose that $\{\mu_k\}_{k\geq 1}$ satisfy $\mu_k\lesssim k^{-1/p}$ for some $0< p \leq 1$. 
	\begin{enumerate}
		\item When $2\leq \ell \leq 4$, there holds
		\begin{align}\label{upperboundT3*}
			&\EE\left[\left(\mathcal{R(\overline{\beta}_{S,\lambda})}-\mathcal{R}(\beta_0)\right)\right]\nonumber\\
			&\lesssim \max\left\{N^{\frac{2\theta(4+\ell)(r-1)}{8+8\theta+2p\ell-\ell}},N^{\frac{2\theta\ell(r-1)-8\theta}{8+4p+8\theta+2p\ell-\ell}},N^{\frac{2\theta(4+2\ell)(r-1)}{8+8\theta+3p\ell}},N^{\frac{4\theta\ell(r-1)-8\theta}{8+4p+8\theta+3p\ell}}\right\}
		\end{align}
		provided that \[m\leq N^r \mbox{ for some $0\leq r\leq \frac{2\theta}{2\theta+p}$}\] and \[\lambda=\max\left\{N^{\frac{(4+\ell)(r-1)}{8+8\theta+2p\ell-\ell}},N^{\frac{\ell(r-1)-4}{8+4p+8\theta+2p\ell-\ell}},N^{\frac{(4+2\ell)(r-1)}{8+8\theta+3p\ell}},N^{\frac{2\ell(r-1)-4}{8+4p+8\theta+3p\ell}}\right\}.\]
		\item When $5\leq \ell \leq 7$, if $\frac{p\ell+8}{4\ell}\leq\theta\leq \frac{1}{2}$, then 
		\begin{equation}\label{upperboundT31}
			\mathbb{E}\left[\mathcal{R}(\overline{\beta}_{S,\lambda})-\mathcal{R}(\beta_0)\right] \lesssim N^{-\frac{2\theta}{2\theta+p}}
		\end{equation} provided that 
		\[m\leq \min\left\{N^{\frac{8+p\ell-4p-4\theta\ell}{(4+2\ell)(2\theta+p)}}, N^{\frac{8+p\ell-8\theta-4\theta\ell}{(4+2\ell)(2\theta+p)}}\right\}\] and 
		\[\lambda = N^{-\frac{1}{2\theta + p}};\]
		if $0<\theta<\frac{p\ell+8}{4\ell}$, then
		\begin{align}\label{upperboundT32}
			&\EE\left[\left(\mathcal{R(\overline{\beta}_{S,\lambda})}-\mathcal{R}(\beta_0)\right)\right]\nonumber\\
			&\lesssim \max\left\{N^{\frac{2\theta(4+\ell)(r-1)}{8+8\theta+2p\ell-\ell}},N^{\frac{2\theta\ell(r-1)-8\theta}{8+4p+8\theta+2p\ell-\ell}},N^{\frac{2\theta(4+2\ell)(r-1)}{8+8\theta+3p\ell}},N^{\frac{4\theta\ell(r-1)-8\theta}{8+4p+8\theta+3p\ell}}\right\}
		\end{align}
		provided that \[m\leq N^r \mbox{ for some $0\leq r\leq \frac{2\theta}{2\theta+p}$}\] and \[\lambda=\max\left\{N^{\frac{(4+\ell)(r-1)}{8+8\theta+2p\ell-\ell}},N^{\frac{\ell(r-1)-4}{8+4p+8\theta+2p\ell-\ell}},N^{\frac{(4+2\ell)(r-1)}{8+8\theta+3p\ell}},N^{\frac{2\ell(r-1)-4}{8+4p+8\theta+3p\ell}}\right\}.\]
		\item When $ \ell \geq 8$, if $\frac{p\ell+8}{2\ell + 16}\leq \theta\leq \frac{1}{2}$, then 
		\begin{equation}\label{upperboundT33}
			\mathbb{E}\left[\mathcal{R}(\overline{\beta}_{S,\lambda})-\mathcal{R}(\beta_0)\right] \lesssim N^{-\frac{2\theta}{2\theta+p}}
		\end{equation} provided that 
		\[m\leq \min\left\{N^{\frac{8+p\ell-4p-16\theta-2\theta\ell}{(12+\ell)(2\theta+p)}}, N^{\frac{8+p\ell-24\theta-2\theta\ell}{(12+\ell)(2\theta+p)}}\right\}\] and 
		\[\lambda = N^{-\frac{1}{2\theta + p}};\]
		if $0<\theta<\frac{p\ell+8}{2\ell + 16}$, then
		\begin{align}\label{upperboundT34}
			&\EE\left[\left(\mathcal{R(\overline{\beta}_{S,\lambda})}-\mathcal{R}(\beta_0)\right)\right]\nonumber\\
			&\lesssim \max\left\{N^{\frac{\theta(4+\ell)(r-1)}{4\theta+p\ell}},N^{\frac{\theta\ell(r-1)-4\theta}{2p+4\theta+p\ell}},N^{\frac{\theta(12+\ell)(r-1)}{4+4p+4\theta+p\ell}},N^{\frac{\theta(8+\ell)(r-1)-4\theta}{4+6p+4\theta+p\ell}}\right\}
		\end{align}
		provided that 
		\[m\leq N^r \mbox{ for some $0\leq r\leq \frac{2\theta}{2\theta+p}$}\]
		and
		\[\lambda = \max\left\{N^{\frac{(4+\ell)(r-1)}{8\theta+2p\ell}},N^{\frac{\ell(r-1)-4}{4p+8\theta+2p\ell}},N^{\frac{(12+\ell)(r-1)}{8+8p+8\theta+2p\ell}},N^{\frac{(8+\ell)(r-1)-4}{8+12p+8\theta+2p\ell}}\right\}.\]
	\end{enumerate}
\end{corollary}
According to Theorem \ref{theorem:lower bound}, the expectation bounds \eqref{upperboundT31} and \eqref{upperboundT33} are minimax optimal. Due to the well-known Markov's inequality, convergence in expectation given by Theorem \ref{theorem:stronger upper bound} and Corollary \ref{corollary:corollary stronger} is stronger, leading to bounds in a similar form as that of Theorem \ref{theorem:basic upper bound}. However, the possible ranges of $\theta$ that achieve the optimal rates in \eqref{upperboundT31} and \eqref{upperboundT33}, given respectively by
$[\frac{p\ell+8}{4\ell},1/2]$ and $[\frac{p\ell+8}{2\ell+16},1/2]$ both of which are covered by $(p/2,1/2]$, become smaller compared to the previous range of $\theta$ in the minimax bound \eqref{extension of Yuan and Cai}.  Moreover, we also observe from Corollary \ref{corollary:corollary stronger} that as the integer $\ell$ in Assumption \ref{momentcondition3} diverges to infinity, the possible ranges of $\theta$ that achieve the minimax expectation bounds will increase to $(p/2,1/2]$ which is exactly the range of $\theta$ leading to the minimax bound \eqref{extension of Yuan and Cai}. Motivated by this observation, we introduce another regularity condition on $X$ to establish optimal expectation error bounds for any $\theta \in (0,1/2]$. 

\begin{assumption}[regularity condition of explanatory variable \uppercase\expandafter{\romannumeral3}]\label{momentcondition5}
	There exists a constant $\rho<\infty$ such that $\{\xi_k\}_{k\geq 1}$ in decomposition \eqref{PCdecomposition} satisfy $\sup_{k\geq 1}|\xi_k| \leq\rho$ and the fourth-order moment condition \eqref{momentcondition2} is satisfied with $c_1>0$.
\end{assumption} 

One can verify that Assumption \ref{momentcondition5} holds true if the expansion of $L^{1/2}_K X$ in \eqref{PCdecomposition} is a summation of finite terms with each bounded $\xi_k$. Recall that the trace of operator $T$ is given by \begin{equation}\label{trace}
	trace(T):=\sum_{k\geq 1}\mu_k.
\end{equation}
Then we have the following theorem.
\begin{theorem}[convergence upper bound \uppercase\expandafter{\romannumeral3}]\label{theorem: extra upper bound}
	Suppose that Assumption \ref{assumption1} is satisfied with $0<\theta\leq 1/2$ and $\gamma_0\in \mathcal{L}^2(\cal T)$. Under Assumption \ref{assumption2} with $\sigma>0$ and Assumption \ref{momentcondition5} with $\rho, c_1>0$, take $\lambda \leq 1$, then there holds
	\begin{equation}\label{totalbound3}
		\begin{aligned}
			&\EE\left[\left(\mathcal{R(\overline{\beta}_{S,\lambda})}-\mathcal{R}(\beta_0)\right)\right]\\
			&\leq 2\lambda^{2\theta}\|\gamma_0\|_{\cL^2}^2 + 16\frac{\cN(\lambda)}{N}\left(c_1\lambda^{2\theta}\|\gamma_0\|^2_{\cL^2}+\sigma^2\right) + 8c_1\frac{m}{N}\cN(\lambda)\lambda^{2\theta}\|\gamma_0\|_{\cL^2}^2\\
			&\quad + c_3c_4\mu_1\frac{4+2m}{N\lambda^{2-2\theta}}\left(1+\frac{m\cN(\lambda)}{N}\right)\cN^{\frac{1}{2}}(\lambda)\exp\left(-\frac{c_5N}{2m\cN(\lambda)}\right)\\
			&\quad + c_4\mu_1\rho^2trace(T)\frac{4\sigma^2}{N\lambda^2}\left(1+\frac{m\cN(\lambda)}{N}\right)\cN^{\frac{1}{2}}(\lambda)\exp\left(-\frac{c_5N}{2m\cN(\lambda)}\right),
		\end{aligned}
	\end{equation}
	where $c_3$, $c_4$ and $c_5$ are universal constants which will be specified in the proof.
\end{theorem}

We further obtain a corollary of Theorem \ref{theorem: extra upper bound}.

\begin{corollary}\label{corollary:corollary extra}
	Under the assumptions of Theorem \ref{theorem: extra upper bound}, suppose that $\{\mu_k\}_{k\geq 1}$ satisfy $\mu_k\lesssim k^{-1/p}$ for some $0< p \leq 1$. Then there holds
	\begin{equation}\label{upperboundC31}
		\mathbb{E}\left[\mathcal{R}(\overline{\beta}_{S,\lambda})-\mathcal{R}(\beta_0)\right] \lesssim N^{-\frac{2\theta}{2\theta+p}}
	\end{equation}
	provided that $m \leq o\left(\frac{N^{\frac{2\theta}{2\theta + p}}}{\log N}\right)$ and $\lambda = N^{-\frac{1}{2\theta+p}}$.
\end{corollary}

As far as we know, the expectation bound \eqref{upperboundC31} establishes the first mini-max optimal rates for all possible $0<\theta \leq 1/2$. One can refer
to Section \ref{section: comparison} for more discussions.

\subsection{Convergence Upper Bounds in Noiseless Case}
In this subsection, we establish fast convergence rates of the excess prediction risk \eqref{excesspredictionerror} for noiseless functional linear model (i.e., $\epsilon=0$ in \eqref{LFRmodel}).

\begin{theorem}[convergence upper bound \uppercase\expandafter{\romannumeral4}]\label{theorem:no random noise}
	Under Assumption \ref{assumption1} with $0<\theta\leq 1/2$, Assumption \ref{assumption2} with $\sigma=0$ and Assumption \ref{momentcondition1} with $c_1>0$, suppose that $\{\mu_k\}_{k\geq 1}$ satisfy $\mu_k\lesssim k^{-1/p}$ for some $0<p\leq 1$. Then for any $0<\eta\leq 1/2$, there holds
	\begin{equation}\label{equation 32}
		\lim_{\Gamma\to\infty}\mathop{\sup\lim}_{N\to \infty}\sup_{\beta_0}\mathbb{P}\left\{\mathcal{R}(\overline{\beta}_{S,\lambda})-\mathcal{R}(\beta_0) \geq \Gamma N^{-\frac{\theta(1-2\eta)}{p}}\right\}=0
	\end{equation}
	provided that $m \leq o \left(N^{\eta}\right)$ and $\lambda = N^{-\frac{1-2\eta}{2p}}$, where the supremum is taken over all $\beta_0 \in \mathcal{L}^2(\cal T)$ satisfying $L^{1/2}_C\beta_0 \in \mathrm{ran}T^{\theta}_*$ with $0<\theta\leq 1/2$.
\end{theorem}

Follow from (\ref{equation 32}), given any $s> 0$ such that $0<p<2/s$ and  $sp<\theta\leq 1/2$, taking $0<\eta \leq \frac{1}{2}- \frac{sp}{2\theta}$, $m \leq o (N^{\eta})$ and $\lambda = N^{-\frac{1-2\eta}{2p}}$ yields 	
\begin{equation*}
	\begin{aligned}
		&\lim_{\Gamma\to\infty}\mathop{\sup\lim}_{N\to \infty}\sup_{\beta_0}\mathbb{P}\left\{\mathcal{R}(\overline{\beta}_{S,\lambda})-\mathcal{R}(\beta_0) \geq \Gamma N^{-s} \right\}\\
		&\leq\lim_{\Gamma\to\infty}\mathop{\sup\lim}_{N\to \infty}\sup_{\beta_0}\mathbb{P}\left\{\mathcal{R}(\overline{\beta}_{S,\lambda})-\mathcal{R}(\beta_0) \geq \Gamma N^{-\frac{\theta(1-2\eta)}{p}} \right\}\\
		&=0,
	\end{aligned}
\end{equation*}
where the inequality follows from $\frac{\theta(1-2\eta)}{p}\geq s$. The difference between noisy and noiseless models
is significant: rates faster than $N^{-1}$ for model \eqref{LFRmodel} are impossible with non-zero additive noise, while we prove that the divided-and-conquer RLS estimators for the noiseless model can converge with arbitrarily fast polynomial rates when $p$ is small enough. To our best knowledge, Theorem \ref{theorem:no random noise} and the related convergence rates \eqref{equation 32} are new to the literature, constituting another contribution of this paper. We will prove all these results in Section \ref{section: convergence analysis}.

\section{Discussions and Comparisons}\label{section: comparison}

In this section, we first comment on Assumption \ref{assumption1} and then compare our convergence analysis with some related results. In the last, we review recent literature for noiseless linear model and point out some possible directions for future study.
\subsection{Discussions on Assumption \ref{assumption1}}\label{subsection: discussions on Assumption 1}
Regularity condition \eqref{regularity condition} in Assumption \ref{assumption1} was first introduced by \cite{fan2019rkhs} and then adopted in the subsequent works (see, e.g., \cite{chen2022online}). From the discussion in Section \ref{section: main results}, we see that $\beta_0\in \mathcal{H}_K$ implies condition \eqref{regularity condition} is satisfied with $\theta=1/2$, while the former is equivalent to $\beta_0=L^{1/2}_K \gamma_0$ for some $\gamma_0 \in \mathcal{L}^2(\mathcal{T})$. Actually, due to Theorem 3 in \cite{chen2022online}, if $L_K \succeq \delta L^{\nu}_C$ for some $\delta>0$ and $\nu>0$, then for any $\beta_0 \in \mathcal{L}^2(\mathcal{T})$, there exists some $\gamma_0\in \mathcal{L}^{2}(\mathcal{T})$ such that condition \eqref{regularity condition} is satisfied with $\theta=1/(2+2\nu)$. Here for any bounded self-adjoint operators $A_1$ and $A_2$ on $L^2(\mathcal{T})$, we write $A_1 \succeq A_2$ if $A_1 - A_2$ is nonnegative. As a special case when $L_K$ and $L_C$ can be simultaneously diagonalized, let $\{\rho_k\}_{k\geq 1}$ and $\{\lambda_k\}_{k\geq 1}$ be eigenvalues of $L_K$ and $L_C$ respectively (both are sorted in decreasing order with geometric multiplicities). When $\rho_k \asymp k^{-1/\omega}$ with $\omega>1$ and $\lambda_k \asymp k^{-1/\tau}$ with $\tau>1$, then $\beta_0\in \mathrm{ran}L_K^{s}$ for some $s\in [0,1/2]$ implies condition \eqref{regularity condition} is satisfied with $\theta=(\omega+2s\tau)/(2\omega+2\tau)$, where $\mathrm{ran}L_K^{s}$ denotes the range space of $L^s_K$. When $K$ is an analytic kernel on $\mathcal{T}$, the eigenvalues of $L_K$ decay exponentially, and then condition \eqref{regularity condition} can be satisfied for $\theta$ arbitrarily close to $1/2$ (but still strictly less than $1/2$). From the discussions above, Assumption \ref{assumption1} is mild and provides an intrinsic measurement for the complexity of the prediction problem through the regularity condition \eqref{regularity condition}. Recently, under Assumption \ref{assumption1}, Assumption \ref{assumption2} and some boundedness condition on $K$ and $C$, the works \cite{chen2022online} and \cite{guo2023capacity} apply stochastic gradient descent to solve functional linear regression model \eqref{LFRmodel} and establish convergence rates for prediction and estimation errors.

\subsection{Comparisons with Relevant Results}
Convergence performance of kernel ridge regression and its variants in model misspecification scenario has been intensively studied by many recent works (see, for instance, \cite{pillaud2018statistical,shi2019distributed,fischer2020sobolev,lin2020optimalJMLR,lin2020optimal,sun2021optimal}). Among all available literature, the work \cite{fischer2020sobolev} obtained the best known convergence rates by applying the integral operator techniques combined with an embedding property (see condition (EMB) in \cite{fischer2020sobolev}). As far as we know, our paper is the first work to consider functional linear regression in a model misspecification scenario. To make a further comparison, we first introduce an embedding condition equivalent to the one in \cite{fischer2020sobolev} (i.e., condition \eqref{equation: condition substitutes Assumption 4} in this paper) under the functional linear regression setting. Then we apply this condition to derive convergence rates and compare them with related results in Section \ref{section: main results}. 

\begin{assumption}[regularity condition of explanatory variable \uppercase\expandafter{\romannumeral4}]\label{momentcondition6}
	There exist constants $\kappa>0$ and $0<t\leq 1$ such that $\{\xi_k\}_{k\geq 1}$ in decomposition \eqref{PCdecomposition} satisfy
	\begin{equation}\label{equation: condition substitutes Assumption 4}
		\sum_{k\geq 1} \mu_k^t\xi_k^2 \leq \kappa^2.
	\end{equation} Moreover, the fourth-order moment condition \eqref{momentcondition2} is satisfied with some $c_1>0$. 
\end{assumption} 

Condition \eqref{equation: condition substitutes Assumption 4} 
actually describes the $\mathcal{L}^{\infty}-$embedding property of $T^{(t-1)/2}L_K^{1/2}X$ for $0<t\leq 1$. Then we obtain the following result which also deserve attention in its own right.

\begin{theorem}[convergence upper bound \uppercase\expandafter{\romannumeral5}]\label{theorem: theorem for comparision}
	Under Assumption \ref{assumption1} with $0<\theta\leq 1/2$, Assumption \ref{assumption2} with $\sigma>0$ and Assumption \ref{momentcondition6} with $0<t\leq 1$,  suppose that $\{\mu_k\}_{k\geq 1}$ satisfy $\mu_k\lesssim k^{-1/p}$ for some $0< p \leq 1$. 
	\begin{enumerate}
		\item When $\max\{0, t/2-p/2\} < \theta\leq 1/2$, then
		\begin{equation}\label{bound6.1}
			\mathbb{E}\left[\mathcal{R}(\overline{\beta}_{S,\lambda})-\mathcal{R}(\beta_0)\right] \lesssim N^{-\frac{2\theta}{2\theta+p}}
		\end{equation}
		provided that 
		\[ m\leq o\left(\frac{N^{\frac{2\theta+p-t}{2\theta+p}}}{\log N}\right) \mbox{ and } \lambda = N^{-\frac{1}{2\theta+p}}.\]
		\item When $0<\theta \leq \max\{0, t/2-p/2\}$, then
		\begin{equation}\label{bound6.2}
			\mathbb{E}\left[\mathcal{R}(\overline{\beta}_{S,\lambda})-\mathcal{R}(\beta_0)\right] \lesssim N^{-\frac{2\theta}{t}}(\log N)^{-\frac{4\theta}{t}}
		\end{equation}
		provided that
		\[m\leq o(\log N)  \mbox{ and } \lambda= N^{-\frac{1}{t}}(\log N)^{-\frac{2}{t}}.\]
	\end{enumerate}
\end{theorem}

The proof of Theorem \ref{theorem: theorem for comparision} is also postponed to Section \ref{section: convergence analysis}. Condition \eqref{equation: condition substitutes Assumption 4} characterizes the regularity of $L^{1/2}_K  X$ through the parameter $t\in (0,1]$, of which the most general case is taking $t=1$, i.e., $\sum_{k=1}^{\infty}\mu_k\xi_k^2 \leq \kappa^2$, or equivalently,
\begin{equation}\label{equation: condition substitutes Assumption 4 weakest}
	\left\|L_K^{1/2}X\right\|_{\cL^2}\leq \kappa.
\end{equation}
When $t=1$, we obtain the mini-max rates in expectation for $\theta \in (1/2-p/2,1/2]$ from bound \eqref{bound6.1}. However, as $p\downarrow 0$, which implies the eigenvalues of $T_*$ decay even faster, the rate-optimal interval of $\theta$ is getting smaller. It seems unreasonable that higher regularity of $T_*$ could instead reduce the possible range of $\theta$ that leads to the optimal convergence. This phenomenon is widely observed in the convergence analysis of regularized kernel regression for the model misspecification scenario (see, e.g., \cite{shi2019distributed,fischer2020sobolev,lin2020optimalJMLR,lin2020optimal}). Note that verifying the embedding condition \eqref{equation: condition substitutes Assumption 4} for $t<1$ is highly non-trivial. This condition is automatically satisfied for all $0<t\leq 1$ if the expansion of $L^{1/2}_K X$ in \eqref{PCdecomposition} is a summation of finite terms with each bounded $\xi_k$. However, it is a wide-open question whether this condition holds for more general cases. It is also pointed out by \cite{fischer2020sobolev} that how to obtain the optimal rates for $t>p$ and $\theta \in (0,t/2-p/2]$ is an outstanding problem that can not be addressed by introducing the embedding condition. Comparing Assumption \ref{momentcondition6} to Assumption \ref{momentcondition3}  in Theorem \ref{theorem:stronger upper bound} and Corollary \ref{corollary:corollary stronger}, it is difficult to tell which one is more restrictive. But we think Assumption \ref{momentcondition3} is relatively more adaptive for functional linear regression model \eqref{LFRmodel}, since Assumption \ref{momentcondition6} excludes the most important case when $X$ is the Gaussian random element in $\cL^2(\cT)$. As we discussed in Section \ref{section: main results}, Gaussian random element satisfies Assumption \ref{momentcondition3} for any integer $\ell \geq 2$. Moreover, due to Corollary \ref{corollary:corollary stronger}, higher regularities indicated by larger $\ell$ in Assumption \ref{momentcondition3} or smaller $p$ in the eigenvalue decaying of $T_*$ will result in larger range of $\theta$ in which the estimators are rate-optimal. We believe that convergence analysis based on Assumption \ref{momentcondition3} is more insightful from this perspective. We then illustrate that in most cases, the index $p$ can be close to zero arbitrarily if one of the kernels $K$ and $C$ is smooth enough. To this end, we need the following lemma. 

\begin{lemma}\label{lemma: in comparison}
	Consider two nonnegative and compact operators $L_A$ and $L_B$ on a separable Hilbert space $\mathcal{H}$. Assume $\overline{\mathrm{ran}(L_A^{1/2})}= \mathcal{H}$, then we have 
	\[\rho_k(L_A^{1/2}L_BL_A^{1/2})\leq \rho_k(L_B)\|L_A\|,\]
	where the $\rho_k(L_A^{1/2}L_BL_A^{1/2})$ and $\rho_k(L_B)$ denote the $k$-th eigenvalue (sorted in decreasing order) of operators $L_A^{1/2}L_BL_A^{1/2}$ and $L_B$, respectively.
\end{lemma} We include its proof in the Appendix for the sake of completeness. Following from Lemma \ref{lemma: in comparison} with $\mathcal{H}=\mathcal{L}^2(\cal{T})$ and the fact that $\overline{{\cal H}_K}={\cal L}^2({\cal T})$, or equivalently $\overline{\mathrm{ran}(L_K^{1/2})}=\mathcal{L}^2(\cal{T})$,  we have $\mu_k = \rho_k(L_K^{1/2}L_CL_K^{1/2})\leq \rho_k(L_C)\|L_K\|\lesssim \rho_k(L_C)$. Moreover, if $\overline{\mathrm{ran}(L_C^{1/2})}=\mathcal{L}^2(\cal{T})$, one can deduce $\mu_k = \rho_k(L_C^{1/2}L_KL_C^{1/2})\leq \rho_k(L_K)\|L_C\|\lesssim \rho_k(L_K)$ by the same argument. For example, when $\mathcal{T}=\mathbb{R}$ and $K$ is the reproducing kernel of fractional Sobolev space $W^{\beta,2}(\mathbb{R})$ with $\beta >1/2$, we have $\mu_k \lesssim\rho_k(L_K)\asymp k^{-2\beta}$ and then $p\leq\frac{1}{2\beta}$ can arbitrarily approach zero if $K$ is smooth enough, i.e., $\beta$ is sufficiently large. Another notable example is that when $\mathcal{T}=[0,1]^D$ for some integer $D \geq 1$ and $X$ is a Gaussian random element in $\cL^2(\cT)$ with zero mean and covariance kernel $C_\gamma:\mathcal{T}\times \mathcal{T}\rightarrow \mathbb{R}$ (which is called a square-exponential kernel) defined by  $C_\gamma(x,x'):=\exp(-\frac{\|x-x'\|^2}{\gamma^2}),\forall x,x'\in \mathcal{T}$, i.e., $X\sim\cN(0,L_{C_\gamma})$. Here $\gamma>0$ is a constant and $L_{C_\gamma}$ denotes the covariance operator induced by $C_\gamma$. According to the existing literature about Gaussian process (see, for example, \cite{kanagawa2018gaussian}), we know that $\{\rho_k(L_{C_\gamma})\}_{k\geq 1}$ enjoys an exponential decay. For this case, we can prove that the divided-and-conquer RLS estimators are mini-max optimal for all possible $\theta \in (0,1/2]$ according to Corollary \ref{corollary:corollary stronger}.

We now compare Theorem \ref{theorem: theorem for comparision} with Corollary \ref{corollary:corollary extra} of Theorem \ref{theorem: extra upper bound}. Under the uniformly boundedness condition on $\{\xi_k\}_{k\geq 1}$ in Assumption \ref{momentcondition5}, we simplify the embedding condition \eqref{equation: condition substitutes Assumption 4} by only requiring the sequence $\{\mu^t_k\}_{k\geq 1}$ to be summable, i.e., $\sum_{k\geq 1} \mu^t_k<\infty$, which is satisfied for $t=p+\epsilon$  if $\mu_k\lesssim k^{-1/p}$. Here $\epsilon >0$ can be arbitrarily small. Therefore, under the same assumptions of Corollary \ref{corollary:corollary extra}, the first claims in Theorem \ref{momentcondition6} ensures that for all sufficiently small $\epsilon>0$ and $\epsilon/2<\theta\leq 1/2$, there holds  
\begin{equation*}
	\mathbb{E}\left[\mathcal{R}(\overline{\beta}_{S,\lambda})-\mathcal{R}(\beta_0)\right] \lesssim N^{-\frac{2\theta}{2\theta+p}}
\end{equation*}
with $\lambda = N^{-\frac{1}{2\theta+p}} \mbox{ and } m\leq o\left(\frac{N^{\frac{2\theta-\epsilon}{2\theta+p}}}{\log N}\right).$ Since one can choose an arbitrarily small $\epsilon>0$, the above result actually indicates the rate-optimal convergence for all $0<\theta\leq 1/2$. We see from Corollary \ref{corollary:corollary extra} in  Section \ref{section: main results} that, under Assumption \ref{momentcondition5}, one can obtain the same convergence result with a slightly better estimate on $m$ which only requires $ m\leq o\left(\frac{N^{\frac{2\theta}{2\theta+p}}}{\log N}\right)$.

When we finished this paper, we found that the work \cite{tong2021distributed} also studies the divide and conquer functional linear regression but under a regularity condition of slope function different from \eqref{regularity condition} which actually requires $\beta_0\in\mathcal{H}_K$, and a boundedness assumption equivalent to that \eqref{equation: condition substitutes Assumption 4} is satisfied with $t=1$. And we also note that to achieve optimal convergence rate under condition $\beta_0\in \mathcal{H}_K$, Theorem 2.1 in \cite{tong2021distributed} requires the number of partitions $m=1$, while with an additional assumption that the fourth-moment condition \eqref{momentcondition2} is satisfied, Theorem \ref{theorem: theorem for comparision} in this paper allows the number of partitions $m\leq o\left(\frac{N^{\frac{p}{1+p}}}{\log N}\right)$.

\subsection{Relevant Works on Noiseless Linear Model}
Recent works have intensively investigated the performance of various estimators within the context of a noiseless linear model. In typical learning tasks such as image classification, where human error is nearly impossible (e.g., misidentifying images of dogs as cats), it is reasonable to consider the output unambiguous for a given input. Consequently, many algorithms have incorporated noiseless models into these learning tasks. Theoretical analysis of noiseless models first appeared in studies on classification problems. For instance, \cite{smale2007learning} demonstrates that, compared to models with noise, the convergence analysis of binary classification problems in noiseless models may exhibit a phenomenon known as ``super convergence", where the convergence rate can be faster than $N^{-1}$. The studies by \cite{jun2019kernel,sun2021optimal}  explore the application of kernel-regularized least squares, revealing an improved rate of convergence for noiseless data relative to noisy scenarios. Furthermore, \cite{berthier2020tight} delves into the utilization of stochastic gradient descent for addressing the noiseless linear model in a general Hilbert space, albeit concentrating solely on scenarios where the optimal predictor resides within this space. As far as we know, the convergence of estimator in RKHS as well as its divide-and-conquer counterpart has not been considered in the context of noiseless functional linear model. We establish the first convergence result in this setting when the optimal predictor is outside of the underlying RKHS. The framework and estimations developed in this paper can be extended to study more complex models of nonparametric supervised learning, such as models in \cite{szabo2016learning,guo2021optimal,mao2022coefficient,meunier2022distribution}, which we leave as future work.

\section{Convergence Analysis}\label{section: convergence analysis}

In this section, we first derive the upper bounds of convergence rates presented in Theorem \ref{theorem:basic upper bound} and Theorem \ref{theorem:no random noise}. Then we establish the upper bounds in expectation presented in Theorem \ref{theorem:stronger upper bound}, Theorem \ref{theorem: extra upper bound}, Theorem \ref{theorem: theorem for comparision} and their corollaries. Last we prove the mini-max lower bound in Theorem \ref{theorem:lower bound}.

\subsection{Upper Rates and Upper Bounds}\label{subsection: upper bounds}

Recalling the decomposition \eqref{decomposition 2}, one can bound  $\mathcal{R(\overline{\beta}_{S,\lambda})}-\mathcal{R}(\beta_0)$ through estimating
$\mathscr{S}(S,\lambda)=\left\|L^{1/2}_CL^{1/2}_K\overline{f}_{S,\lambda}-L^{1/2}_CL^{1/2}_Kf_{\lambda}\right\|^2_{{\cal L}^2}$ and $\mathscr{A}(\lambda)=\left\|L^{1/2}_CL^{1/2}_Kf_{\lambda}-L^{1/2}_C\beta_0\right\|^2_{{\cal L}^2}$, respectively. We apply the following lemma to estimate  $\mathscr{A}(\lambda)$.

\begin{lemma}\label{lemma: estimation of A(lambda)}
	Suppose Assumption \ref{assumption1} is satisfied with $0<\theta\leq 1/2$ and $\gamma_0\in \cL^2(\cT)$. Then for any $\lambda > 0$, there holds
	\begin{equation}\label{equation: estimation of A(lambda)}
		\mathscr{A}(\lambda) \leq \lambda^{2\theta}\|\gamma_0\|_{\cL^2}^2.
	\end{equation}
\end{lemma}
\begin{proof}
	Write $\gamma_0 = \sum_{k\geq 1} a_k\varphi_k$, according to singular value decomposition of $T_*$ in (\ref{singular value decomposition}), we have  $L_C^{1/2}\beta_0 = T_*^\theta(\gamma_0) =  \sum_{k\geq 1}\mu_k^\theta a_k\varphi_k$ and
	\[L_C^{1/2}L_K^{1/2}f_\lambda = L_C^{1/2}L_K^{1/2}\left(\lambda I + T\right)^{-1}L_K^{1/2}L_C\beta_0 = \sum_{k=1}^{\infty}\frac{\mu_k^{1+\theta}}{\lambda + \mu_k}a_k\varphi_j.\]
	Therefore,
	\begin{align*}
		\mathscr{A}(\lambda) &=\left\|L_C^{1/2}(L_K^{1/2}f_\lambda - \beta_0)\right\|_{\cL^2}^2\\
		&= \sum_{k=1}^{\infty}\left(\frac{\mu_k^{1+\theta}}{\lambda + \mu_k} - \mu_j^\theta\right)^2a_j^2\\ &= \sum_{k=1}^{\infty} \frac{\lambda^2\mu_k^{2\theta}}{\left(\lambda + \mu_k\right)^2}a_k^2.
	\end{align*}
	While we see that for $0<\theta \leq 1/2$,
	\[\frac{t^{\theta}}{\lambda + t}\leq \theta^\theta (1-\theta)^{1-\theta} \lambda^{\theta - 1}\leq \lambda^{\theta - 1}, \quad \forall t>0, \]
	which implies that \[\mathscr{A}(\lambda) = \sum_{k=1}^{\infty} \frac{\lambda^2\mu_k^{2\theta}}{\left(\lambda + \mu_k\right)^2}a_k^2 \leq  \lambda^{2\theta}\sum_{k=1}^{\infty}a_k^2 = \lambda^{2\theta} \left\|\gamma_0\right\|_{\cL^2}^2.\]  The proof is then finished.
\end{proof}

In the rest part of this subsection, we focus on estimating $\mathscr{S}(S,\lambda)$. Recall that $S=\cup_{j=1}^m S_j$ with $S_j \cap S_k = \emptyset$ for $j\neq k$, the empirical operator $T_{{\bf X}_j}$ is defined with ${\bf X}_j=\{X_i:(X_i,Y_i)\in S_j\}$ according to \eqref{Tx}. For any $j=1,2,\dots,m$, define the event \[\cU_j = \left\{\bX_j: \left\|(\lambda I + T)^{-1/2}(T_{\bX_j}- T)(\lambda I + T)^{-1/2}\right\| \geq 1/2\right\},\]
and denote its complement by $\cU_j^c$. Let $\cU = \cup_{j=1}^m\cU_j$ be the union of above events. Then the complement of $\cU$ is given by $\cU^c=\cap_{j=1}^{m} \cU^c_j$. Hereafter,  let $\II_\mathcal{E}$ denote the indicator function of the event $\mathcal{E}$ and $\PP(\mathcal{E})=\mathbb{E} \left[\II_\mathcal{E}\right]$.  We first give the following estimation
\begin{equation}\label{equation: inverse norm under event}
	\begin{split}
		&\left\|(\lambda I + T)^{1/2}(\lambda I + T_{\bX_j})^{-1}(\lambda I + T)^{1/2}\right\|\II_{\cU_j^c}\\
		&= \left\|(I- (\lambda I + T)^{-1/2}(T-T_{\bX_j})(\lambda I + T)^{-1/2})^{-1}\right\|\II_{\cU_j^c}\\
		&\overset{(*)}{\leq} 1+\sum_{k=1}^{\infty}\left\|(\lambda I + T)^{-1/2}(T-T_{\bX_j})(\lambda I + T)^{-1/2}\right\|^{k}\II_{\cU_j^c}\\
		&\leq 1+\sum_{k=1}^{\infty} \frac{1}{2^k}=2,
	\end{split}
\end{equation}  where inequality $(*)$ follows by expanding the inverse in Neumann series.

The following lemma plays a crucial role in bounding the convergence rate of $\mathscr{S}(S,\lambda)$.

\begin{lemma}\label{lemma: rates upper bound 1}
	For any $m\geq 1$, there holds
	\begin{equation}\label{equation for theorem wup 1} 
		\begin{aligned}
			&\mathbb{E}\left[\mathscr{S}(S,\lambda)\II_{\cU^c}\right]\\
			&\quad\leq \frac{1}{m}\mathbb{E}\left[\left\|L_C^{1/2}L_K^{1/2}(\hat{f}_{S_1,\lambda}-f_\lambda)\right\|^2_{\cL^2}\II_{\cU_1^c}\right]+ \left\|L_C^{1/2}L_K^{1/2}\mathbb{E}\left[(\hat{f}_{S_1,\lambda}-f_\lambda)\II_{\cU_1^c}\right]\right\|^2_{\cL^2}.
		\end{aligned}
	\end{equation}
\end{lemma}
\begin{proof}
	When $m\geq 2$, as 
	$$\mathscr{S}(S,\lambda)=\left\|L^{1/2}_CL^{1/2}_K\overline{f}_{S,\lambda}-L^{1/2}_CL^{1/2}_Kf_{\lambda}\right\|^2_{{\cal L}^2}=\left\|\frac{1}{m}\sum_{i=1}^mL^{1/2}_CL^{1/2}_K\hat{f}_{S_i,\lambda}-L^{1/2}_CL^{1/2}_Kf_{\lambda}\right\|^2_{{\cal L}^2},$$
	then we have 
	\begin{align*}
		&\EE[\mathscr{S}(S,\lambda)\II_{\cU^c}]= \EE\left[\left\|L_C^{1/2}L_K^{1/2}\left(\frac{1}{m}\sum_{i=1}^{m}\hat{f}_{S_i,\lambda}-f_\lambda\right)\right\|^2_{\cL^2}\II_{\cU^c}\right]\\
		&\overset{(\romannumeral1)}{=}\frac{1}{m^2}\sum_{i=1}^{m}\EE\left[\left\|L_C^{1/2}L_K^{1/2}\left(\hat{f}_{S_i,\lambda}-f_\lambda\right)\right\|^2_{\cL^2}\II_{\cU^c}\right]\\
		&\quad + \frac{1}{m^2}\sum_{i\neq j} \EE\left[\left\langle L_C^{1/2}L_K^{1/2}\left(\hat{f}_{S_i,\lambda}-f_\lambda\right),L_C^{1/2}L_K^{1/2}\left(\hat{f}_{S_j,\lambda}-f_\lambda\right)\right\rangle_{\cL^2} \II_{\cU^c}\right]\\
		&\overset{(\romannumeral2)}{=} \frac{1}{m}\EE\left[\left\|L_C^{1/2}L_K^{1/2}\left(\hat{f}_{S_1,\lambda}-f_\lambda\right)\right\|^2_{\cL^2}\II_{\cU_1^c}\right]\PP\left(\cap_{j=2}^m\cU_j^c\right)\\
		&\quad+\frac{m(m-1)}{m^2}\EE\left[\left\langle L_C^{1/2}L_K^{1/2}\left(\hat{f}_{S_1,\lambda}-f_\lambda\right),L_C^{1/2}L_K^{1/2}\left(\hat{f}_{S_2,\lambda}-f_\lambda\right)\right\rangle_{\cL^2} \II_{\cU_1^c}\II_{\cU_2^c}\right]\PP\left(\cap_{j=3}^m \cU_j^c\right)\\
		&\overset{(\romannumeral3)}{\leq} \frac{1}{m}\mathbb{E}\left[\left\|L_C^{1/2}L_K^{1/2}\left(\hat{f}_{S_1,\lambda}-f_\lambda\right)\right\|^2_{\cL^2}\II_{\cU_1^c}\right]+ \left\|L_C^{1/2}L_K^{1/2}\mathbb{E}\left[\left(\hat{f}_{S_1,\lambda}-f_\lambda\right)\II_{\cU_1^c}\right]\right\|^2_{\cL^2}.
	\end{align*}
	Here equality (\romannumeral1) follows from the binomial expansion. Equality (\romannumeral2) uses the fact that $\II_{\cU^c}=\II_{\cU_1^c}\II_{\cU_2^c}\cdots\II_{\cU_m^c}$ and for any $1\leq i \neq j\leq m$, $\left(\hat{f}_{S_i,\lambda}-f_{\lambda}\right)\II_{\cU_i^c}$ and $\left(\hat{f}_{S_j,\lambda}-f_{\lambda}\right)\II_{\cU_j^c}$ are independent and identically distributed random elements. Inequality (\romannumeral3) is from \[\EE\left[\langle L_C^{1/2}L_K^{1/2}(\hat{f}_{S_1,\lambda}-f_\lambda),L_C^{1/2}L_K^{1/2}\left(\hat{f}_{S_2,\lambda}-f_\lambda\right)\rangle_{\cL^2} \II_{\cU_1^c}\II_{\cU_2^c}\right] = \left\|L_C^{1/2}L_K^{1/2}\mathbb{E}\left[\left(\hat{f}_{S_1,\lambda}-f_\lambda\right)\II_{\cU_1^c}\right]\right\|^2_{\cL^2}\]
	This completes the proof.
\end{proof}

For simplicity of notation, in the rest of this paper, we always denote \begin{equation}\label{equation: simple notation}
	n:= |S_1|= \frac{N}{m}  \qquad and \qquad \left\{\left(X_{1,i},Y_{1,i}\right)\right\}_{i=1}^n:= S_1.
\end{equation} We establish the following bounds on the right hand side of \eqref{equation for theorem wup 1} in Lemma \ref{lemma: rates upper bound 1}.

\begin{lemma}\label{lemma: rates upper bound 2}
	Suppose that Assumption \ref{assumption1} is satisfied with $0<\theta\leq 1/2$ and $\gamma_0\in \cL^2(\cT)$, Assumption \ref{assumption2} is satisfied with $\sigma> 0$ and Assumption \ref{momentcondition1} is satisfied with $c_1>0$. Then there hold
	\begin{equation}\label{equation: lemma rates ub 1}
		\mathbb{E}\left[\left\|L_C^{1/2}L_K^{1/2}\left(\hat{f}_{S_1,\lambda}-f_\lambda\right)\right\|^2_{\cL^2}\II_{\cU_1^c}\right]\leq 8\frac{m}{N}\cN(\lambda)\left(c_1\lambda^{2\theta}\left\|\gamma_0\right\|^2_{\cL^2}+\sigma^2\right)
	\end{equation}
	and
	\begin{equation}\label{equation: lemma rates ub 2}
		\left\|L_C^{1/2}L_K^{1/2}\mathbb{E}\left[\left(\hat{f}_{S_1,\lambda}-f_\lambda\right)\II_{\cU_1^c}\right]\right\|^2_{\cL^2}\leq 4c_1\frac{m}{N}\cN(\lambda)\lambda^{2\theta}\left\|\gamma_0\right\|^2_{\cL^2},
	\end{equation} where $\mathcal{N}(\lambda)$ is the effective dimension given by \eqref{effectivedimension}.
\end{lemma}

\begin{proof} We first prove the second inequality (\ref{equation: lemma rates ub 2}). Recalling \eqref{equation: simple notation}, we can write
	$$\hat{f}_{S_1,\lambda}=\left(\lambda I+T_{\bX_1}\right)^{-1}\frac{1}{n}\sum_{i=1}^{n}L_K^{1/2}X_{1,i}Y_{1,i}.$$
	Then
	\begin{align*}
		&\left\|L_C^{1/2}L_K^{1/2}\mathbb{E}\left[\left(\hat{f}_{S_1,\lambda}-f_\lambda\right)\II_{\cU_1^c}\right]\right\|^2_{\cL^2}\\
		&=\left\|L_C^{1/2}L_K^{1/2}\mathbb{E}\left[\left((\lambda I+T_{\bX_1})^{-1}\frac{1}{n}\sum_{i=1}^{n}L_K^{1/2}X_{1,i}Y_{1,i}-f_\lambda\right)\II_{\cU_1^c}\right]\right\|^2_{\cL^2}\\
		&\overset{(\romannumeral1)}{=}\left\|L_C^{1/2}L_K^{1/2}\mathbb{E}\left[(\lambda I+T_{\bX_1})^{-1}\left(\frac{1}{n}\sum_{i=1}^{n}L_K^{1/2}X_{1,i}\left\langle X_{1,i},\beta_0-L_K^{1/2}f_\lambda\right\rangle_{\cL^2}-\lambda f_\lambda\right)\II_{\cU_1^c}\right]\right\|^2_{\cL^2}\\
		&\overset{(\romannumeral2)}{\leq}\mathbb{E}\left[\left\|L_C^{1/2}L_K^{1/2}(\lambda I+T_{\bX_1})^{-1}\left(\frac{1}{n}\sum_{i=1}^{n}L_K^{1/2}X_{1,i}\left\langle X_{1,i},\beta_0-L_K^{1/2}f_\lambda\right\rangle_{\cL^2}-\lambda f_\lambda\right)\right\|^2_{\cL^2}\II_{\cU_1^c}\right]\\
		&\leq\left\|L_C^{1/2}L_K^{1/2}(\lambda I + T)^{-1/2}\right\|^2\left\|(\lambda I + T)^{1/2}(\lambda I + T_{\bX_1})^{-1}(\lambda I + T)^{1/2}\right\|^2\II_{\cU_1^c}\\
		&\quad\times\EE\left[\left\|(\lambda I + T)^{-1/2}\left(\frac{1}{n}\sum_{i=1}^{n}L_K^{1/2}X_{1,i}\left\langle X_{1,i},\beta_0-L_K^{1/2}f_\lambda\right\rangle_{\cL^2}-\lambda f_\lambda\right)\right\|_{\cL^2}^2\right]\\
		&\overset{(\romannumeral3)}{\leq} 4\EE\left[\left\|(\lambda I + T)^{-1/2}\left(\frac{1}{n}\sum_{i=1}^{n}L_K^{1/2}X_{1,i}\left\langle X_{1,i},\beta_0-L_K^{1/2}f_\lambda\right\rangle_{\cL^2}-\lambda f_\lambda\right)\right\|_{\cL^2}^2\right].
	\end{align*}
	Here equality (\romannumeral1) follows from the fact that $\epsilon$ is a centered random variable independent of $X$, inequality (\romannumeral2) uses Jensen's inequality, and inequality (\romannumeral3) is due to \eqref{equation: inverse norm under event} and the calculation that 
	\begin{align*}
		\left\|L_C^{1/2}L_K^{1/2}(\lambda I + T)^{-1/2}\right\|^2
		&=\left\|(\lambda I + T)^{-1/2}L_K^{1/2}L_CL_K^{1/2}(\lambda I + T)^{-1/2}\right\|\\
		&=\left\|(\lambda I + T)^{-1/2}T(\lambda I + T)^{-1/2}\right\|\leq 1.
	\end{align*}
	Note that for any $1\leq i\leq n$, $L_K^{1/2}X_{1,i}\left\langle X_{1,i},\beta_0-L_K^{1/2}f_\lambda\right\rangle_{\cL^2}-\lambda f_\lambda$ is a zero-mean random element. Then  we have
	\begin{align*}
		&\EE\left[\left\|(\lambda I + T)^{-1/2}\left(\frac{1}{n}\sum_{i=1}^{n}L_K^{1/2}X_{1,i}\left\langle X_{1,i},\beta_0-L_K^{1/2}f_\lambda\right\rangle_{\cL^2}-\lambda f_\lambda\right)\right\|_{\cL^2}^2\right]\\
		&= \frac{1}{n^2}\sum_{i=1}^{n}\EE\left[\left\|(\lambda I + T)^{-1/2}\left(L_K^{1/2}X_{1,i}\left\langle X_{1,i},\beta_0-L_K^{1/2}f_\lambda\right\rangle_{\cL^2}-\lambda f_\lambda\right)\right\|_{\cL^2}^2\right]\\
		&\leq \frac{1}{n^2}\sum_{i=1}^{n}\EE\left[\left\|(\lambda I + T)^{-1/2}L_K^{1/2}X_{1,i}\left\langle X_{1,i},\beta_0-L_K^{1/2}f_\lambda\right\rangle_{\cL^2}\right\|_{\cL^2}^2\right]\\
		&= \frac{1}{n^2}\sum_{i=1}^{n}\sum_{j=1}^\infty\EE\left[\left\langle(\lambda I + T)^{-1/2}L_K^{1/2}X_{1,i},\phi_{j}\right\rangle_{\cL^2}^2\left\langle X_{1,i},\beta_0-L_K^{1/2}f_\lambda\right\rangle_{\cL^2}^2\right]\\
		&\overset{(\romannumeral1)}{\leq} \frac{1}{n^2}\sum_{i=1}^{n}\sum_{j=1}^\infty\left[\EE\left\langle(\lambda I + T)^{-1/2}L_K^{1/2}X_{1,i},\phi_{j}\right\rangle_{\cL^2}^4\right]^{\frac{1}{2}}\left[\EE\left\langle X_{1,i},\beta_0-L_K^{1/2}f_\lambda\right\rangle_{\cL^2}^4\right]^{\frac{1}{2}}\\
		&\overset{(\romannumeral2)}{\leq} \frac{c_1}{n^2}\sum_{i=1}^{n}\sum_{j=1}^\infty\EE\left[\left\langle(\lambda I + T)^{-1/2}L_K^{1/2}X_{1,i},\phi_{j}\right\rangle_{\cL^2}^2\right]\EE\left[\left\langle X_{1,i},\beta_0-L_K^{1/2}f_\lambda\right\rangle_{\cL^2}^2\right]\\
		&=\frac{c_1}{n}\sum_{j=1}^\infty\frac{1}{\lambda+\mu_j}\left\langle T\phi_j,\phi_j\right\rangle_{\cL^2}\left\|L_C^{1/2}\left(\beta_0-L_K^{1/2}f_\lambda\right)\right\|_{\cL^2}^2\\
		&\overset{(\romannumeral3)}{=} \frac{c_1}{n}\cN(\lambda)\mathscr{A}(\lambda)\overset{(\romannumeral4)}{\leq} c_1\frac{m}{N}\cN(\lambda)\lambda^{2\theta}\|\gamma_0\|_{\cL^2}^2.
	\end{align*}
	Here $\{\phi_j\}_{j=1}^\infty$ is given by the singular value decomposition of $T$ in \eqref{singular value decomposition}. Inequality (\romannumeral1) uses Cauchy-Schwartz inequality. Inequality (\romannumeral2) is from Assumption \ref{momentcondition1}. Equality (\romannumeral3) follows from $\mathscr{A}(\lambda)=\left\|L_C^{1/2}\left(\beta_0-L_K^{1/2}f_\lambda\right)\right\|_{\cL^2}^2$ and the calculation that $\sum_{j=1}^\infty\frac{1}{\lambda+\mu_j}\langle T\phi_j,\phi_j\rangle_{\cL^2}=\sum_{j=1}^\infty\frac{\mu_j}{\lambda+\mu_j}=\cN(\lambda)$. Inequality (\romannumeral4) is due to Lemma \ref{lemma: estimation of A(lambda)} and $n=N/m$.
	Combining the above two estimations, we have completed the proof of \eqref{equation: lemma rates ub 2}.
	
	Next we prove the first inequality \eqref{equation: lemma rates ub 1}.  According to the expression of $\hat{f}_{S_1,\lambda}$ and the triangular inequality, we have
	\begin{align}
		&\mathbb{E}\left[\left\|L_C^{1/2}L_K^{1/2}\left(\hat{f}_{S_1,\lambda}-f_\lambda\right)\right\|^2_{\cL^2}\II_{\cU_1^c}\right]\notag\\
		&\leq 2\EE\left[\left\|L_C^{1/2}L_K^{1/2}\left(\lambda I+ T_{\bX_1}\right)^{-1}\left(\frac{1}{n}\sum_{i=1}^nL_K^{1/2}X_{1,i}\left\langle X_{1,i},\beta_0-L_K^{1/2}f_\lambda\right\rangle_{\cL^2}-\lambda f_\lambda\right) \right\|^2_{\cL^2}\II_{\cU_1^c}\right]\tag{\ref{equation: lemma rates ub 1}a}\\
		&\quad+ 2\EE\left[\left\|L_C^{1/2}L_K^{1/2}\left(\lambda I+ T_{\bX_1}\right)^{-1}\frac{1}{n}\sum_{i=1}^nL_K^{1/2}X_{1,i}\epsilon_{1,i}\right\|^2_{\cL^2}\II_{\cU_1^c}\right] \tag{\ref{equation: lemma rates ub 1}b},
	\end{align} where $\epsilon_{1,i}:=Y_{1,i}-\left\langle \beta_0, X_{1,i} \right\rangle_{\cL^2}$. We have bounded the term (\ref{equation: lemma rates ub 1}a) in the proof of (\ref{equation: lemma rates ub 2}), which is given by 
	\begin{equation}\label{equation: lemma rates ub 1a}
		\begin{aligned}
			&\EE\left[\left\|L_C^{1/2}L_K^{1/2}(\lambda I+ T_{\bX_1})^{-1}\left(\frac{1}{n}\sum_{i=1}^nL_K^{1/2}X_{1,i}\left\langle X_{1,i},\beta_0-L_K^{1/2}f_\lambda\right\rangle_{\cL^2}-\lambda f_\lambda\right) \right\|^2_{\cL^2}\II_{\cU_1^c}\right]\\
			&\leq 4c\frac{m}{N}\cN(\lambda)\lambda^{2\theta}\|\gamma_0\|_{\cL^2}^2.
		\end{aligned}
	\end{equation}
	Note that for any $1\leq i \leq n$, $L^{1/2}X_{1,i}\epsilon_{1,i}$ is also a zero-mean random element. Analogously, one can bound (\ref{equation: lemma rates ub 1}b) as 
	\begin{align*}
		&\EE\left[\left\|L_C^{1/2}L_K^{1/2}(\lambda I+ T_{\bX_1})^{-1}\frac{1}{n}\sum_{i=1}^nL_K^{1/2}X_{1,i}\epsilon_{1,i}\right\|^2_{\cL^2}\II_{\cU_1^c}\right]\\
		&\leq\left\|L_C^{1/2}L_K^{1/2}(\lambda I + T)^{-1/2}\right\|^2\left\|(\lambda I + T)^{1/2}(\lambda I + T_{\bX_1})^{-1}(\lambda I + T)^{1/2}\right\|^2\II_{\cU_1^c}\\
		&\quad \times\EE\left[\left\|(\lambda I + T)^{-1/2}\frac{1}{n}\sum_{i=1}^nL_K^{1/2}X_{1,i}\epsilon_{1,i}\right\|^2_{\cL^2}\right]\\
		&\leq 4\EE\left[\left\|(\lambda I + T)^{-1/2}\frac{1}{n}\sum_{i=1}^nL_K^{1/2}X_{1,i}\epsilon_{1,i}\right\|^2_{\cL^2}\right]= \frac{4}{n^2}\sum_{i=1}^{n}\EE\left[\left\|(\lambda I + T)^{-1/2}L_K^{1/2}X_{1,i}\right\|^2_{\cL^2}\epsilon_{1,i}^2\right]\\
		&\overset{(*)}{\leq}\frac{4\sigma^2}{n^2}\sum_{i=1}^{n}\sum_{j=1}^\infty\EE\left[\left\langle (\lambda I + T)^{-1/2}L_K^{1/2}X_{1,i},\phi_j\right\rangle_{\cL^2}^2 \right]=\frac{4\sigma^2}{n}\sum_{j=1}^\infty\frac{1}{\lambda+\mu_j}\langle T\phi_j,\phi_j\rangle_{\cL^2}=4\frac{m}{N}\cN(\lambda)\sigma^2.
	\end{align*}
	Here inequality $(*)$ is from Assumption \ref{assumption2}.
	
	Then we obtain inequality \eqref{equation: lemma rates ub 1} and the proof of Lemma \ref{lemma: rates upper bound 2} is finished. \end{proof}

We also need the following lemma to estimate the probability of event $\cU_1$. Recall that $\cU_1$ is defined as
\[\cU_1 = \left\{\bX_1: \left\|(\lambda I + T)^{-1/2}(T_{\bX_1}- T)(\lambda I + T)^{-1/2}\right\| \geq 1/2\right\}.\]

\begin{lemma}\label{lemma: basic probability esstimation of U1}
	Suppose that Assumption \ref{momentcondition1} is satisfied with $c_1>0$, then
	\begin{equation}\label{equation: basic probability esstimation of U1}
		\PP(\cU_1)\leq 4c_1\frac{m}{N}\cN^2(\lambda),
	\end{equation} where $\mathcal{N}(\lambda)$ is the effective dimension given by \eqref{effectivedimension}.
\end{lemma}
\begin{proof}
	Recall \eqref{equation: simple notation}. We first bound $\EE\left[\left\|(\lambda I + T)^{-1/2}(T_{\bX_1}- T)(\lambda I + T)^{-1/2}\right\|^2\right]$ as
	\begin{align*}
		&\EE\left[\left\|(\lambda I + T)^{-1/2}\left(T_{\bX_1}- T\right)(\lambda I + T)^{-1/2}\right\|^2\right]\\
		&\overset{(\romannumeral1)}{\leq} \EE\left[\left\|(\lambda I + T)^{-1/2}\left(T_{\bX_1}- T\right)(\lambda I + T)^{-1/2}\right\|_{HS}^2\right]\\
		&= \sum_{j=1}^{\infty}\sum_{k=1}^{\infty}\EE\left[\left\langle(\lambda I + T)^{-1/2}\left(\frac{1}{n}\sum_{i=1}^{n}L_K^{1/2}X_{1,i}\otimes L_K^{1/2}X_{1,i} - T\right)(\lambda I + T)^{-1/2}\phi_j,\phi_k\right\rangle^2_{\cL^2} \right]\\
		&\overset{(\romannumeral2)}{\leq} \frac{1}{n^2}\sum_{i=1}^n\sum_{j=1}^{\infty}\sum_{k=1}^{\infty}\frac{1}{\lambda + \mu_j}\frac{1}{\lambda + \mu_k}\EE\left[\left\langle L_K^{1/2}X_{1,i},\phi_j\right\rangle^2_{\cL^2}\left\langle L_K^{1/2}X_{1,i},\phi_k\right\rangle^2_{\cL^2}\right]\\
		&\overset{(\romannumeral3)}{\leq} \frac{1}{n^2}\sum_{i=1}^n\sum_{j=1}^{\infty}\sum_{k=1}^{\infty}\frac{1}{\lambda + \mu_j}\frac{1}{\lambda + \mu_k}\left[\EE\left\langle L_K^{1/2}X_{1,i},\phi_j\right\rangle^4_{\cL^2}\right]^{\frac{1}{2}}\left[\EE\left\langle L_K^{1/2}X_{1,i},\phi_k\right\rangle^4_{\cL^2}\right]^{\frac{1}{2}}\\
		&\overset{(\romannumeral4)}{\leq} \frac{c_1}{n^2}\sum_{i=1}^n\sum_{j=1}^{\infty}\sum_{k=1}^{\infty}\frac{1}{\lambda + \mu_j}\frac{1}{\lambda + \mu_k}\EE\left[\left\langle L_K^{1/2}X_{1,i},\phi_j\right\rangle^2_{\cL^2}\right]\EE\left[\left\langle L_K^{1/2}X_{1,i},\phi_k\right\rangle^2_{\cL^2}\right]\\
		&=\frac{c_1}{n}\sum_{j=1}^{\infty}\sum_{k=1}^{\infty}\frac{1}{\lambda + \mu_j}\frac{1}{\lambda + \mu_k}\left\langle T\phi_j,\phi_j \right\rangle_{\cL^2}\left\langle T\phi_k,\phi_k \right\rangle_{\cL^2}\\
		&=\frac{c_1}{n}\sum_{j=1}^{\infty}\sum_{k=1}^{\infty}\frac{\mu_j}{\lambda + \mu_j}\frac{\mu_k}{\lambda + \mu_k}=c_1\frac{m}{N}\cN^2(\lambda).
	\end{align*}
	Here $\{\phi_j\}_{j=1}^\infty$ is given by the singular value decomposition of $T$ in \eqref{singular value decomposition}. Inequality (\romannumeral1) follows from \eqref{relationship between L2 and Linfty norm}. Inequality (\romannumeral2) is from the fact that for any $1\leq i\leq n$, $L_K^{1/2}X_i\otimes L_K^{1/2}X_i - T$ is a zero-mean random element. Inequality (\romannumeral3) uses Cauchy-Schwartz inequality. Inequality (\romannumeral4) applies Assumption \ref{momentcondition1}.
	
	Combining the above estimation with Chebyshev inequality, we obtain
	\[\begin{aligned}
		\PP(\cU_1)&= \PP\left(\left\{\bX_1: \left\|(\lambda I + T)^{-1/2}\left(T_{\bX_1}- T\right)(\lambda I + T)^{-1/2}\right\|\geq  1/2\right\}\right)\\
		&\leq 4\EE\left[\left\|(\lambda I + T)^{-1/2}\left(T_{\bX_1}- T\right)(\lambda I + T)^{-1/2}\right\|^2\right]\\
		&\leq 4c_1\frac{m}{N}\cN^2(\lambda).
	\end{aligned}\]
	Then we obtain the desired result and complete the proof.
\end{proof}

The following lemma provides an estimation of $\cN(\lambda)$ under the polynomial decaying condition of the eigenvalues.
\begin{lemma}\label{lemma: estimation of N(lambda)}
	Suppose that $\{\mu_k\}_{k\geq1}$ satisfy $\mu_k\lesssim k^{-1/p}$ for some $0<p\leq 1$, then there holds
	\begin{equation}\label{equation: estimation of N(lambda)}
		\cN(\lambda) \lesssim \lambda^{-p}, \quad \forall 0<\lambda\leq 1.
	\end{equation} 
\end{lemma}
The estimation in Lemma \ref{lemma: estimation of N(lambda)} can be found in \cite{guo2017learning,lin2017distributed,guo2019optimal}.

We have established preliminary estimations for Theorem \ref{theorem:basic upper bound} and Theorem \ref{theorem:no random noise}. We are in the position
to prove these two theorems. To this end, we also need to introduce the notations $o_{_\PP}(\cdot)$ and $\mathcal{O}_{\PP}(\cdot)$. For a sequence of random variables $\{\xi_k\}_{k=1}^\infty$, we write $\xi_k\leq o_{_\PP}(1)$ if \begin{align*}
	\lim_{k\rightarrow\infty}\PP\left(\left|\xi_k\right|\geq d\right)=0, \forall d>0.
\end{align*} And we write $\xi_k\leq \mathcal{O}_{\PP}(1)$ if \begin{align*}
	\lim_{D\rightarrow \infty}\sup_{k\geq 1}\PP\left(\left|\xi_k\right|\geq D\right)=0. 
\end{align*}
In addition, suppose that there is a positive sequence $\{a_k\}_{k=1}^\infty$. Then we write 
$\xi_k\leq o_{_\PP}(a_k)$ if $\xi_k/a_k \leq o_{_\PP}(1)$, and $\xi_k\leq \mathcal{O}_{\PP}(a_k)$ if $\xi_k/a_k \leq \mathcal{O}_{\PP}(1)$.

\noindent
{\bf Proof of Theorem \ref{theorem:basic upper bound}}. Combining the decomposition \eqref{decomposition 2} and \eqref{equation: estimation of A(lambda)} in Lemma \ref{lemma: estimation of A(lambda)} yields
\begin{align}\label{equation: Theorem 2 1}
	\mathcal{R(\overline{\beta}_{S,\lambda})}-\mathcal{R}(\beta_0)
	&\leq 2\mathscr{S}(S,\lambda) + 2\mathscr{A}(\lambda)\nonumber\\
	&\leq 2\mathscr{S}(S,\lambda) + 2\lambda^{2\theta}\|\gamma_0\|_{\cL^2}^2.
\end{align}

We first decompose $\mathscr{S}(S,\lambda)$ as
\begin{align}\label{proof of theorem 3 1}
	\mathscr{S}(S,\lambda)= \mathscr{S}(S,\lambda)\II_{\cU} + \mathscr{S}(S,\lambda)\II_{\cU^c}.
\end{align}
For the term $\mathscr{S}(S,\lambda)\II_{\cU}$, following from \eqref{equation: basic probability esstimation of U1} in Lemma \ref{lemma: basic probability esstimation of U1}, we have
\[\EE\left[\II_{\cU}\right]=\PP(\cU)\leq \sum_{j=1}^m\PP(\cU_j) = m\PP(\cU_1)\leq 4c_1\frac{m^2}{N}\cN^2(\lambda).\]
Then using Markov's inequality, we can write
\begin{align}\label{proof of theorem 3 2}
	\mathscr{S}(S,\lambda)\II_{\cU}\leq \mathcal{O}_{\PP}\left(\frac{m^2}{N}\cN^2(\lambda)\right)\mathscr{S}(S,\lambda).
\end{align}

For the term $\mathscr{S}(S,\lambda)\II_{\cU^c}$, combining \eqref{equation for theorem wup 1} in Lemma \ref{lemma: rates upper bound 1} with \eqref{equation: lemma rates ub 1} and \eqref{equation: lemma rates ub 2} in Lemma \ref{lemma: rates upper bound 2} yields
\begin{align*}
	\EE\left[\mathscr{S}(S,\lambda)\II_{\cU^c}\right]\leq 8\frac{\cN(\lambda)}{N}\left(c_1\lambda^{2\theta}\|\gamma_0\|_{\cL^2}^2+\sigma^2\right) + 4c_1\frac{m}{N}\cN(\lambda)\lambda^{2\theta}\|\gamma_0\|_{\cL^2}^2.
\end{align*} 
Then using Markov's inequality, we can write
\begin{align}\label{proof of theorem 3 3}
	\mathscr{S}(S,\lambda)\II_{\cU^c}\leq \mathcal{O}_{\PP}\left(\frac{\cN(\lambda)}{N}+\frac{m}{N}\cN(\lambda)\lambda^{2\theta}\right).
\end{align}

Therefore, combining \eqref{proof of theorem 3 1}, \eqref{proof of theorem 3 2} and \eqref{proof of theorem 3 3}, we have
\begin{align*}
	\left[1-\mathcal{O}_{\PP}\left(\frac{m^2}{N}\cN^2(\lambda)\right)\right]\mathscr{S}(S,\lambda)\leq \mathcal{O}_{\PP}\left(\frac{\cN(\lambda)}{N}+\frac{m}{N}\cN(\lambda)\lambda^{2\theta}\right).
\end{align*}
Then applying the estimation of $\cN(\lambda)$ \eqref{equation: estimation of N(lambda)} in Lemma \ref{lemma: estimation of N(lambda)}, taking $\lambda\leq 1$, we can write
\begin{align}\label{equation: Theorem 2 2}
	\left[1-\mathcal{O}_{\PP}\left(\frac{m^2}{N}\lambda^{-2p}\right)\right]\mathscr{S}(S,\lambda)\leq \mathcal{O}_{\PP}\left(\frac{\lambda^{-p}}{N}+\frac{m}{N}\lambda^{2\theta-p}\right).
\end{align}

Take $m$ and $\lambda$ satisfying $m^2\lambda^{-2p}\leq o(N)$ and $\lambda\leq 1$, then \eqref{equation: Theorem 2 2} implies that 
\begin{align*}
	\left[1-o_{_\PP}(1)\right]\mathscr{S}(S,\lambda)\leq \mathcal{O}_{\PP}\left(\frac{\lambda^{-p}}{N}+\frac{m}{N}\lambda^{2\theta-p}\right),\mbox{ as } \mathcal{O}_{\PP}\left(\frac{m^2}{N}\lambda^{-2p}\right)\leq o_{_\PP}(1).
\end{align*}
Thus, we can write
\begin{align*}
	\mathscr{S}(S,\lambda)\leq \mathcal{O}_{\PP}\left(\frac{\lambda^{-p}}{N}+\frac{m}{N}\lambda^{2\theta-p}\right).
\end{align*}
Combining the above estimation with \eqref{equation: Theorem 2 1} yields
\begin{align}\label{equation: Theorem 2 3}
	\mathcal{R(\overline{\beta}_{S,\lambda})}-\mathcal{R}(\beta_0)\leq \mathcal{O}_{\PP}\left(\lambda^{2\theta}+\frac{\lambda^{-p}}{N}+\frac{m}{N}\lambda^{2\theta-p}\right)
\end{align}
provided that 
$$m^2\lambda^{-2p}\leq o(N) \mbox{ and } \lambda\leq 1.$$

When $p/2<\theta\leq 1/2$, take $m\leq o\left(N^{\frac{2\theta-p}{4\theta+2p}}\right)$ and $\lambda=N^{-\frac{1}{2\theta+p}}$, then there hold $m^2\lambda^{-2p}\leq o(N)$ and $\lambda\leq 1$. Therefore, following from \eqref{equation: Theorem 2 3}, we can write
\begin{align*}
	\mathcal{R(\overline{\beta}_{S,\lambda})}-\mathcal{R}(\beta_0)\leq\mathcal{O}_{\PP}\left(N^{-\frac{2\theta}{2\theta+p}}\right),
\end{align*}
or equivalently, 
\begin{align*}
	\lim_{\Gamma\to\infty}\mathop{\sup\lim}_{N\to \infty}\sup_{\beta_0}\mathbb{P}\left\{\mathcal{R}(\overline{\beta}_{S,\lambda})-\mathcal{R}(\beta_0)\geq \Gamma N^{-\frac{2\theta}{2\theta+p}}\right\}=0.
\end{align*}
This completes the proof of \eqref{extension of Yuan and Cai}.

When $0<\theta\leq p/2$, take $m$ and $\lambda$ satisfying $m^2\lambda^{-2p}\leq o(N)$ and $\lambda\leq 1$, then following from \eqref{equation: Theorem 2 3}, one can calculate
\begin{align*}
	\mathcal{R(\overline{\beta}_{S,\lambda})}-\mathcal{R}(\beta_0)\leq\mathcal{O}_{\PP}\left(\lambda^{2\theta}\right),
\end{align*}
or equivalently,
\begin{align*}
	\lim_{\Gamma\to\infty}\mathop{\sup\lim}_{N\to \infty}\sup_{\beta_0}\mathbb{P}\left\{\mathcal{R}(\overline{\beta}_{S,\lambda})-\mathcal{R}(\beta_0)\geq \Gamma \lambda^{2\theta}\right\}=0,
\end{align*}
which further implies (\ref{upperbound22}) and (\ref{upperbound23}). The proof of Theorem \ref{theorem:basic upper bound} is then completed.
\qed

Now we turn to prove Theorem \ref{theorem:no random noise}.

\noindent
{\bf Proof of Theorem \ref{theorem:no random noise}}.
Recalling \eqref{equation: simple notation}, under the noiseless condition, we can write
\[\hat{f}_{S_1,\lambda}= \left(\lambda I + T_{\bX_1}\right)^{-1}\frac{1}{n}\sum_{i=1}^{n}L_K^{1/2}X_{1,i}\langle X_{1,i},\beta_0\rangle_{\cL^2}.\]
Then we can give an improved estimation of the left hand side of \eqref{equation: lemma rates ub 1} as
\begin{align}\label{equation: lemma rates ub 1'}
	&\mathbb{E}\left[\left\|L_C^{1/2}L_K^{1/2}\left(\hat{f}_{S_1,\lambda}-f_\lambda\right)\right\|^2_{\cL^2}\II_{\cU_1^c}\right]\nonumber\\
	&=\EE\left[\left\|L_C^{1/2}L_K^{1/2}(\lambda I+ T_{\bX_1})^{-1}\left(\frac{1}{n}\sum_{i=1}^nL^{1/2}X_{1,i}\langle X_{1,i},\beta_0-L_K^{1/2}f_\lambda\rangle_{\cL^2}-\lambda f_\lambda\right) \right\|^2_{\cL^2}\II_{\cU_1^c}\right]\nonumber\\
	&\overset{(*)}{\leq} 4c_1\frac{m}{N}\cN(\lambda)\lambda^{2\theta}\|\gamma_0\|_{\cL^2}^2,
\end{align}
where inequality $(*)$ follows from \eqref{equation: lemma rates ub 1a}.

Utilizing \eqref{equation: lemma rates ub 1'} and following the same arguments in the proof of Theorem \ref{theorem:basic upper bound}, we have
\begin{align}\label{equation: Theorem 5 1}
	.	\mathcal{R(\overline{\beta}_{S,\lambda})}-\mathcal{R}(\beta_0)\leq \mathcal{O}_{\PP}\left(\lambda^{2\theta}+\frac{m}{N}\lambda^{2\theta-p}\right),
\end{align}
provided that 
$$m^2\lambda^{-2p}\leq o(N) \mbox{ and } \lambda\leq 1.$$
For any $0<\eta\leq 1/2$, take $m\leq o\left(N^{\eta}\right) \mbox{ and } \lambda=N^{-\frac{1-2\eta}{2p}}$, then there hold $m^2\lambda^{-2p}\leq o(N) \mbox{ and } \lambda\leq 1$. Therefore, following from \eqref{equation: Theorem 5 1}, one can calculate
\begin{align*}
	\mathcal{R(\overline{\beta}_{S,\lambda})}-\mathcal{R}(\beta_0)\leq \mathcal{O}_{\PP}\left(\lambda^{2\theta}\right)\leq\mathcal{O}_{\PP}\left(N^{-\frac{\theta(1-2\eta)}{p}}\right),
\end{align*}
or equivalently,
\begin{align*}
	\lim_{\Gamma\to\infty}\mathop{\sup\lim}_{N\to \infty}\sup_{\beta_0}\mathbb{P}\left\{\mathcal{R}(\overline{\beta}_{S,\lambda})-\mathcal{R}(\beta_0) \geq \Gamma N^{-\frac{\theta(1-2\eta)}{p}}\right\}=0.
\end{align*} 
We have obtained \eqref{equation 32}. The proof of Theorem \ref{theorem:no random noise} is then finished.
\qed

We next aim to prove Theorem \ref{theorem:stronger upper bound} and Corollary \ref{corollary:corollary stronger}. We also need several lemmas before  proving them.

When Assumption \ref{momentcondition1} is enhanced to Assumption \ref{momentcondition3}, we can estimate the probability of
event $\cU_1$ better than Lemma \ref{lemma: basic probability esstimation of U1}.
\begin{lemma}\label{lemma: stronger probability esstimation of U1}
	Suppose that Assumption \ref{momentcondition3} is satisfied with some integer $\ell\geq 2$. Then there holds 
	\begin{equation}\label{equation: stronger probability estimation of U1}
		\PP(\cU_1)\leq c(\ell)2^{4\ell}\rho^{4\ell}\left(\frac{m\cN^2(\lambda)}{N}\right)^\ell,
	\end{equation}
	where $c(\ell)$ is a constant only depends on $\ell$ and $\mathcal{N}(\lambda)$ is given by \eqref{effectivedimension}.
\end{lemma}
Lemma \ref{lemma: stronger probability esstimation of U1} can be proved by employing Markov's inequality combined with the following lemma.

\begin{lemma}\label{lemma: estimation of l-th HS-norm}
	Suppose that Assumption \ref{momentcondition3} is satisfied with some integer $\ell\geq2$. Then
	\begin{equation}\label{equation: estimation of l-th HS-norm}
		\EE\left[\left\|(\lambda I + T)^{-1/2}\left(T_{\bX_1} - T\right)(\lambda I + T)^{-1/2}\right\|_{HS}^{2\ell}\right]\leq c(\ell)2^{2\ell}\rho^{4\ell}\left(\frac{m\cN^2(\lambda)}{N}\right)^\ell
	\end{equation}
	and
	\begin{equation}\label{equation: estimation of l-th HS-norm*}
		\EE\left[\left\|(\lambda I + T)^{-1/2}\left(T_{\bX_1} - T\right)\right\|_{HS}^{2\ell}\right]\leq c(\ell)2^{2\ell}\rho^{4\ell}trace^{\ell}(T)\left(\frac{m\cN(\lambda)}{N}\right)^\ell,
	\end{equation}
	where $trace(T)=\sum_{j=1}^\infty \mu_j$ denotes the trace of operator $T$, $\mathcal{N}(\lambda)$ is the effective dimension given by \eqref{effectivedimension}, and $c(\ell)$ is a constant only depends on $\ell$.
\end{lemma}
\begin{proof}
	We first prove inequality (\ref{equation: estimation of l-th HS-norm}). Recalling	\eqref{equation: simple notation}, for brevity of notations, we define
	\[	Q_i:= (\lambda I + T)^{-1/2}\left(L_K^{1/2}X_{1,i}\otimes L_K^{1/2}X_{1,i} -T\right)(\lambda I + T)^{-1/2}, \quad i=1,2,\cdots,n. \]
	Then we can write
	\begin{equation*}
		\begin{aligned}
			&\EE\left[\left\|(\lambda I + T)^{-1/2}\left(T_{\bX_1} -T\right)(\lambda I + T)^{-1/2}\right\|_{HS}^{2\ell}\right]\\
			&=\EE\left[\left\langle \frac{1}{n}\sum_{i=1}^{n}Q_i,\frac{1}{n}\sum_{j=1}^{n}Q_j\right\rangle _{HS}^\ell\right]\\
			&=\frac{1}{n^{2\ell}}\sum_{i_1=1}^{n}...\sum_{i_\ell=1}^{n}\sum_{j_1 = 1}^{n}...\sum_{j_\ell=1}^{n} \EE\left[\langle Q_{i_1},Q_{j_1}\rangle_{HS}\cdots\langle Q_{i_\ell},Q_{j_\ell}\rangle_{HS}\right].
		\end{aligned}
	\end{equation*}
	When the indices in group $\{i_1,...,i_\ell,j_1,...,j_\ell\}$ are all distinct, then following from the independence, there holds 
	$\EE\left[\langle Q_{i_1},Q_{j_1}\rangle_{HS}\cdots\langle Q_{i_\ell},Q_{j_\ell}\rangle_{HS}\right] = 0$. We denote the set of all index-distinct groups by $\Omega(n,\ell)$. Let $\Theta(n,\ell)=\{1,...,n\}^{2\ell}\backslash \Omega(n,\ell)$. Using these notations, we can write
	\begin{equation}\label{equation: lemma 6 2}
		\begin{aligned}
			&\EE\left[\left\|(\lambda I + T)^{-1/2}(T_{\bX_1} -T)(\lambda I + T)^{-1/2}\right\|_{HS}^{2\ell}\right]\\
			=&\frac{1}{n^{2\ell}}\sum_{\{i_1,...,i_\ell,j_1,...,j_\ell\}\in \Theta(n,\ell)} \EE\left[\langle Q_{i_1},Q_{j_1}\rangle _{HS}...\langle Q_{i_\ell},Q_{j_\ell}\rangle _{HS}\right].
		\end{aligned}
	\end{equation}
	We estimate the cardinality of $\Theta(n,k)$ as
	\begin{equation}\label{equation: lemma 6 3}
		\begin{aligned}
			|\Theta(n,\ell)| &= |\Theta_\ell(n,\ell)| +...+ |\Theta_1(n,\ell)|\\ &\leq (2\ell)!\left[\binom{n}{\ell}+ \binom{n}{\ell-1}(\ell-1)^{2}+\cdots+\binom{n}{1}\right]\leq (2\ell)!\ell^{2\ell+1} n^\ell:= c(\ell)n^\ell,
		\end{aligned} 
	\end{equation}
	where $c(\ell):= (2\ell)!\ell^{2\ell+1}$. Let $\Theta_i(n,\ell)$ denote a subset of $\Theta(n,\ell)$ consisting of all groups with exactly $i$ different indices. Then $\Theta(n,\ell)=\cup_{i=1}^\ell\Theta_i(n,\ell)$ and $|\Theta_i(n,\ell)| \leq (2\ell)!\binom{n}{i}i^{2(\ell-i)}\leq (2\ell)!\ell^{2\ell}n^{\ell}$.
	
	For any $\{i_1,\cdots,i_\ell,j_1,\cdots,j_\ell\}\in \Theta(n,\ell)$, we have
	\begin{equation}\label{equation: lemma 6 5}
		\begin{aligned}
			&\EE\left[\left\langle Q_{i_1},Q_{j_1}\right\rangle_{HS}\cdots\left\langle Q_{i_\ell},Q_{j_\ell} \right\rangle_{HS}\right]\\
			&\leq \EE\left[\left\|Q_{i_1}\right\|_{HS}\left\|Q_{j_1}\right\|_{HS}\cdots\left\|Q_{i_\ell}\right\|_{HS}\left\|Q_{j_\ell}\right\|_{HS}\right]\\
			&\overset{(\dagger)}{\leq}
			\left[\EE\left\|Q_{i_1}\right\|_{HS}^{2\ell}\right]^{\frac{1}{2\ell}}\left[\EE\left\|Q_{j_1}\right\|_{HS}^{2\ell}\right]^{\frac{1}{2\ell}}\cdots \left[\EE\left\|Q_{i_\ell}\right\|_{HS}^{2\ell}\right]^{\frac{1}{2\ell}}\left[\EE\left\|Q_{j_\ell}\right\|_{HS}^{2\ell}\right]^{\frac{1}{2\ell}},
		\end{aligned}
	\end{equation}
	where inequality $(\dagger)$ uses H\"older inequality. Then we further bound $\EE\left[\left\|Q_{i}\right\|_{HS}^{2\ell}\right]$ for any $1\leq i\leq n$, which is given 
	\begin{align}\label{equation: lemma 6 6}
		&\EE\left[\left\|Q_{i}\right\|_{HS}^{2\ell}\right]\\
		&=\EE\left[\left\|(\lambda I + T)^{-1/2}\left(L_K^{1/2}X_{1,i}\otimes L_K^{1/2}X_{1,i}-T\right)(\lambda I + T)^{-1/2}\right\|_{HS}^{2\ell}\right]\nonumber\\ 
		&= \EE\left[\left(\sum_{j=1}^{\infty}\sum_{k=1}^{\infty}\frac{1}{\lambda + \mu_j}\frac{1}{\lambda + \mu_k}\left\langle \left(X_{1,i}\otimes X_{1,i}- T\right)\phi_j, \phi_k\right\rangle_{\cL^2}^2\right)^\ell\right]\nonumber\\ 
		&=\sum_{j_1=1}^{\infty } \cdots \sum_{j_\ell=1}^{\infty} \sum_{k_1=1}^{\infty} \cdots \sum_{k_\ell=1}^{\infty}\EE\bigg[\frac{1}{\lambda + \mu_{j_1}}\frac{1}{\lambda +\mu_{k_1}}\left\langle \left(L_K^{1/2}X_{1,i}\otimes L_K^{1/2}X_{1,i}- T\right)\phi_{j_1}, \phi_{k_1}\right\rangle_{\cL^2}^2\nonumber\\
		&\quad\times\cdots\times
		\frac{1}{\lambda + \mu_{j_\ell}}\frac{1}{\lambda +\mu_{k_\ell}}\left\langle \left(L_K^{1/2}X_{1,i}\otimes L_K^{1/2}X_{1,i}- T\right)\phi_{j_\ell}, \phi_{k_\ell}\right\rangle_{\cL^2}^2\bigg]\\
		&\overset{(*)}{\leq}\sum_{j_1=1}^{\infty }\sum_{k_1=1}^{\infty}\frac{1}{\lambda + \mu_{j_1}}\frac{1}{\lambda +\mu_{k_1}}\left[\EE\left\langle \left(L_K^{1/2}X_{1,i}\otimes L_K^{1/2}X_{1,i}- T\right)\phi_{j_1}, \phi_{k_1}\right\rangle^{2\ell}\right]^{\frac{1}{\ell}}\times\cdots\times\nonumber \\
		&\quad\sum_{j_\ell=1}^{\infty}\sum_{k_\ell=1}^{\infty}\frac{1}{\lambda + \mu_{j_\ell}}\frac{1}{\lambda +\mu_{k_\ell}}\left[\EE\left\langle \left(L_K^{1/2}X_{1,i}\otimes L_K^{1/2}X_{1,i}- T\right)\phi_{j_\ell}, \phi_{l_\ell}\right\rangle_{\cL^2}^{2\ell}\right]^{\frac{1}{\ell}},\nonumber
	\end{align}
	where inequality $(*)$ also uses H\"older inequality. It remains to estimate
	$$\EE\left[\left\langle \left(L_K^{1/2}X_{1,i}\otimes L_K^{1/2}X_{1,i}- T\right)\phi_{j}, \phi_{k}\right\rangle_{\cL^2}^{2\ell}\right],\quad \forall 1\leq i\leq n \mbox{ and }\forall 1\leq j,k<\infty.$$
	
	When $j\neq k$, we have
	\begin{align}\nonumber
		&\EE\left[\left\langle \left(L_K^{1/2}X_{1,i}\otimes L_K^{1/2}X_{1,i}- T\right)\phi_{j}, \phi_{k}\right\rangle_{\cL^2}^{2\ell}\right]\nonumber\\
		&= \EE\left[\left\langle L_K^{1/2}X_{1,i},\phi_j\right\rangle_{\cL^2}^{2\ell}\langle L_K^{1/2}X_{1,i},\phi_k\rangle_{\cL^2}^{2\ell}\right]\nonumber\\
		&\overset{(\romannumeral1)}{\leq} \left[\EE\left\langle L_K^{1/2}X_{1,i},\phi_j\right\rangle_{\cL^2}^{4\ell}\right]^{\frac{1}{2}}\left[\EE\left\langle L_K^{1/2}X_{1,i},\phi_k\right\rangle_{\cL^2} ^{4\ell}\right]^{\frac{1}{2}} \overset{(\romannumeral2)}{\leq} \rho^{4\ell}\mu_j^{\ell}\mu_k^{\ell},\nonumber
	\end{align} where inequality (\romannumeral1) is from Cauchy-Schwarz inequality and inequality (\romannumeral2) utilizes Assumption \ref{momentcondition3}. 
	
	When $j = k$, we have
	\begin{align}\nonumber
		&\EE\left[\left\langle \left(L_K^{1/2}X_{1,i}\otimes L_K^{1/2}X_{1,i}- T\right)\phi_{j}, \phi_{j}\right\rangle_{\cL^2}^{2\ell}\right]\\
		&= \EE\left[\left(\left\langle L_K^{1/2}X_{1,i}\otimes L_K^{1/2}X_{1,i} \phi_j,\phi_j\right\rangle_{\cL^2}-\mu_j\right)^{2\ell}\right]\nonumber\\
		&= 2^{2\ell}\EE\left[\left(\frac{1}{2}\left\langle L_K^{1/2}X_{1,i}\otimes L_K^{1/2}X_{1,i} \phi_j,\phi_j\right\rangle_{\cL^2}-\frac{1}{2}\mu_j\right)^{2\ell}\right]\nonumber\\
		&\overset{(\romannumeral1)}{\leq} 2^{2\ell-1} \left(\EE\left[\left\langle L_K^{1/2}X_{1,i}\otimes L_K^{1/2}X_{1,i} \phi_j,\phi_j\right\rangle_{\cL^2}^{2\ell}\right] +\mu_j^{2\ell} \right)\nonumber\\
		&=2^{2\ell-1} \left(\EE\left[\left\langle L_K^{1/2}X_{1,i}, \phi_j\right\rangle_{\cL^2}^{4\ell}\right] +\mu_j^{2\ell} \right)\overset{(\romannumeral2)}{\leq} 2^{2\ell}\rho^{4\ell}\mu_j^{2\ell},\nonumber
	\end{align}
	where inequality (\romannumeral1) is due to Jensen's inequality and inequality (\romannumeral2) follows from Assumption \ref{momentcondition3} and the fact that $\rho\geq 1$.
	
	Combining the above estimations, for any $1\leq i\leq n$ and $1\leq j,k <\infty$, there holds
	\begin{equation}\label{equation: lemma 6 7}
		\EE\left[\left\langle \left(L_K^{1/2}X_{1,i}\otimes L_K^{1/2}X_{1,i}- T\right)\phi_{j}, \phi_{k}\right\rangle^{2\ell}\right]\leq 2^{2\ell}\rho^{4\ell}\mu_j^\ell\mu_k^\ell.
	\end{equation}
	Recall that $n = N/m$. Combining (\ref{equation: lemma 6 2}), (\ref{equation: lemma 6 3}), (\ref{equation: lemma 6 5}), (\ref{equation: lemma 6 6}) and (\ref{equation: lemma 6 7}) yields
	\[	\begin{aligned}
		\EE\left[\left\|(\lambda I + T)^{-1/2}\left(T_{\bX_1} -T\right)(\lambda I + T)^{-1/2}\right\|_{HS}^{2\ell}\right]
		&\leq c(\ell)2^{2\ell}\rho^{4\ell}\left(\frac{m\cN^2(\lambda)}{N}\right)^\ell.
	\end{aligned}\]
	This completes the proof of \eqref{equation: estimation of l-th HS-norm}.
	
	Analogously, we can demonstrate the second inequality (\ref{equation: estimation of l-th HS-norm*}) through
	\[\EE\left[\left\|(\lambda I + T)^{-1/2}\left(L_K^{1/2}X_{1,i}\otimes L_K^{1/2}X_{1,i}-T\right)\right\|_{HS}^{2\ell}\right]\leq 2^{2\ell}\rho^{4\ell}trace^{\ell}(T)\cN^\ell(\lambda)\]
	and
	\[\EE\left[\left\|(\lambda I + T)^{-1/2}\left(T_{\bX_1} -T\right)\right\|_{HS}^{2\ell}\right]\leq c(\ell)2^{2\ell}\rho^{4\ell}trace^{\ell}(T)\left(\frac{m\cN(\lambda)}{N}\right)^\ell.\]
	The proof of Lemma \ref{lemma: estimation of l-th HS-norm} is then finished.
\end{proof}

The following lemma plays a key role in estimating the upper bound  of $\mathscr{S}(S,\lambda)$ under Assumption \ref{momentcondition3}.

\begin{lemma}\label{lemma: estimation of expectation of F(S,lambda)}
	Suppose that Assumption \ref{assumption1} is satisfied with $0<\theta\leq 1/2$ and $\gamma_0\in \mathcal{L}^2(\cal T)$. Under Assumption \ref{assumption2} and Assumption \ref{momentcondition3}, taking $\lambda\leq 1$ yields
	\begin{equation}\label{equation: estimtion of 4-th norm1}
		\begin{aligned}
			&\EE\left[\left\|(\lambda I + T)^{-1/2}\frac{1}{|S_1|}\left(\sum_{X\in \bX_1}L_K^{1/2}X\langle X,\beta_0-L_K^{1/2}f_\lambda\rangle_{\cL^2} - \lambda f_\lambda\right)\right\|_{\cL^2}^4\right]\\
			&\leq c_6^2\frac{m^2}{N^2}(1+\lambda^{4\theta}\mathcal{N}^2(\lambda)),
		\end{aligned}
	\end{equation}
	where $c_6$ is a universal constant and $\mathcal{N}(\lambda)$ is given by \eqref{effectivedimension}.
\end{lemma}

\begin{proof}
	Recalling \eqref{equation: simple notation}, for simplicity of notations, we define
	\[\alpha_i := L_K^{1/2}X_{1,i}\left\langle X_{1,i},\beta_0-L_K^{1/2}f_\lambda\right\rangle_{\cL^2}-\lambda f_\lambda, \quad i=1,2,\cdots,n.\]
	
	We begin with the proof of the first inequality (\ref{equation: estimtion of 4-th norm1}). Note that $\left\{(\lambda I + T)^{-1/2}\alpha_i\right\}_{i=1}^n$ are independent operator-valued zero-mean random elements. Then we can write
	\begin{equation}\label{equation: lemma 7 1}
		\begin{aligned}
			&\EE\left[\left\|(\lambda I + T)^{-1/2}\frac{1}{n}\sum_{i=1}^{n}\alpha_i\right\|_{\cL^2}^4\right]\\
			&=\frac{1}{n^4}\sum_{i_1=1}^{\infty}\sum_{i_2=1}^\infty\sum_{j_1=1}^\infty\sum_{j_2=1}^\infty \EE\Big[\left\langle (\lambda I + T)^{-1/2}\alpha_{i_1},(\lambda I + T)^{-1/2}\alpha_{i_2}\right\rangle_{\cL^2} \\
			&\quad\times\left\langle (\lambda I + T)^{-1/2}\alpha_{j_1},(\lambda I + T)^{-1/2}\alpha_{j_2}\right\rangle_{\cL^2} \Big]\\
			&=\frac{1}{n^4}\sum_{\{i_1,i_2,j_1,j_2\}\in \Theta(n,2)}\EE\Big[\left\langle (\lambda I + T)^{-1/2}\alpha_{i_1},(\lambda I + T)^{-1/2}\alpha_{i_2}\right\rangle_{\cL^2} \\
			&\quad\times\left\langle(\lambda I + T)^{-1/2}\alpha_{j_1},(\lambda I + T)^{-1/2}\alpha_{j_2}\right\rangle_{\cL^2} \Big],
		\end{aligned}
	\end{equation}
	where $\Theta(n,2)=\{1,...,n\}^{4}\backslash \Omega(n,2)$ and $\Omega(n,2)$ denotes the set of all index-distinct group $\{i_1,i_2,j_1,j_2\}$. Then 
	\begin{equation}\label{equation: lemma 7 2}
		|\Theta(n,2)|\leq 4!\left[\binom{n}{2}+\binom{n}{1}\right] \leq 24n^2, \quad \forall n\geq 1.
	\end{equation}
	And for any $\{i_1,i_2,j_1,j_2\}\in \Theta(n,2)$, we have
	\begin{align}\label{equation: lemma 7 3}
		&\EE\left[\left\langle (\lambda I + T)^{-1/2}\alpha_{i_1},(\lambda I + T)^{-1/2}\alpha_{i_2}\rangle_{\cL^2} \langle (\lambda I + T)^{-1/2}\alpha_{j_1},(\lambda I + T)^{-1/2}\alpha_{j_2}\right\rangle_{\cL^2} \right]\nonumber\\
		&\leq \EE\left[\left\|(\lambda I + T)^{-1/2}\alpha_{i_1}\right\|_{\cL^2}\left\|(\lambda I + T)^{-1/2}\alpha_{i_2}\right\|_{\cL^2}\left\|(\lambda I + T)^{-1/2}\alpha_{j_1}\right\|_{\cL^2}\left\|(\lambda I + T)^{-1/2}\alpha_{j_2}\right\|_{\cL^2}\right]\nonumber\\
		&\overset{(*)}{\leq} \left[\EE\left\|(\lambda I + T)^{-1/2}\alpha_{i_1}\right\|_{\cL^2}^4\right]^{\frac{1}{4}}\left[\EE\left\|(\lambda I + T)^{-1/2}\alpha_{i_2}\right\|_{\cL^2}^4\right]^{\frac{1}{4}}\\
		&\quad\times\left[\EE\left\|(\lambda I + T)^{-1/2}\alpha_{j_1}\right\|_{\cL^2}^4\right]^{\frac{1}{4}}\left[\EE\left\|(\lambda I + T)^{-1/2}\alpha_{j_2}\right\|_{\cL^2}^4\right]^{\frac{1}{4}},\nonumber
	\end{align} where inequality $(*)$ uses H\"older inequality.
	
	It remains to estimate $\EE\left[\left\|(\lambda I + T)^{-1/2}\alpha_{i}\right\|_{\cL^2}^4\right], \forall 1\leq i\leq n$. For brevity of notations, we define
	\[\widetilde{\alpha}_i :=  L_K^{1/2}X_{1,i}\left\langle X_{1,i},\beta_0-L_K^{1/2}f_\lambda\right\rangle_{\cL^2}, \quad i=1,2,\cdots,n.\] Then we see that $\alpha_i = \widetilde{\alpha}_i - \lambda f_\lambda$ and for any $1\leq i \leq n$, 
	\begin{align}\label{equation: lemma 7 4'}
		&\EE\left[\left\|(\lambda I + T_0)^{-1/2}\widetilde{\alpha}_i\right\|_{\cL^2}^4\right]\nonumber\\
		&= \EE\left[\left(\sum_{j=1}^\infty\left\langle (\lambda I + T)^{-1/2}\widetilde{\alpha}_i,\phi_{j}\right\rangle_{\cL^2}^2\right)^2\right]\nonumber\\
		&= \EE\left[\sum_{j_1=1}^{\infty}\sum_{j_2=1}^\infty\frac{1}{\lambda+\mu_{j_1}}\frac{1}{\lambda+\mu_{j_2}}\left\langle \widetilde{\alpha}_i,\phi_{j_1}\right\rangle_{\cL^2}^2\left\langle \widetilde{\alpha}_i,\phi_{j_2}\right\rangle_{\cL^2}^2\right]\\
		&\overset{(\dagger)}{\leq} \sum_{j_1=1}^{\infty}\sum_{j_2=1}^\infty\frac{1}{\lambda+\mu_{j_1}}\frac{1}{\lambda+\mu_{j_2}} \left[\EE\left\langle \widetilde{\alpha}_i,\phi_{j_1}\right\rangle_{\cL^2}^4\right]^{\frac{1}{2}}\left[\EE\left\langle \widetilde{\alpha}_i,\phi_{j_2}\right\rangle_{\cL^2}^4\right]^{\frac{1}{2}},\nonumber
	\end{align}
	where inequality $(\dagger)$ uses Cauchy-Schwartz inequality. We further bound $\mathbb{E}\left[\langle \widetilde{\alpha}_i,\phi_{j}\rangle_{\cL^2}^4\right]$ as
	\begin{equation}\label{equation: lemma 7 4''}
		\begin{aligned}
			&\EE[\langle \widetilde{\alpha}_i,\phi_{j}\rangle_{\cL^2}^4]\\
			&= \EE\left[\left\langle L_K^{1/2}X_{1,i},\phi_{j}\right\rangle_{\cL^2}^4\left\langle X_{1,i},\beta_0-L_K^{1/2}f_\lambda\right\rangle_{\cL^2}^4\right]\\
			&\overset{(\romannumeral1)}{\leq} \left[\EE\left\langle L_K^{1/2}X_i,\phi_{j}\right\rangle_{\cL^2}^8\right]^{\frac{1}{2}}\left[\EE\left\langle X_i,\beta_0-L_K^{1/2}f_\lambda\right\rangle_{\cL^2}^8\right]^{\frac{1}{2}}\\
			&\overset{(\romannumeral2)}{\leq} c_2\rho^4\mu_j^2\left[\EE\left\langle X_i,\beta_0-L_K^{1/2}f_\lambda\right\rangle_{\cL^2}^2\right]^2= c_2\rho^4\mu_j^2\mathscr{A}^4(\lambda)\overset{(\romannumeral3)}{\leq} c_2\rho^4\mu_j^2 \|\gamma_0\|_{\cL^2}^4\lambda^{4\theta},
		\end{aligned}
	\end{equation}
	where inequality (\romannumeral1) again uses Cauchy-Schwartz inequality, inequality (\romannumeral2) is due to Assumption \ref{momentcondition3} and inequality (\romannumeral3) is from Lemma \ref{lemma: estimation of A(lambda)}. 
	
	Combining (\ref{equation: lemma 7 4'}) and (\ref{equation: lemma 7 4''}) yields
	\[
	\EE\left[\left\|(\lambda I + T)^{-1/2}\widetilde{\alpha}_i\right\|_{\cL^2}^4\right]\leq c_2\rho^4 \|\gamma_0\|_{\cL^2}^4\lambda^{4\theta}\mathcal{N}^{2}(\lambda),\quad \forall 1\leq i\leq n.
	\]
	Then for any $1\leq i \leq n$, we have
	\begin{align}\label{equation: lemma 7 5}
		&\EE\left[\left\|(\lambda I + T)^{-1/2}\alpha_i\right\|_{\cL^2}^4\right]=\EE\left[\left\|(\lambda I + T)^{-1/2}(\tilde{\alpha}_i-\lambda f_\lambda)\right\|_{\cL^2}^4\right]\nonumber\\
		&\overset{(\romannumeral1)}{\leq} 8\EE\left[\left\|(\lambda I + T)^{-1/2}\widetilde{\alpha}_i\right\|_{\cL^2}^4\right]+ 8 \left\|(\lambda I + T)^{-1/2}\lambda f_\lambda\right\|_{\cL^2}^4\nonumber\\
		&\overset{(\romannumeral2)}{=}8\EE\left[\left\|(\lambda I + T)^{-1/2}\widetilde{\alpha}_i\right\|_{\cL^2}^4\right]+ 8\left\|(\lambda I + T)^{-1/2}\lambda(\lambda I + T)^{-1}L_K^{1/2}L_C^{1/2}T_*^\theta(\gamma_0)\right\|_{\cL^2}^4\\
		&\overset{(\romannumeral3)}{\leq} 8c_2\rho^4 \|\gamma_0\|_{\cL^2}^4\lambda^{4\theta}\mathcal{N}^{2}(\lambda)+ 8\left\|(\lambda I + T)^{-1}\lambda\right\|^4\left\|(\lambda I + T)^{-1/2}L_K^{1/2}L_C^{1/2}\right\|^4\left\|T_*^\theta\right\|^4\|\gamma_0\|_{\cL^2}^4\nonumber\\
		&\leq 8c_2\rho^4 \|\gamma_0\|_{\cL^2}^4\lambda^{4\theta}\mathcal{N}^{2}(\lambda)+ 8\left\|T_*^\theta\right\|^4\|\gamma_0\|_{\cL^2}^4 =8c_2\rho^4 \|\gamma_0\|_{\cL^2}^4\lambda^{4\theta}\mathcal{N}^{2}(\lambda)+ 8\mu_1^{4\theta}\|\gamma_0\|_{\cL^2}^4,\nonumber
	\end{align}
	where inequality (\romannumeral1) uses the triangular inequality, inequality (\romannumeral2) follows from Assumption \ref{assumption1} and the expression of $f_\lambda$ and inequality (\romannumeral3) applies the above estimation.
	
	Recall that $n=N/m$ and take $\lambda\leq 1$. Combining with (\ref{equation: lemma 7 1}), (\ref{equation: lemma 7 2}), (\ref{equation: lemma 7 3}) and (\ref{equation: lemma 7 5}), we obtain
	\begin{equation*}
		\begin{split}
			&\EE\left[\left\|(\lambda I + T_0)^{-1/2}\frac{1}{n}\sum_{i=1}^{n}\alpha_i\right\|_{\cL^2}^4\right]\\
			&\leq \frac{192m^2}{N^2}\left(c_2\rho^4 ||\gamma_0||_{\cL^2}^4\lambda^{4\theta}\mathcal{N}^{2}(\lambda)+ \mu_1^{4\theta}||\gamma_0||_{\cL^2}^4\right)\leq c_6^2\frac{m^2}{N^2}(1+\lambda^{4\theta}\mathcal{N}^2(\lambda)),
		\end{split}
	\end{equation*}
	where $c_6^2:= 192\left(c_2\rho^4 ||\gamma_0||_{\cL^2}^4 + \max\{\mu_1^2,1\}||\gamma_0||_{\cL^2}^4\right)$. We have completed the proof of Lemma \ref{lemma: estimation of expectation of F(S,lambda)}.
\end{proof}

We propose the following lemma to decompose $\EE[\mathscr{S}(S,\lambda)]$.
\begin{lemma}\label{lemma: expecatation upper bound 1}
	For any $m\geq 1$, there holds
	\begin{equation}\label{equation for theorem sup 1} 
		\mathbb{E}\left[\mathscr{S}(S,\lambda)\right]\leq \frac{1}{m}\mathbb{E}\left[\left\|L_C^{1/2}L_K^{1/2}\left(\hat{f}_{S_1,\lambda}-f_\lambda\right)\right\|^2_{\cL^2}\right]+ \left\|L_C^{1/2}L_K^{1/2}\mathbb{E}\left[(\hat{f}_{S_1,\lambda}-f_\lambda)\right]\right\|^2_{\cL^2}.
	\end{equation}
\end{lemma} 
\begin{proof}
	When $m\geq 2$, as 
	$$\mathscr{S}(S,\lambda)=\left\|L^{1/2}_CL^{1/2}_K\overline{f}_{S,\lambda}-L^{1/2}_CL^{1/2}_Kf_{\lambda}\right\|^2_{{\cal L}^2}=\left\|\frac{1}{m}\sum_{i=1}^mL^{1/2}_CL^{1/2}_K\hat{f}_{S_i,\lambda}-L^{1/2}_CL^{1/2}_Kf_{\lambda}\right\|^2_{{\cal L}^2},$$
	we can write
	\begin{align*}
		\EE[\mathscr{S}(S,\lambda)]&= \EE\left[\left\|L_C^{1/2}L_K^{1/2}\left(\frac{1}{m}\sum_{i=1}^{m}\hat{f}_{S_i,\lambda}-f_\lambda\right)\right\|^2_{\cL^2}\right]\\
		&\overset{(\romannumeral1)}{=}\frac{1}{m^2}\sum_{i=1}^{m}\EE\left[\left\|L_C^{1/2}L_K^{1/2}\left(\hat{f}_{S_i,\lambda}-f_\lambda\right)\right\|^2_{\cL^2}\right]\\
		&\quad + \frac{1}{m^2}\sum_{i\neq j} \EE\left[\left\langle L_C^{1/2}L_K^{1/2}\left(\hat{f}_{S_i,\lambda}-f_\lambda\right),L_C^{1/2}L_K^{1/2}\left(\hat{f}_{S_j,\lambda}-f_\lambda\right)\right\rangle_{\cL^2}\right]\\
		&\overset{(\romannumeral2)}{\leq} \frac{1}{m}\mathbb{E}\left[\left\|L_C^{1/2}L_K^{1/2}\left(\hat{f}_{S_1,\lambda}-f_\lambda\right)\right\|^2_{\cL^2}\right]+ \left\|L_C^{1/2}L_K^{1/2}\mathbb{E}\left[\left(\hat{f}_{S_1,\lambda}-f_\lambda\right)\right]\right\|^2_{\cL^2}.
	\end{align*}
	where equality (\romannumeral1) follows from the binomial expansion and inequality (\romannumeral2) is from \[\EE\left[\langle L_C^{1/2}L_K^{1/2}\left(\hat{f}_{S_1,\lambda}-f_\lambda\right),L_C^{1/2}L_K^{1/2}\left(\hat{f}_{S_2,\lambda}-f_\lambda\right)\rangle_{\cL^2} \right] = \left\|L_C^{1/2}L_K^{1/2}\mathbb{E}\left[\left(\hat{f}_{S_1,\lambda}-f_\lambda\right)\right]\right\|^2_{\cL^2}.\]
	When $m=1$, \eqref{equation for theorem sup 1} is obvious.
	
	Thus, we have completed the proof of Lemma \ref{lemma: expecatation upper bound 1}.
\end{proof}

Now we are in the position to prove Theorem \ref{theorem:stronger upper bound}.

\noindent
{\bf Proof of Theorem \ref{theorem:stronger upper bound}}.
Combining (\ref{decomposition 2}), (\ref{equation: estimation of A(lambda)}) and (\ref{equation for theorem sup 1}) yields
\begin{align}\label{equation: theorem 3 total 1}
	&\EE\left[\left(\mathcal{R(\overline{\beta}_{S,\lambda})}-\mathcal{R}(\beta_0)\right)\right]\leq 2\EE\left[\mathscr{S}(S,\lambda)\right] + 2 \mathscr{A}(\lambda)\\
	&\leq\frac{2}{m}\mathbb{E}\left[\left\|L_C^{1/2}L_K^{1/2}(\hat{f}_{S_1,\lambda}-f_\lambda)\right\|^2_{\cL^2}\right]+ 2\left\|L_C^{1/2}L_K^{1/2}\mathbb{E}\left[(\hat{f}_{S_1,\lambda}-f_\lambda)\right]\right\|^2_{\cL^2}+2 \mathscr{A}(\lambda)\nonumber\\
	&\leq\frac{2}{m}\mathbb{E}\left[\left\|L_C^{1/2}L_K^{1/2}(\hat{f}_{S_1,\lambda}-f_\lambda)\right\|^2_{\cL^2}\right]+ 2\left\|L_C^{1/2}L_K^{1/2}\mathbb{E}\left[(\hat{f}_{S_1,\lambda}-f_\lambda)\right]\right\|^2_{\cL^2}+2\lambda^{2\theta}\|\gamma_0\|_{\cL^2}^2.\nonumber
\end{align}

In the following part of the proof, we aim to bound the terms $\mathbb{E}\left[\left\|L_C^{1/2}L_K^{1/2}(\hat{f}_{S_1,\lambda}-f_\lambda)\right\|^2_{\cL^2}\right]$ and $\left\|L_C^{1/2}L_K^{1/2}\mathbb{E}\left[(\hat{f}_{S_1,\lambda}-f_\lambda)\right]\right\|^2_{\cL^2}$, respectively. Recalling \eqref{equation: simple notation} and $Y_{1,i}=\langle X_{1,i},\beta_0\rangle_{\cL^2}+\epsilon_{1,i}, \forall 1\leq i\leq n$, for simplicity of notations, let
\[\alpha_i := L_K^{1/2}X_{1,i}\left\langle X_{1,i},\beta_0-L_K^{1/2}f_\lambda\right\rangle_{\cL^2}-\lambda f_\lambda, \quad i=1,2,\cdots,n.\]
Then  \[\hat{f}_{S_1,\lambda}-f_\lambda = \left(\lambda I + T_{\bX_1}\right)^{-1}\frac{1}{n}\sum_{i=1}^n\left(\alpha_{i}+L_K^{1/2}X_{1,i}\epsilon_{1,i}\right).\] Using this expression, we can bound $\mathbb{E}\left[\left\|L_C^{1/2}L_K^{1/2}\left(\hat{f}_{S_1,\lambda}-f_\lambda\right)\right\|^2_{\cL^2}\right]$ as
\begin{align}\label{equation: theorem 3 total 2}
	&\mathbb{E}\left[\left\|L_C^{1/2}L_K^{1/2}\left(\hat{f}_{S_1,\lambda}-f_\lambda\right)\right\|^2_{\cL^2}\right]\nonumber\\
	&= \mathbb{E}\left[\left\|L_C^{1/2}L_K^{1/2}\left(\hat{f}_{S_1,\lambda}-f_\lambda\right)\right\|^2_{\cL^2}\II_{\cU_1^c}\right] + \mathbb{E}\left[\left\|L_C^{1/2}L_K^{1/2}\left(\hat{f}_{S_1,\lambda}-f_\lambda\right)\right\|^2_{\cL^2}\II_{\cU_1}\right]\nonumber\\
	&\overset{(\romannumeral1)}{\leq} 8\frac{m}{N}\cN(\lambda)\left(c_2\lambda^{2\theta}\|\gamma_0\|^2_{\cL^2}+\sigma^2\right) + 2\EE\left[\left\|L_C^{1/2}L_K^{1/2}(\lambda I + T_{\bX_1})^{-1}\frac{1}{n}\sum_{i=1}^n\alpha_{i}\right\|_{\cL^2}^2\II_{\cU_1}\right]\nonumber\\
	&\quad+ 2\EE\left[\left\|L_C^{1/2}L_K^{1/2}(\lambda I + T_{\bX_1})^{-1}\frac{1}{n}\sum_{i=1}^nL_K^{1/2}X_{1,i}\epsilon_{1,i}\right\|_{\cL^2}^2\II_{\cU_1}\right]\\
	&\overset{(\romannumeral2)}{\leq} 8\frac{m}{N}\cN(\lambda)\left(c_2\lambda^{2\theta}\|\gamma_0\|^2_{\cL^2}+\sigma^2\right) + 2\EE\left[\left\|L_C^{1/2}L_K^{1/2}(\lambda I + T_{\bX_1})^{-1}\frac{1}{n}\sum_{i=1}^n\alpha_{i}\right\|_{\cL^2}^2\II_{\cU_1}\right]\nonumber\\
	&\quad + \frac{2\sigma^2}{n^2}\sum_{i=1}^n\EE\left[\left\|L_C^{1/2}L_K^{1/2}(\lambda I + T_{\bX_1})^{-1}L_K^{1/2}X_{1,i}\right\|_{\cL^2}^2\II_{\cU_1}\right].\nonumber
\end{align}

Here inequality (\romannumeral1) follows from (\ref{equation: lemma rates ub 1}) in Lemma \ref{lemma: rates upper bound 2} and the triangular inequality. Inequality (\romannumeral2) is due to Assumption \ref{assumption2}.

We next bound $\left\|L_C^{1/2}L_K^{1/2}\mathbb{E}\left[\hat{f}_{S_1,\lambda}-f_\lambda\right]\right\|^2_{\cL^2}$ as
\begin{equation}\label{equation: theorem 3 total 3}
	\begin{aligned}
		&\left\|L_C^{1/2}L_K^{1/2}\mathbb{E}\left[\hat{f}_{S_1,\lambda}-f_\lambda\right]\right\|^2_{\cL^2}\\
		&\overset{(\romannumeral1)}{=}\left\|L_C^{1/2}L_K^{1/2}\mathbb{E}\left[(\lambda I + T_{\bX_1})^{-1}\frac{1}{n}\sum_{i=1}^n\alpha_i\right]\right\|^2_{\cL^2}\\
		&\leq \EE\left[\left\|L_C^{1/2}L_K^{1/2}(\lambda I + T_{\bX_1})^{-1}\frac{1}{n}\sum_{i=1}^n\alpha_i\right\|^2_{\cL^2}\II_{\cU_1^c}\right]\\
		&\quad + \EE\left[\left\|L_C^{1/2}L_K^{1/2}(\lambda I + T_{\bX_1})^{-1}\frac{1}{n}\sum_{i=1}^n\alpha_i\right\|^2_{\cL^2}\II_{\cU_1}\right]\\
		&\overset{(\romannumeral2)}{\leq} 4c_2\frac{m}{N}\cN(\lambda)\lambda^{2\theta}\|\gamma_0\|_{\cL^2}^2 + \EE\left[\left\|L_C^{1/2}L_K^{1/2}(\lambda I + T_{\bX_1})^{-1}\frac{1}{n}\sum_{i=1}^n\alpha_i\right\|^2_{\cL^2}\II_{\cU_1}\right].
	\end{aligned}
\end{equation}
Here equality (\romannumeral1) is from Assumption \ref{assumption2}. Inequality (\romannumeral2) follows from Jensen's inequality and \eqref{equation: lemma rates ub 2} in Lemma \ref{lemma: rates upper bound 2}.

The key point in the rest of the proof is to estimate $\EE\left[\left\|L_C^{1/2}L_K^{1/2}(\lambda I + T_{\bX_1})^{-1}\frac{1}{n}\sum_{i=1}^n\alpha_{i}\right\|_{\cL^2}^2\II_{\cU_1}\right]$ and $\EE\left[\left\|L_C^{1/2}L_K^{1/2}(\lambda I + T_{\bX_1})^{-1}L_K^{1/2}X_{1,i}\right\|_{\cL^2}^2\II_{\cU_1}\right], \forall 1\leq i\leq n$. For the first term, we have

\begin{align}\label{equation: Theorem 3 1}
	&\EE\left[\left\|L_C^{1/2}L_K^{1/2}\left(\lambda I + T_{\bX_1}\right)^{-1}\frac{1}{n}\sum_{i=1}^n\alpha_{i}\right\|_{\cL^2}^2\II_{\cU_1}\right]\nonumber\\
	&\leq \EE\left[\left\|L_C^{1/2}L_K^{1/2}(\lambda I + T)^{-1/2}\right\|^2\left\|(\lambda I + T)^{1/2}\left(\lambda I + T_{\bX_1}\right)^{-1}\frac{1}{n}\sum_{i=1}^n\alpha_{i}\right\|_{\cL^2}^2\II_{\cU_1}\right]\nonumber\\
	&\overset{(\romannumeral1)}{\leq} \EE\left[\left\|(\lambda I + T)^{1/2}\left(\lambda I + T_{\bX_1}\right)^{-1}\frac{1}{n}\sum_{i=1}^n\alpha_{i}\right\|_{\cL^2}^2\II_{\cU_1}\right]\\
	&\overset{(\romannumeral2)}{\leq} \left[\EE\left\|(\lambda I + T)^{1/2}\left(\lambda I + T_{\bX_1}\right)^{-1}(\lambda I + T)^{1/2}\right\|^4\II_{\cU_1}\right]^{\frac{1}{2}}\left[\EE\left\|(\lambda I + T)^{-1/2}\frac{1}{n}\sum_{i=1}^{n}\alpha_i\right\|_{\cL^2}^4\right]^{\frac{1}{2}}\nonumber\\
	&\overset{(\romannumeral3)}{\leq} \left[\EE\left\|(\lambda I + T)^{1/2}\left(\lambda I + T_{\bX_1}\right)^{-1}(\lambda I + T)^{1/2}\right\|^8\right]^{\frac{1}{4}}\PP^{\frac{1}{4}}(\cU_1)\left[\EE\left\|(\lambda I + T)^{-1/2}\frac{1}{n}\sum_{i=1}^{n}\alpha_i\right\|_{\cL^2}^4\right]^{\frac{1}{2}}.\nonumber
\end{align}
Here inequality (\romannumeral1) follows from the fact that   \[\begin{aligned}
	\left\|L_C^{1/2}L_K^{1/2}(\lambda I + T)^{-1/2}\right\|^2
	&=\left\|(\lambda I + T)^{-1/2}L_K^{1/2}L_CL_K^{1/2}(\lambda I + T)^{-1/2}\right\|\\
	&=\left\|(\lambda I + T)^{-1/2}T(\lambda I + T)^{-1/2}\right\|\leq 1.
\end{aligned}\] Inequalities (\romannumeral2) and (\romannumeral3) are from Cauchy-Schwartz inequality.

Analogously, for the second term, we have
\begin{align}\label{equation: Theorem 3 2}
	&\EE\left[\left\|L_C^{1/2}L_K^{1/2}(\lambda I + T_{\bX_1})^{-1}L_K^{1/2}X_{1,i}\right\|_{\cL^2}^2\II_{\cU_1}\right]\\
	&\leq \left[\EE\left\|(\lambda I + T)^{1/2}\left(\lambda I + T_{\bX_1}\right)^{-1}(\lambda I + T)^{1/2}\right\|^8\right]^{\frac{1}{4}}\PP^{\frac{1}{4}}(\cU_1)\left[\EE\left\|(\lambda I + T)^{-1/2}L_K^{1/2}X_{1,i}\right\|_{\cL^2}^4\right]^{\frac{1}{2}}.\nonumber
\end{align}
While we can write
\begin{align}\label{equation: Theorem 3 3}
	&\EE\left[\left\|(\lambda I + T)^{-1/2}L_K^{1/2}X_{1,i}\right\|_{\cL^2}^4\right]\nonumber\\
	&= \EE\left[\left(\sum_{j=1}^\infty\frac{1}{\lambda + \mu_j}\left\langle L_K^{1/2}X_{1,i}, \phi_{j}\right\rangle_{\cL^2}^2\right)^2\right]\nonumber\\
	&= \sum_{j=1}^\infty\sum_{k=1}^\infty\frac{1}{\lambda + \mu_j}\frac{1}{\lambda + \mu_k}\EE\left[\left\langle L_K^{1/2}X_{1,i}, \phi_{j}\right\rangle_{\cL^2}^2\left\langle L_K^{1/2}X_{1,i}, \phi_{k}\right\rangle_{\cL^2}^2\right]\\
	&\overset{(\romannumeral1)}{\leq} \sum_{j=1}^\infty\sum_{k=1}^\infty\frac{1}{\lambda + \mu_j}\frac{1}{\lambda + \mu_k}\left[\EE\left\langle L_K^{1/2}X_{1,i}, \phi_{j}\right\rangle_{\cL^2}^4\right]^{\frac{1}{2}}\left[\EE\left\langle L_K^{1/2}X_{1,i}, \phi_{k}\right\rangle_{\cL^2}^4\right]^{\frac{1}{2}}\nonumber\\
	&\overset{(\romannumeral2)}{\leq} \rho^4\sum_{j=1}^\infty\sum_{k=1}^\infty\frac{\mu_j}{\lambda + \mu_j}\frac{\mu_k}{\lambda + \mu_k}=\rho^4\cN^2(\lambda).\nonumber
\end{align}
Here $\{\phi_k\}_{k=1}^\infty$ is given by the singular value decomposition of $T$ in \eqref{singular value decomposition}. Inequality (\romannumeral1) is from Cauchy-Schwartz inequality. Inequality (\romannumeral2) is due to the decomposition of $L_K^{1/2}X$ \eqref{PCdecomposition} and Assumption \ref{momentcondition3}. 

For the term $\EE\left[\left\|(\lambda I + T)^{1/2}\left(\lambda I + T_{\bX_1}\right)^{-1}(\lambda I + T)^{1/2}\right\|^8\right]$, first applying the second-order decomposition, which was introduced in \cite{guo2017learning,lin2017distributed,guo2019optimal}, to $\left (\lambda I + T_{\bX_1}\right)^{-1}$ yields that
\begin{equation}\label{equation: second order decomposition}
	\begin{aligned}
		(\lambda I + T_{\bX_1})^{-1} &= (\lambda I + T)^{-1} + (\lambda I + T_{\bX_1})^{-1}(T - T_{\bX_1})(\lambda I + T)^{-1}\\
		&= (\lambda I + T)^{-1} + (\lambda I + T)^{-1}(T - T_{\bX_1})(\lambda I + T)^{-1}\\
		&\quad+ (\lambda I + T)^{-1}(T - T_{\bX_1})(\lambda I + T_{\bX_1})^{-1}(T - T_{\bX_1})(\lambda I + T)^{-1}.
	\end{aligned}
\end{equation}
If $2\leq\ell<8$, applying the above second-order decomposition of $(\lambda I + T_{\bX_1})^{-1}$ and taking $\lambda\leq 1$, we have \begin{align}\label{equation: Theorem 3 4}
	&\EE\left[\left\|(\lambda I + T)^{1/2}\left(\lambda I + T_{\bX_1}\right)^{-1}(\lambda I + T)^{1/2}\right\|^8\right]\nonumber\\
	&\leq (1 + \mu_1)^{8-\ell}\frac{1}{\lambda^{8-\ell}}\EE\left[\left\|(\lambda I + T)^{1/2}\left(\lambda I + T_{\bX_1}\right)^{-1}(\lambda I + T)^{1/2}\right\|^\ell\right]\nonumber\\
	&\overset{(\romannumeral1)}{=}(1 + \mu_1)^{8-\ell}\frac{3^\ell}{\lambda^{8-\ell}}\EE\Bigg[\Big\|\frac{1}{3}I + \frac{1}{3}(\lambda I + T)^{-1/2}(T-T_{\bX_1})(\lambda I + T)^{-1/2}\nonumber\\
	&\quad + \frac{1}{3}(\lambda I + T)^{-1/2}(T-T_{\bX_1})(\lambda I +T_{\bX_1})^{-1}(T-T_{\bX_1})(\lambda I + T)^{-1/2}\Big\|^\ell\Bigg]\nonumber\\
	&\overset{(\romannumeral2)}{\leq} (1 + \mu_1)^{8-\ell}\frac{3^{\ell-1}}{\lambda^{8-\ell}}\Bigg\{\|I\|^\ell + \EE\left[\left\|(\lambda I + T)^{-1/2}(T-T_{\bX_1})(\lambda I + T)^{-1/2}\right\|^{\ell}\right] \nonumber\\
	&\quad + \EE\left[\left\|(\lambda I + T)^{-1/2}(T-T_{\bX_1})\right\|^{2\ell}\right]\frac{1}{\lambda^{\ell}} \Bigg\}\\
	&\overset{(\romannumeral3)}{\leq} (1 + \mu_1)^{8-\ell}\frac{3^{\ell-1}}{\lambda^{8-\ell}}\Bigg\{\|I\|^\ell + \EE\left[\left\|(\lambda I + T)^{-1/2}(T-T_{\bX_1})(\lambda I + T)^{-1/2}\right\|_{HS}^{2\ell}\right]^{\frac{1}{2}} \nonumber\\
	&\quad + \EE\left[\left\|(\lambda I + T)^{-1/2}(T-T_{\bX_1})\right\|_{HS}^{2\ell}\right]\frac{1}{\lambda^{\ell}} \Bigg\}\nonumber\\
	&\overset{(\romannumeral4)}{\leq} (1+\mu_1)^{8-\ell}\frac{3^{\ell-1}}{\lambda^{8-\ell}}\left[1+ c^{\frac{1}{2}}(\ell)2^\ell\rho^{2\ell}\left(\frac{m\cN^2(\lambda)}{N}\right)^{\frac{\ell}{2}} + c(\ell)2^{2\ell}\rho^{4\ell}trace^{\ell}(T)\frac{1}{\lambda^\ell}\left(\frac{m\cN(\lambda)}{N}\right)^\ell\right]\nonumber\\
	&\leq (1+\mu_1)^83^7c(8)2^{16}\rho^{32}\max\left\{1,trace^8(T)\right\}\frac{1}{\lambda^{8-\ell}}\left[1+ \left(\frac{m\cN^2(\lambda)}{N}\right)^{\frac{\ell}{2}}+\frac{1}{\lambda^\ell}\left(\frac{m\cN(\lambda)}{N}\right)^\ell\right]\nonumber\\
	&= c_7^4\lambda^{\ell-8}\left[1+ \left(\frac{m\cN^2(\lambda)}{N}\right)^{\frac{\ell}{2}}+\lambda^{-\ell}\left(\frac{m\cN(\lambda)}{N}\right)^\ell\right],\nonumber
\end{align}
where $c_7^4:= (1+\mu_1)^83^7c(8)2^{16}\rho^{32}\max\left\{1,trace^8(T)\right\}$. Here equality (\romannumeral1) is from the second-order decomposition of $(\lambda I + T_{\bX_1})^{-1}$ \eqref{equation: second order decomposition}. Inequality (\romannumeral2) uses Jensen's inequality. Inequality (\romannumeral3) is due to Cauchy-Schwartz inequality and \eqref{relationship between L2 and Linfty norm}. Inequality (\romannumeral4) follows from estimations \eqref{equation: estimation of l-th HS-norm} and \eqref{equation: estimation of l-th HS-norm*} in Lemma \ref{lemma: estimation of l-th HS-norm}. 

Analogously, if $\ell\geq 8$, applying the second-order decomposition of $(\lambda I + T_{\bX_1})^{-1}$ \eqref{equation: second order decomposition} and taking $\lambda\leq 1$, we have
\begin{align}\label{equation: Theorem 3 5}
	&\EE\left[\left\|(\lambda I + T)^{1/2}\left(\lambda I + T_{\bX_1}\right)^{-1}(\lambda I + T)^{1/2}\right\|^8\right]\nonumber\\
	&\leq 3^7\left[1+ c^{\frac{1}{2}}(8)2^{8}\rho^{16}\left(\frac{m\cN^2(\lambda)}{N}\right)^{4}+c(8)2^{16}\rho^{32}trace^{8}(T)\frac{1}{\lambda^8}\left(\frac{m\cN(\lambda)}{N}\right)^8\right]\nonumber\\
	&\leq c_7^4\left[1+ \left(\frac{m\cN^2(\lambda)}{N}\right)^{4}+\frac{1}{\lambda^8}\left(\frac{m\cN(\lambda)}{N}\right)^8\right],
\end{align}
where $c_7^4= (1+\mu_1)^83^7c(8)2^{16}\rho^{32}\max\left\{1,trace^8(T)\right\}$.

We can now prove \eqref{equation: theorem stronger 1}. 

If $2\leq \ell <8$, taking $\lambda\leq 1$, for the term $\EE\left[\left\|L_C^{1/2}L_K^{1/2}(\lambda I + T_{\bX_1})^{-1}\frac{1}{n}\sum_{i=1}^n\alpha_{i}\right\|_{\cL^2}^2\II_{\cU_1}\right]$, combining \eqref{equation: Theorem 3 1} and \eqref{equation: Theorem 3 4} with \eqref{equation: stronger probability estimation of U1} in Lemma \ref{lemma: stronger probability esstimation of U1} and \eqref{equation: estimtion of 4-th norm1} in Lemma \ref{lemma: estimation of expectation of F(S,lambda)}, we have
\begin{align*}
	&\EE\left[\left\|L_C^{1/2}L_K^{1/2}\left(\lambda I + T_{\bX_1}\right)^{-1}\frac{1}{n}\sum_{i=1}^n\alpha_{i}\right\|_{\cL^2}^2\II_{\cU_1}\right]\\
	&\leq c_6c_7c^{\frac{1}{4}}(\ell)2^\ell\rho^\ell\lambda^{\frac{\ell-8}{4}}\left[1+ \left(\frac{m\cN^2(\lambda)}{N}\right)^{\frac{\ell}{2}}+\lambda^{-\ell}\left(\frac{m\cN(\lambda)}{N}\right)^\ell\right]^{\frac{1}{4}}\left(\frac{m\cN^2(\lambda)}{N}\right)^{\frac{\ell}{4}}\frac{m}{N}\left(1+\lambda^{4\theta}\cN^2(\lambda)\right)^{\frac{1}{2}}\\
	&\leq c_6c_7c^{\frac{1}{4}}(\ell)2^\ell\rho^\ell\lambda^{\frac{\ell-8}{4}}\left[1+ \left(\frac{m\cN^2(\lambda)}{N}\right)^{\frac{\ell}{8}}+\lambda^{-\frac{\ell}{4}}\left(\frac{m\cN(\lambda)}{N}\right)^{\frac{\ell}{4}}\right]\left(\frac{m\cN^2(\lambda)}{N}\right)^{\frac{\ell}{4}}\frac{m}{N}\left(1+\lambda^{2\theta}\cN(\lambda)\right)\\
	&= b_1(\ell)\lambda^{\frac{\ell-8}{4}}\left[1+ \left(\frac{m\cN^2(\lambda)}{N}\right)^{\frac{\ell}{8}}+\lambda^{-\frac{\ell}{4}}\left(\frac{m\cN(\lambda)}{N}\right)^{\frac{\ell}{4}}\right]\left(\frac{m\cN^2(\lambda)}{N}\right)^{\frac{\ell}{4}}\frac{m}{N}\left(1+\lambda^{2\theta}\cN(\lambda)\right),
\end{align*}
where $b_1(\ell):= c_6c_7c^{\frac{1}{4}}(\ell)2^\ell\rho^\ell$.

For the term $\EE\left[\left\|L_C^{1/2}L_K^{1/2}(\lambda I + T_{\bX_1})^{-1}L_K^{1/2}X_{1,i}\right\|_{\cL^2}^2\II_{\cU_1}\right], \forall 1\leq i\leq n$, combining \eqref{equation: Theorem 3 2}, \eqref{equation: Theorem 3 3} and \eqref{equation: Theorem 3 4} with \eqref{equation: stronger probability estimation of U1} in Lemma \ref{lemma: stronger probability esstimation of U1}, we have
\begin{align*}
	&\EE\left[\left\|L_C^{1/2}L_K^{1/2}(\lambda I + T_{\bX_1})^{-1}L_K^{1/2}X_{1,i}\right\|_{\cL^2}^2\II_{\cU_1}\right]\\
	&\leq c_7c^{\frac{1}{4}}(\ell)2^\ell\rho^{\ell+2}\lambda^{\frac{\ell-8}{4}}\left[1+ \left(\frac{m\cN^2(\lambda)}{N}\right)^{\frac{\ell}{8}}+\lambda^{-\frac{\ell}{4}}\left(\frac{m\cN(\lambda)}{N}\right)^{\frac{\ell}{4}}\right]\left(\frac{m\cN^2(\lambda)}{N}\right)^{\frac{\ell}{4}}\cN(\lambda)\\
	&= b_2(\ell)\lambda^{\frac{\ell-8}{4}}\left[1+ \left(\frac{m\cN^2(\lambda)}{N}\right)^{\frac{\ell}{8}}+\lambda^{-\frac{\ell}{4}}\left(\frac{m\cN(\lambda)}{N}\right)^{\frac{\ell}{4}}\right]\left(\frac{m\cN^2(\lambda)}{N}\right)^{\frac{\ell}{4}}\cN(\lambda),
\end{align*}
where $b_2(\ell):=c_7c^{\frac{1}{4}}(\ell)2^\ell\rho^{\ell+2}$.

Then recall that $n=N/m$, combining the above two estimations with \eqref{equation: theorem 3 total 1}, \eqref{equation: theorem 3 total 2} and \eqref{equation: theorem 3 total 3} yields
\begin{align*}
	&\EE\left[\left(\mathcal{R(\overline{\beta}_{S,\lambda})}-\mathcal{R}(\beta_0)\right)\right]\\
	&\leq 2\lambda^{2\theta}\|\gamma_0\|_{\cL^2}^2 + 16\frac{\cN(\lambda)}{N}\left(c_2\lambda^{2\theta}\|\gamma_0\|^2_{\cL^2}+\sigma^2\right) + 8c_2\frac{m}{N}\cN(\lambda)\lambda^{2\theta}\|\gamma_0\|_{\cL^2}^2\\
	&\quad + b_1(\ell)\lambda^{\frac{\ell-8}{4}}\left[1+ \left(\frac{m\cN^2(\lambda)}{N}\right)^{\frac{\ell}{8}}+\lambda^{-\frac{\ell}{4}}\left(\frac{m\cN(\lambda)}{N}\right)^{\frac{\ell}{4}}\right]\left(\frac{m\cN^2(\lambda)}{N}\right)^{\frac{\ell}{4}}\frac{4+2m}{N}\left(1+\lambda^{2\theta}\cN(\lambda)\right)\\
	&\quad + b_2(\ell)\lambda^{\frac{\ell-8}{4}}\left[1+ \left(\frac{m\cN^2(\lambda)}{N}\right)^{\frac{\ell}{8}}+\lambda^{-\frac{\ell}{4}}\left(\frac{m\cN(\lambda)}{N}\right)^{\frac{\ell}{4}}\right]\left(\frac{m\cN^2(\lambda)}{N}\right)^{\frac{\ell}{4}}\frac{4\sigma^2}{N}\cN(\lambda).
\end{align*}
This completes the proof of \eqref{equation: theorem stronger 1}.

We next give the proof of \eqref{equation: theorem stronger 2}.

If $\ell \geq 8$, taking $\lambda\leq 1$, utilizing \eqref{equation: Theorem 3 5} and following the same arguments in the proof of \eqref{equation: theorem stronger 1}, we obtain
\begin{align*}
	&\EE\left[\left\|L_C^{1/2}L_K^{1/2}\left(\lambda I + T_{\bX_1}\right)^{-1}\frac{1}{n}\sum_{i=1}^n\alpha_{i}\right\|_{\cL^2}^2\II_{\cU_1}\right]\\
	&\leq c_6c_7c^{\frac{1}{4}}(\ell)2^\ell\rho^\ell\left[1+ \frac{m\cN^2(\lambda)}{N}+\frac{1}{\lambda^{2}}\left(\frac{m\cN(\lambda)}{N}\right)^{2}\right]\left(\frac{m\cN^2(\lambda)}{N}\right)^{\frac{\ell}{4}}\frac{m}{N}\left(1+\lambda^{2\theta}\cN(\lambda)\right)\\
	&= b_1(\ell)\left[1+ \frac{m\cN^2(\lambda)}{N}+\frac{1}{\lambda^{2}}\left(\frac{m\cN(\lambda)}{N}\right)^{2}\right]\left(\frac{m\cN^2(\lambda)}{N}\right)^{\frac{\ell}{4}}\frac{m}{N}\left(1+\lambda^{2\theta}\cN(\lambda)\right),
\end{align*}
where $b_1(\ell)= c_6c_7c^{\frac{1}{4}}(\ell)2^\ell\rho^\ell$.

And
\begin{align*}
	&\EE\left[\left\|L_C^{1/2}L_K^{1/2}(\lambda I + T_{\bX_1})^{-1}L_K^{1/2}X_{1,i}\right\|_{\cL^2}^2\II_{\cU_1}\right]\\
	&\leq c_3c^{\frac{1}{4}}(\ell)2^\ell\rho^{\ell+2}\left[1+ \frac{m\cN^2(\lambda)}{N}+\frac{1}{\lambda^{2}}\left(\frac{m\cN(\lambda)}{N}\right)^{2}\right]\left(\frac{m\cN^2(\lambda)}{N}\right)^{\frac{\ell}{4}}\cN(\lambda)\\
	&= b_2(\ell)\left[1+ \frac{m\cN^2(\lambda)}{N}+\frac{1}{\lambda^{2}}\left(\frac{m\cN(\lambda)}{N}\right)^{2}\right]\left(\frac{m\cN^2(\lambda)}{N}\right)^{\frac{\ell}{4}}\cN(\lambda),
\end{align*}
where $b_2(\ell)=c_7c^{\frac{1}{4}}(\ell)2^\ell\rho^{\ell+2}$.

And then
\begin{align*}
	&\EE\left[\left(\mathcal{R(\overline{\beta}_{S,\lambda})}-\mathcal{R}(\beta_0)\right)\right]\\
	&\leq 2\lambda^{2\theta}\|\gamma_0\|_{\cL^2}^2 + 16\frac{\cN(\lambda)}{N}(c_2\lambda^{2\theta}\|\gamma_0\|^2_{\cL^2}+\sigma^2) + 8c_2\frac{m}{N}\cN(\lambda)\lambda^{2\theta}\|\gamma_0\|_{\cL^2}^2\\
	&\quad + b_1(\ell)\left[1+ \frac{m\cN^2(\lambda)}{N}+\frac{1}{\lambda^{2}}\left(\frac{m\cN(\lambda)}{N}\right)^{2}\right]\left(\frac{m\cN^2(\lambda)}{N}\right)^{\frac{\ell}{4}}\frac{4+2m}{N}\left(1+\lambda^{2\theta}\cN(\lambda)\right)\\
	&\quad + b_2(\ell)\left[1+ \frac{m\cN^2(\lambda)}{N}+\frac{1}{\lambda^{2}}\left(\frac{m\cN(\lambda)}{N}\right)^{2}\right]\left(\frac{m\cN^2(\lambda)}{N}\right)^{\frac{\ell}{4}}\frac{4\sigma^2}{N}\cN(\lambda).
\end{align*}
We have completed the proof of inequality \eqref{equation: theorem stronger 2}. The proof of Theorem \ref{theorem:stronger upper bound} is then finished.
\qed

We next prove Corollary \ref{corollary:corollary stronger}.

\noindent
{\bf Proof of Corollary \ref{corollary:corollary stronger}}.  We prove the desired bounds in three cases, respectively.

When $2\leq \ell\leq 4$, taking $\lambda\leq 1$, (\ref{equation: theorem stronger 1}) and (\ref{equation: estimation of N(lambda)}) implies
\begin{align}\label{equation: proof of corollary 5 1}
	&\EE\left[\left(\mathcal{R(\overline{\beta}_{S,\lambda})}-\mathcal{R}(\beta_0)\right)\right]\nonumber\\
	&\lesssim \lambda^{2\theta} + \frac{\lambda^{-p}}{N} + m\frac{\lambda^{2\theta-p}}{N}\\
	&\quad+ \lambda^{\frac{\ell-8}{4}}\left[1+ \left(\frac{m\lambda^{-2p}}{N}\right)^{\frac{\ell}{8}}+\lambda^{-\frac{\ell}{4}}\left(\frac{m\lambda^{-p}}{N}\right)^{\frac{\ell}{4}}\right]\left(\frac{m\lambda^{-2p}}{N}\right)^{\frac{\ell}{4}}\left(\frac{m}{N}+\frac{m\lambda^{2\theta-p}}{N}+\frac{\lambda^{-p}}{N}\right).\nonumber
\end{align}
As $\theta\leq \frac{1}{2}\leq \frac{p\ell+8}{4\ell}$, taking $m\leq N^r$ for some $0\leq r\leq\frac{2\theta}{2\theta+p}$, we have
\begin{align*}
	&\EE\left[\left(\mathcal{R(\overline{\beta}_{S,\lambda})}-\mathcal{R}(\beta_0)\right)\right]\lesssim \max\left\{N^{\frac{2\theta(4+\ell)(r-1)}{8+8\theta+2p\ell-\ell}},N^{\frac{2\theta\ell(r-1)-8\theta}{8+4p+8\theta+2p\ell-\ell}},N^{\frac{2\theta(4+2\ell)(r-1)}{8+8\theta+3p\ell}},N^{\frac{4\theta\ell(r-1)-8\theta}{8+4p+8\theta+3p\ell}}\right\}
\end{align*}
provided that $$\lambda=\max\left\{N^{\frac{(4+\ell)(r-1)}{8+8\theta+2p\ell-\ell}},N^{\frac{\ell(r-1)-4}{8+4p+8\theta+2p\ell-\ell}},N^{\frac{(4+2\ell)(r-1)}{8+8\theta+3p\ell}},N^{\frac{2\ell(r-1)-4}{8+4p+8\theta+3p\ell}}\right\}.$$
This completes the proof of case $2\leq \ell \leq 4$.

When $5\leq \ell\leq 7$, taking $\lambda\leq 1$, inequality \eqref{equation: proof of corollary 5 1} still holds. Then taking $\lambda= N^{-\frac{1}{2\theta+p}}$ yields 
\[\EE\left[\left(\mathcal{R(\overline{\beta}_{S,\lambda})}-\mathcal{R}(\beta_0)\right)\right]\lesssim N^{\frac{2\theta}{2\theta+p}}\]
provided that $\frac{p\ell+8}{4\ell}\leq\theta\leq \frac{1}{2}$ and $m\leq \min\left\{N^{\frac{8+p\ell-4p-4\theta\ell}{(4+2\ell)(2\theta+p)}}, N^{\frac{8+p\ell-8\theta-4\theta\ell}{(4+2\ell)(2\theta+p)}}\right\}$.

If $\theta\frac{p\ell+8}{4\ell}$, take $m\leq N^r$ for some $0\leq r\leq\frac{2\theta}{2\theta+p}$. Then we have
\begin{align*}
	&\EE\left[\left(\mathcal{R(\overline{\beta}_{S,\lambda})}-\mathcal{R}(\beta_0)\right)\right]\lesssim \max\left\{N^{\frac{2\theta(4+\ell)(r-1)}{8+8\theta+2p\ell-\ell}},N^{\frac{2\theta\ell(r-1)-8\theta}{8+4p+8\theta+2p\ell-\ell}},N^{\frac{2\theta(4+2\ell)(r-1)}{8+8\theta+3p\ell}},N^{\frac{4\theta\ell(r-1)-8\theta}{8+4p+8\theta+3p\ell}}\right\}
\end{align*}
provided that $$\lambda=\max\left\{N^{\frac{(4+\ell)(r-1)}{8+8\theta+2p\ell-\ell}},N^{\frac{\ell(r-1)-4}{8+4p+8\theta+2p\ell-\ell}},N^{\frac{(4+2\ell)(r-1)}{8+8\theta+3p\ell}},N^{\frac{2\ell(r-1)-4}{8+4p+8\theta+3p\ell}}\right\}.$$
This completes the proof of case $5\leq \ell \leq 7$.

When $\ell \geq 8$, taking $\lambda\leq 1$, (\ref{equation: theorem stronger 2}) and (\ref{equation: estimation of N(lambda)}) implies
\[\begin{aligned}
	&\EE\left[\left(\mathcal{R(\overline{\beta}_{S,\lambda})}-\mathcal{R}(\beta_0)\right)\right]\\
	&\lesssim \lambda^{2\theta} + \frac{\lambda^{-p}}{N} + m\frac{\lambda^{2\theta-p}}{N}\\
	&\quad+ \left(1+\frac{m\lambda^{-2p}}{N}+\frac{m^2\lambda^{-2p-2}}{N^2}\right)\left(\frac{m\lambda^{-2p}}{N}\right)^{\frac{\ell}{4}}\left(\frac{m}{N}+\frac{m\lambda^{2\theta-p}}{N}+\frac{\lambda^{-p}}{N}\right).
\end{aligned}\]
Taking $\lambda= N^{-\frac{1}{2\theta+p}}$ yields
\[\EE\left[\left(\mathcal{R(\overline{\beta}_{S,\lambda})}-\mathcal{R}(\beta_0)\right)\right]\lesssim N^{\frac{2\theta}{2\theta+p}}\]
provided that 
$\frac{p\ell + 8}{2\ell + 16}\leq \theta\leq 1/2$ and $m\leq \min\left\{N^{\frac{8+p\ell-4p-16\theta-2\theta\ell}{(12+\ell)(2\theta+p)}}, N^{\frac{8+p\ell-24\theta-2\theta\ell}{(12+\ell)(2\theta+p)}}\right\}$.

If $\theta<\frac{p\ell+8}{2\ell+16}$, take $m\leq N^r$ for some $0\leq r\leq \frac{2\theta}{2\theta+p}$. Then we have
\begin{align*}
	\EE\left[\left(\mathcal{R(\overline{\beta}_{S,\lambda})}-\mathcal{R}(\beta_0)\right)\right]\lesssim \max\left\{N^{\frac{\theta(4+\ell)(r-1)}{4\theta+p\ell}},N^{\frac{\theta\ell(r-1)-4\theta}{2p+4\theta+p\ell}},N^{\frac{\theta(12+\ell)(r-1)}{4+4p+4\theta+p\ell}},N^{\frac{\theta(8+\ell)(r-1)-4\theta}{4+6p+4\theta+p\ell}}\right\}
\end{align*}
provided that 
\[\lambda = \max\left\{N^{\frac{(4+\ell)(r-1)}{8\theta+2p\ell}},N^{\frac{\ell(r-1)-4}{4p+8\theta+2p\ell}},N^{\frac{(12+\ell)(r-1)}{8+8p+8\theta+2p\ell}},N^{\frac{(8+\ell)(r-1)-4}{8+12p+8\theta+2p\ell}}\right\}.\]
We have completed the proof of case $\ell\geq 8$. Then proof is then finished.
\qed

We next turn to prove Theorem \ref{theorem: extra upper bound} and Corollary \ref{corollary:corollary extra}. If Assumption \ref{momentcondition5} is satisfied, we can estimate the probability of event $\cU_1$ better than Lemma \ref{lemma: basic probability esstimation of U1} and Lemma \ref{lemma: stronger probability esstimation of U1}.
\begin{lemma}\label{lemma: strongest probability esstimation of U1}
	Suppose that Assumption \ref{momentcondition5} is satisfied, then there holds
	\begin{equation}\label{equation: strongest probability esstimation of U1}
		\PP(\cU_1)\leq c_4^4\left(1+ \frac{m^2\cN^2(\lambda)}{N^2}\right)\cN\left(\lambda\right)\exp\left(-c_5\frac{N}{m\cN(\lambda)}\right).
	\end{equation}
	Where $c_4$ and $c_5$ are universal constants and $\mathcal{N}(\lambda)$ is the effective dimension given by \eqref{effectivedimension}.
\end{lemma}
\begin{proof}
	Our proof relies on the Bernstein's inequality for the sum of self-adjoint random operators (see, Lemma \ref{lemma: Bernstein's inequality for self-adjoint r.o.s}).  Recalling \eqref{equation: simple notation}, define 
	\[\zeta_i := (\lambda I + T)^{-1/2}L_K^{1/2}X_{1,i}\otimes L_K^{1/2}X_{1,i}(\lambda I + T)^{-1/2} \quad and \quad \eta_i := \frac{1}{n}\left(\zeta_i - \EE[\zeta_i]\right),\quad 1\leq i\leq n.\] Then we can write
	\[(\lambda I + T)^{-1/2}(T_{\bX_1} - T)(\lambda I + T)^{-1/2}= \sum_{i=1}^n\eta_i.\]
	Using expression \eqref{PCdecomposition}, we have
	\begin{equation}\label{equation: corollary of strongest assumption}
		\begin{aligned}
			&\left\|(\lambda I + T)^{-1/2}L_K^{1/2}X\otimes L_K^{1/2}X(\lambda I + T)^{-1/2}\right\|\\
			&= \sup_{\|f\|_{\cL^2}=1,\|g\|_{\cL^2}=1}\left\langle(\lambda I + T)^{-1/2}L_K^{1/2}X\otimes L_K^{1/2}X(\lambda I + T)^{-1/2}f,g\right\rangle_{\cL^2}\\
			&\leq \left\|(\lambda I + T)^{-1/2}L_K^{1/2}X\right\|_{\cL^2}^2=\sum_{k=1}^\infty\frac{\mu_k}{\lambda +\mu_k}\xi_k^2\\
			&\overset{(*)}{\leq} \rho^2\sum_{k=1}^\infty\frac{\mu_k}{\lambda+\mu_k}=\rho^2\cN(\lambda).
		\end{aligned}
	\end{equation} Here inequality $(*)$ is from Assumption \ref{momentcondition5}.
	
	Then for any $1\leq i\leq n$, one can calculate
	\begin{align}\label{equation: lemma 9 1}
		\left\|\eta_i\right\|
		&=\left\|\frac{1}{n}\left(\zeta_i-\EE[\zeta_i]\right)\right\|
		\overset{(\romannumeral1)}{\leq} \frac{1}{n}\left\|\zeta_i\right\|+\frac{1}{n}\EE\left[\left\|\zeta_i\right\|\right]
		\overset{(\romannumeral2)}{\leq} 2\rho^2\frac{\cN(\lambda)}{n}.
	\end{align}
	Here inequality (\romannumeral1) uses the triangle inequality and Jensen's inequality. Inequality (\romannumeral2) follows from \eqref{equation: corollary of strongest assumption}. And then we have
	\begin{align}\label{equation: lemma 9 2}
		&\left\|\EE\left[\left(\overline{\eta}\right)^2\right]\right\|= \left\|\EE\left[\left(\sum_{i=1}^n\eta_i\right)^2\right]\right\|\nonumber\\
		&\overset{(\romannumeral1)}{=} \sup_{f\in \cL^2(\cT),\|f\|_{\cL^2}=1}\sum_{i=1}^n\left\langle f, \EE[\eta_i^2]f\right\rangle_{\cL^2}\nonumber\\
		&=\frac{1}{n^2}\sup_{f\in \cL^2(\cT),\|f\|_{\cL^2}=1}\sum_{i=1}^n\left(\left\langle f, \EE[\zeta_i^2]f\right\rangle_{\cL^2}-\left\langle f, [\EE\zeta_i]^2f\right\rangle_{\cL^2}\right)\nonumber\\
		&\leq \frac{1}{n^2}\sup_{f\in \cL^2(\cT),\|f\|_{\cL^2}=1}\sum_{i=1}^n\left\langle f, \EE[\zeta_i^2]f\right\rangle_{\cL^2}\\
		&= \frac{1}{n^2}\sup_{f\in \cL^2(\cT),\|f\|_{\cL^2}=1}\sum_{i=1}^n\EE\left[\left\langle (\lambda I + T)^{-1/2}L_K^{1/2}X_{1,i},f\right\rangle_{\cL^2}^2\left\|(\lambda I + T)^{-1/2}L_K^{1/2}X_{1,i}\right\|_{\cL^2}^2\right]\nonumber\\
		&\overset{(\romannumeral2)}{\leq} \rho^2\frac{\cN(\lambda)}{n^2}\sup_{f\in \cL^2(\cT),\|f\|_{\cL^2}=1}\sum_{i=1}^n\EE\left[\left(\sum_{k=1}^\infty\sqrt{\frac{1}{\lambda + \mu_k}}\left\langle f,\phi_k \right\rangle_{\cL^2}\left\langle L_K^{1/2}X_{1,i},\phi_k \right\rangle_{\cL^2}\right)^2\right]\nonumber\\
		&\overset{(\romannumeral3)}{=} \rho^2\frac{\cN(\lambda)}{n^2}\sup_{f\in \cL^2(\cT),\|f\|_{\cL^2}=1}\sum_{i=1}^n\sum_{k=1}^\infty\frac{\mu_k}{\lambda + \mu_k}\left\langle f,\phi_k\right\rangle_{\cL^2}^2\leq \rho^2\frac{\cN(\lambda)}{n}.\nonumber
	\end{align}
	Here equality (\romannumeral1) is due to the equivalent expression of the operator norm of a nonnegative operator \eqref{equation: equivalent definition of nonnegative operator} and the fact that $\EE[\eta_i] = 0$. Inequality (\romannumeral2) follows from the fact that $$\left\|(\lambda I + T)^{-1/2}L_K^{1/2}X_{1,i}\right\|_{\cL^2}^2\leq \rho^2\cN(\lambda)$$ which is given by \eqref{equation: corollary of strongest assumption}. Equality (\romannumeral3) is from the fact that $$\EE\left[\left\langle L_K^{1/2}X,\phi_j\right\rangle_{\cL^2}\left\langle L_K^{1/2}X,\phi_k\right\rangle_{\cL^2}\right] = \langle T\phi_j,\phi_k\rangle_{\cL^2}=\mu_k\delta_j^k,$$
	We also need the following estimate given by 
	\begin{align}\label{equation: lemma 9 3}
		trace\left(\EE\left[\left(\overline{\eta}\right)^2\right]\right)&\overset{(\romannumeral1)}{=} \sum_{k=1}^\infty\left\langle \EE\left[\left(\sum_{i=1}^n\eta_i\right)^2\right]\phi_k,\phi_k \right\rangle_{\cL^2}\nonumber\\
		&= \sum_{i=1}^n\sum_{k=1}^\infty\left\langle \EE\left[\eta_i^2\right]\phi_k,\phi_k \right\rangle_{\cL^2}\leq \frac{1}{n^2}\sum_{i=1}^n\sum_{k=1}^\infty\left\langle \EE\left[\zeta_i^2\right]\phi_k,\phi_k \right\rangle_{\cL^2}\nonumber\\
		&= \frac{1}{n^2}\sum_{i=1}^n\sum_{k=1}^\infty\frac{1}{\lambda + \mu_k}\EE\left[\left\|(\lambda I + T)^{-1/2}L_K^{1/2}X_{1,i}\right\|_{\cL^2}^2\left\langle L_K^{1/2}X_{1,i},\phi_k \right\rangle_{\cL^2}^2\right]\nonumber\\
		&\overset{(\romannumeral2)}{\leq} \rho^2\frac{\cN(\lambda)}{n^2}\sum_{i=1}^n\sum_{k=1}^\infty\frac{1}{\lambda + \mu_k}\EE\left[\left\langle L_K^{1/2}X_{1,i},\phi_k \right\rangle_{\cL^2}^2\right]\overset{(\romannumeral3)}{\leq}\rho^2\frac{\cN^2(\lambda)}{n}.
	\end{align}
	Here equality (\romannumeral1) is from the formulation of the trace norm of an operator \eqref{equation: trace class}. Inequality (\romannumeral2) is due to the fact that $\left\|(\lambda I + T)^{-1/2}L_K^{1/2}X_{1,i}\right\|_{\cL^2}^2\leq \rho^2\cN(\lambda)$. Inequality (\romannumeral3)  follows from the calculation that $\sum_{k=1}^\infty\frac{1}{\lambda+\mu_k}\EE\left[\left\langle L_K^{1/2}X_{1,i},\phi_k\right\rangle_{\cL^2}^2\right]= \frac{\mu_k}{\lambda+\mu_k}=\cN(\lambda)$.

	Recall that $n=N/m$. Based on (\ref{equation: lemma 9 1}), (\ref{equation: lemma 9 2}) and (\ref{equation: lemma 9 3}), one can apply Lemma \ref{lemma: Bernstein's inequality for self-adjoint r.o.s} with $L=2\rho^2\frac{m\cN(\lambda)}{N}$, $v=\rho^2\frac{m\cN(\lambda)}{N}$, $d = \cN(\lambda)$ and $s=1/2$ to obtain
	\begin{align}\nonumber
		\PP(\cU_1)&=\PP\left(\left\|\sum_{i=1}^n\eta_i\right\|\geq 1/2\right)\nonumber\\
		&\leq \left[1+6\left(\rho^2\frac{4m\cN(\lambda)}{N}+ \rho^2\frac{4m\cN(\lambda)}{3N}\right)^2\right]\cN(\lambda) \exp\left(-\frac{3N}{32\rho^2m\cN(\lambda)}\right)\nonumber\\
		&\leq c_4^4\left(1+ \frac{m^2\cN^2(\lambda)}{N^2}\right)\cN(\lambda)\exp\left(-c_5\frac{N}{m\cN(\lambda)}\right),\nonumber
	\end{align}
	where $c_4^4:= \left[1+ 6\left(4\rho^2+\frac{4\rho^2}{3}\right)^2\right]$ and $c_5:=\frac{3}{32\rho^2}$. The proof is then completed.
\end{proof}

Now we can prove Theorem \ref{theorem: extra upper bound}.

\noindent
{\bf Proof of Theorem \ref{theorem: extra upper bound}}.
Under Assumption \ref{momentcondition5}, there holds
\begin{align}\label{equation: theorem 4 *}
	\left\|L_K^{1/2}X\right\|_{\cL^2}=\left(\sum_{k=1}^\infty\mu_k \xi_k^2\right)^{\frac{1}{2}}\leq \rho \left(\sum_{k=1}^\infty\mu_k\right)^{\frac{1}{2}}=\rho\cdot trace^{\frac{1}{2}}(T).
\end{align}
Therefore, recalling \eqref{equation: simple notation}, for any $1\leq i\leq n$, we can write
\begin{align}\label{equation: theorem 4 1}
	&\EE\left[\left\|L_C^{1/2}L_K^{1/2}(\lambda I + T_{\bX_1})^{-1}L_K^{1/2}X_{1,i}\right\|_{\cL^2}^2\II_{\cU_1}\right]\nonumber\\
	&\leq \frac{\mu_1}{\lambda^2}\EE\left[\left\|L_K^{1/2}X_{1,i}\right\|_{\cL^2}^2\II_{\cU_1}\right]\overset{(\romannumeral1)}{\leq} \frac{\mu_1}{\lambda^2}\left[\EE\left\|L_K^{1/2}X_{1,i}\right\|_{\cL^2}^4\right]^{\frac{1}{2}}\PP^{\frac{1}{2}}(\cU_1)\\
	&\overset{(\romannumeral2)}{\leq} c_4\mu_1\rho^2trace(T)\frac{1}{\lambda^2}\left(1+\frac{m\cN(\lambda)}{N}\right)\cN^{\frac{1}{2}}(\lambda)\exp\left(-\frac{c_5N}{2m\cN(\lambda)}\right).\nonumber
\end{align}
Here inequality (\romannumeral1) is from Cauchy-Schwartz inequality. Inequality (\romannumeral2) follows from \eqref{equation: strongest probability esstimation of U1} in Lemma \ref{lemma: strongest probability esstimation of U1} and \eqref{equation: theorem 4 1}.

Then under Assumption \ref{assumption1} and \ref{momentcondition5}, one can calculate
\begin{align}\nonumber
	&\EE\left[\left\|L_K^{1/2}X\left\langle X,\beta_0-L_K^{1/2}f_\lambda\right\rangle_{\cL^2}\right\|_{\cL^2}^4 \right]\nonumber\\
	&\overset{(\romannumeral1)}{\leq}  \rho^4trace^2(T)\EE\left[\left\langle X,\beta_0-L_K^{1/2}f_\lambda \right\rangle_{\cL^2}^4 \right]\nonumber\\
	&\overset{(\romannumeral2)}{\leq} c_1\rho^4trace^2(T)\left[\EE\left\langle X,\beta_0-L_K^{1/2}f_\lambda \right\rangle_{\cL^2}^2\right]^2\nonumber\\
	&= c_1\rho^4trace^2(T) \left\|L_C^{1/2}\left(\beta_0-L_K^{1/2}f_\lambda\right)\right\|_{\cL^2}^4\nonumber\\
	&\overset{(\romannumeral3)}{\leq} c_1\rho^4trace^2(T)\|\gamma_0\|_{\cL^2}^4\lambda^{4\theta},\nonumber
\end{align}
where inequality (\romannumeral1) follows from \eqref{equation: theorem 4 1}, inequality (\romannumeral2) uses the fourth-moment condition \eqref{momentcondition2} and inequality (\romannumeral3) is due to Lemma \ref{lemma: estimation of A(lambda)}. Then utilizing the above bound and following the same estimates in the proof of \eqref{equation: estimtion of 4-th norm1}, taking $\lambda\leq 1$, we obtain
\begin{equation}\label{equation: theorem 4 pro2}
	\begin{aligned}
		&\EE\left[\left\|\frac{1}{n}\sum_{i=1}^n\left(L_K^{1/2}X_{1,i}\left\langle X_{1,i},\beta_0-L_K^{1/2}f_\lambda\right\rangle_{\cL^2}-\lambda f_\lambda\right)\right\|_{\cL^2}^4\right]\\
		&\leq \frac{192m^2}{N^2}\left(c_1\rho^4trace^2(T) \|\gamma_0\|_{\cL^2}^4\lambda^{4\theta}+ \mu_1^{4\theta}\|\gamma_0\|_{\cL^2}^4\lambda^2\right)\leq c_3^2\frac{m^2}{N^2}\lambda^{4\theta},
	\end{aligned}
\end{equation}
where $c_3:=  192\left(c_1\rho^4trace^2(T) ||\gamma_0||_{\cL^2}^4+\max\{\mu_1^{2},1\}||\gamma_0||_{\cL^2}^4\right)$.

For simplicity of notations, define 
\[\alpha_i := L_K^{1/2}X_{1,i}\left\langle X_{1,i},\beta_0-L_K^{1/2}f_\lambda\right\rangle_{\cL^2}-\lambda f_\lambda, \quad i=1,2,\cdots,n.\]
Then we can write
\begin{align}\label{equation: theorem 4 2}
	&\EE\left[\left\|L_C^{1/2}L_K^{1/2}(\lambda I +T_{\bX_1})^{-1}\frac{1}{n}\sum_{i=1}^n\alpha_{i}\right\|_{\cL^2}^2\II_{\cU_1}\right]\nonumber\\
	&\leq \frac{\mu_1}{\lambda^2}\EE\left[\left\|\frac{1}{n}\sum_{i=1}^n\alpha_{i}\right\|_{\cL^2}^2\II_{\cU_1}\right]\overset{(\romannumeral1)}{\leq}\frac{\mu_1}{\lambda^2}\left[\EE\left\|\frac{1}{n}\sum_{i=1}^n\alpha_{i}\right\|_{\cL^2}^4\right]^{\frac{1}{2}}\PP^{\frac{1}{2}}(\cU_1)\\
	&\overset{(\romannumeral2)}{\leq} c_3c_4\mu_1\frac{m}{N\lambda^{2-2\theta}}\left(1+\frac{m\cN(\lambda)}{N}\right)\cN^{\frac{1}{2}}(\lambda)\exp\left(-\frac{c_5N}{2m\cN(\lambda)}\right).\nonumber
\end{align}
Here inequality (\romannumeral1) is due to Cauchy-Schwartz inequality. Inequality (\romannumeral2) is from \eqref{equation: strongest probability esstimation of U1} in Lemma \ref{lemma: strongest probability esstimation of U1} and \eqref{equation: theorem 4 pro2}.

Finally, utilizing estimates \eqref{equation: theorem 4 1} and \eqref{equation: theorem 4 2}, we follow the same arguments in the proof of \eqref{equation: theorem stronger 1} and then obtain 
\begin{align*}
	&\EE\left[\left(\mathcal{R(\overline{\beta}_{S,\lambda})}-\mathcal{R}(\beta_0)\right)\right]\\
	&\leq 2\lambda^{2\theta}\|\gamma_0\|_{\cL^2}^2 + 16\frac{\cN(\lambda)}{N}\left(c_1\lambda^{2\theta}\|\gamma_0\|^2_{\cL^2}+\sigma^2\right) + 8c_1\frac{m}{N}\cN(\lambda)\lambda^{2\theta}\|\gamma_0\|_{\cL^2}^2\\
	&\quad + c_3c_4\mu_1\frac{4+2m}{N\lambda^{2-2\theta}}\left(1+\frac{m\cN(\lambda)}{N}\right)\cN^{\frac{1}{2}}(\lambda)\exp\left(-\frac{c_5N}{2m\cN(\lambda)}\right)\\
	&\quad + c_4\mu_1\rho^2trace(T)\frac{4\sigma^2}{N\lambda^2}\left(1+\frac{m\cN(\lambda)}{N}\right)\cN^{\frac{1}{2}}(\lambda)\exp\left(-\frac{c_5N}{2m\cN(\lambda)}\right)
\end{align*} 
The proof of Theorem \ref{theorem: extra upper bound} is then finished.\qed

We next turn to prove Corollary \ref{corollary:corollary extra}.

\noindent
{\bf Proof of Corollary \ref{corollary:corollary extra}}.
Taking $m\leq o\left(\frac{N^{\frac{2\theta}{2\theta + p}}}{\log N}\right)$ and $\lambda=N^{-\frac{1}{2\theta+p}}$, \eqref{equation: estimation of N(lambda)} implies
\[m\frac{\cN(\lambda)}{N}\lesssim m\frac{\lambda^{-p}}{N}\leq o\left(\frac{1}{\log N}\right).\]
Therefore, for any $r>0$, there holds
\[\liminf_{N\rightarrow \infty} N^r\exp\left(-\frac{c_5N}{2m\cN(\lambda)}\right) = 0.\]
Then using Theorem \ref{theorem: extra upper bound}, we obtain
\[\EE\left[\mathcal{R}(\overline{\beta}_{S,\lambda}) - \mathcal{R}(\beta_0)\right] \lesssim \lambda^{2\theta}+ \frac{\cN(\lambda)}{N}+ m\frac{\cN(\lambda)}{N}\lambda^{2\theta}\lesssim N^{-\frac{2\theta}{2\theta+p}}.\]
The proof of Corollary \ref{corollary:corollary extra} is then finished.
\qed

Finally, we will provide the proof of Theorem \ref{theorem: theorem for comparision}. Before that, we establish the following lemma to estimate $\PP(\cU_1)$ based on Assumption \ref{momentcondition6}. 

\begin{lemma}\label{lemma: theorem 6 probability}
	Suppose that Assumption \ref{momentcondition6} is satisfied, then there holds
	\begin{equation}\label{equation: theorem 6 probability}
		\PP(\cU_1)\leq \left[1+6\left(\frac{4m\kappa^2}{N\lambda^t}+\frac{4m\kappa^2}{3N\lambda^t}\right)^2\right]\cN(\lambda) \exp\left(-\frac{3N\lambda^t}{32m\kappa^2}\right).
	\end{equation}
\end{lemma}
\begin{proof}
	Our proof relies on Lemma \ref{lemma: Bernstein's inequality for self-adjoint r.o.s}. Recalling \eqref{equation: simple notation} and the definition of $\cU_1$ given by
	\[\cU_1 = \left\{\bX_1: \left\|(\lambda I + T)^{-1/2}(T_{\bX_1}- T)(\lambda I + T)^{-1/2}\right\| \geq 1/2\right\},\] let \[\zeta_i := (\lambda I + T)^{-1/2}L_K^{1/2}X_{1,i}\otimes L_K^{1/2}X_{1,i}(\lambda I + T)^{-1/2} \quad and \quad \eta_i := \frac{1}{n}\left(\zeta_i - \EE[\zeta_i]\right),\quad 1\leq i\leq n.\]
	Then 
	\[(\lambda I + T)^{-1/2}(T_{\bX_1} - T)(\lambda I + T)^{-1/2}= \sum_{i=1}^n\eta_i.\]
	Due to the decomposition of $L_K^{1/2}X$ in \eqref{PCdecomposition}, there holds
	\begin{align}\label{equation: corollary of moment condition 6}
		&\left\|(\lambda I + T)^{-1/2}L_K^{1/2}X\otimes L_K^{1/2}X(\lambda I + T)^{-1/2}\right\|\nonumber\\
		&=\sum_{\|f\|_{\cL^2}=1,\|g\|_{\cL^2}=1}\left\langle (\lambda I + T)^{-1/2}L_K^{1/2}X\otimes L_K^{1/2}X(\lambda I + T)^{-1/2}f,g\right\rangle_{\cL^2}\nonumber\\
		&\leq\left\|(\lambda I + T)^{-1/2}L_K^{1/2}X\right\|_{\cL^2}^2
		= \sum_{k=1}^\infty\frac{\mu_k^{1-t}\mu_k^t\xi_k^2}{\lambda + \mu_k}\leq \frac{1}{\lambda^t}\sum_{k=1}^\infty\mu_k^t\xi_k^2
		\overset{(*)}{\leq} \frac{\kappa^2}{\lambda^t}.
	\end{align}
	Here inequality $(*)$ follows from Assumption \ref{momentcondition6}. Then for any $1\leq i\leq n$, 
	\begin{align}\label{equation: theorem 6 1}
		\|\eta_i\|
		=\left\|\frac{1}{n}\left(\zeta_i-\EE[\zeta_i]\right)\right\|
		\leq \frac{1}{n}\|\zeta_i\|+\frac{1}{n}\EE\left[\left\|\zeta_i\right\|\right]
		\overset{(\dagger)}{\leq} \frac{2\kappa^2}{n}.
	\end{align}
	Here inequality $(\dagger)$ follows from (\ref{equation: corollary of moment condition 6}) and the definition that $\zeta_i = (\lambda I + T)^{-1/2}L_K^{1/2}X_{1,i}\otimes L_K^{1/2}X_{1,i}(\lambda I + T)^{-1/2}$.
	
	Following the same arguments of \eqref{equation: lemma 9 2}, we have
	\begin{align}	\label{equation: theorem 6 2}
		&\left\|\EE\left[(\overline{\eta})^2\right]\right\|\nonumber\\
		&\leq\frac{1}{n^2}\sup_{f\in \cL^2(\cT),\|f\|_{\cL^2}=1}\sum_{i=1}^n\EE\left[\left\langle (\lambda I + T)^{-1/2}L_K^{1/2}X_{1,i},f\right\rangle_{\cL^2}^2\left\|(\lambda I + T)^{-1/2}L_K^{1/2}X_{1,i}\right\|_{\cL^2}^2\right]\nonumber\\
		&= \frac{1}{n}\sup_{f\in \cL^2(\cT),\|f\|_{\cL^2}=1}\EE\left[\left\langle (\lambda I + T)^{-1/2}L_K^{1/2}X,f\right\rangle_{\cL^2}^2\left\|(\lambda I + T)^{-1/2}L_K^{1/2}X\right\|_{\cL^2}^2\right]\\
		&\overset{(\romannumeral1)}{\leq} \frac{\kappa^2}{n\lambda^t}\sup_{f\in \cL^2(\cT),\|f\|_{\cL^2}=1}\EE\left[ \left(\sum_{k=1}^\infty\sqrt{\frac{\mu_k}{\lambda + \mu_k}}\langle f,\phi_k \rangle_{\cL^2}\xi_k\right)^2 \right]\nonumber\\
		&\overset{(\romannumeral2)}{\leq} \frac{\kappa^2}{n\lambda^t}\sup_{f\in \cL^2(\cT),\|f\|_{\cL^2}=1}\sum_{k=1}^\infty\frac{\mu_k}{\lambda+\mu_k}\langle f,\phi_k \rangle_{\cL^2}^2\leq\frac{\kappa^2}{n\lambda^t}.\nonumber
	\end{align}
	Here inequality (\romannumeral1) is from \eqref{equation: corollary of moment condition 6} and inequality (\romannumeral2) is due to the decomposition of $L_K^{1/2}X$ in \eqref{PCdecomposition} and the fact that $\EE[\xi_j\xi_k]=\delta_j^k$. And following the same arguments of \eqref{equation: lemma 9 3}, we have
	\begin{align}\label{equation: theorem 6 3}
		&trace(\EE[(\overline{\eta})^2])\nonumber\\
		&\leq \frac{1}{n^2}\sum_{i=1}^n\sum_{k=1}^\infty\frac{1}{\lambda + \mu_k}\EE\left[\left\|(\lambda I + T)^{-1/2}L_K^{1/2}X_{1,i}\right\|_{\cL^2}^2\langle L_K^{1/2}X_{1,i},\phi_k \rangle_{\cL^2}^2\right]\nonumber\\
		&= \frac{1}{n}\sum_{k=1}^\infty \frac{1}{\lambda + \mu_k}\EE\left[\left\|(\lambda I + T)^{-1/2}L_K^{1/2}X\right\|_{\cL^2}^2\langle L_K^{1/2}X,\phi_k \rangle_{\cL^2}^2\right]\nonumber\\
		&\overset{(\dagger)}{\leq} \frac{\kappa^2}{n\lambda^t}\sum_{k=1}^\infty\frac{\mu_k}{\lambda+\mu_k}\EE\left[\xi_k^2\right]\leq \frac{\kappa^2\cN(\lambda)}{n\lambda^t}.
	\end{align}
	Here inequality $(\dagger)$ is from \eqref{equation: corollary of moment condition 6} and decomposition of $L_K^{1/2}X$ \eqref{PCdecomposition}.
	
	Recall that $n=N/m$. \eqref{equation: theorem 6 1}, \eqref{equation: theorem 6 2} and \eqref{equation: theorem 6 3} imply that we can employ Lemma \ref{lemma: Bernstein's inequality for self-adjoint r.o.s} with $L= \frac{2m\kappa^2}{N\lambda^t}$, $v = \frac{m\kappa^2}{N\lambda^t}$, $d= \cN(\lambda)$ and $s=1/2$ to obtain
	\begin{equation*}
		\begin{aligned}
			\PP(\cU_1)
			&=\PP\left(\left\|\sum_{i=1}^n\eta_i\right\|\geq 1/2\right)\\
			&\leq \left[1+6\left(\frac{4m\kappa^2}{N\lambda^t}+\frac{4m\kappa^2}{3N\lambda^t}\right)^2\right]\cN(\lambda) \exp\left(-\frac{3N\lambda^t}{32m\kappa^2}\right).
		\end{aligned}
	\end{equation*}
	This completes the proof.
\end{proof}

Now we are ready to prove Theorem \ref{theorem: theorem for comparision}.

\noindent
{\bf Proof of Theorem \ref{theorem: theorem for comparision}}. Recalling \eqref{equation: simple notation} and $Y_{1,i}=\langle X_{1,i},\beta_0\rangle_{\cL^2}+\epsilon_{1,i}, \forall 1\leq i\leq n$, we define
\[\alpha_i = L_K^{1/2}X_{1,i}\left\langle X_{1,i},\beta_0-L_K^{1/2}f_\lambda\right\rangle_{\cL^2}-\lambda f_\lambda, \quad i=1,2,\cdots,n.\]
Under Assumption \ref{momentcondition6}, we have
\begin{align*}
	\left\|L_K^{1/2}X\right\|_{\cL^2}=\left(\sum_{k=1}^\infty\mu_k\xi_k^2\right)^{\frac{1}{2}}\leq \mu_1^{\frac{1-t}{2}}\left(\sum_{k=1}^\infty\mu_k^t\xi_k^2\right)^{\frac{1}{2}}\leq \mu_1^{\frac{1-t}{2}}\kappa.
\end{align*}
Utilizing the above estimation and \eqref{equation: theorem 6 probability}, following the same arguments in the proof of Theorem \ref{theorem: extra upper bound}, we have
\begin{align*}
	&\EE\left[\left\|L_C^{1/2}L_K^{1/2}(\lambda I + T_{\bX_1})^{-1}L_K^{1/2}X_{1,i}\right\|_{\cL^2}^2\II_{\cU_1}\right]\nonumber\\
	&\lesssim \frac{1}{\lambda^2}\left(1+\frac{m}{N\lambda^t}\right)\cN^{\frac{1}{2}}(\lambda)\exp\left(-\frac{3N\lambda^t}{64m\kappa^2}\right)
\end{align*}
and
\begin{align*}
	&\EE\left[\left\|L_C^{1/2}L_K^{1/2}(\lambda I +T_{\bX_1})^{-1}\frac{1}{n}\sum_{i=1}^n\alpha_{i}\right\|_{\cL^2}^2\II_{\cU_1}\right]\nonumber\\
	&\lesssim \frac{m}{N\lambda^{2-2\theta}}\left(1+\frac{m}{N\lambda^t}\right)\cN^{\frac{1}{2}}(\lambda)\exp\left(-\frac{3N\lambda^t}{64m\kappa^2}\right)
\end{align*}
and then
\begin{align}\label{equation: Theorem 6 1}
	\EE\left[\left(\mathcal{R(\overline{\beta}_{S,\lambda})}-\mathcal{R}(\beta_0)\right)\right]&\lesssim \lambda^{2\theta}+ \frac{\cN(\lambda)}{N} + \frac{m}{N}\cN(\lambda)\lambda^{2\theta}\nonumber\\
	&\quad + \frac{m}{N\lambda^{2-2\theta}}\left(1+\frac{m}{N\lambda^t}\right)\cN^{\frac{1}{2}}(\lambda)\exp\left(-\frac{3N\lambda^t}{64m\kappa^2}\right)\\
	&\quad + \frac{1}{N\lambda^2}\left(1+\frac{m}{N\lambda^t}\right)\cN^{\frac{1}{2}}(\lambda)\exp\left(-\frac{3N\lambda^t}{64m\kappa^2}\right).\nonumber
\end{align}

Recall that $\{\mu_k\}_{k\geq 1}$ satisfy $\mu_k\lesssim k^{-1/p}$ for some $0< p \leq 1$.  

When $\max\{0, t/2-p/2\}\leq\theta\leq 1/2$, taking $m\leq o\left(\frac{N^{\frac{2\theta+p-t}{2\theta+p}}}{\log N}\right)$ and $\lambda = N^{-\frac{1}{2\theta+p}}$ yields that for any $r>0$, there holds\[\limsup_{N\rightarrow \infty} N^r\exp\left(-\frac{3N\lambda^t}{64m\kappa^2}\right)= 0,\mbox{ as } \frac{m}{N\lambda^t}\leq o\left(\frac{1}{\log N}\right).\]
Then combining with \eqref{equation: estimation of N(lambda)} and \eqref{equation: Theorem 6 1}, we have
\begin{align*}
	\EE\left[\left(\mathcal{R(\overline{\beta}_{S,\lambda})}-\mathcal{R}(\beta_0)\right)\right]
	&\lesssim \lambda^{2\theta}+ \frac{\lambda^{-p}}{N} + \frac{m\lambda^{2\theta-p}}{N}\lesssim N^{-\frac{2\theta}{2\theta+p}}
\end{align*}

When $\theta<\max\{0, t/2-p/2\}$ which implies $t>p>0$, taking $m\leq o\left(\log N\right)$ and $\lambda= N^{-\frac{1}{t}}(\log N)^{-\frac{2}{t}}$ yields that for any $r>0$, there holds\[\limsup_{N\rightarrow \infty} N^r\exp\left(-\frac{3N\lambda^t}{64m\kappa^2}\right)= 0,\mbox{ as } \frac{m}{N\lambda^t}\leq o\left(\frac{1}{\log N}\right).\]
Then combining with \eqref{equation: estimation of N(lambda)} and \eqref{equation: Theorem 6 1}, we have
\begin{align*}
	\EE\left[\left(\mathcal{R(\overline{\beta}_{S,\lambda})}-\mathcal{R}(\beta_0)\right)\right]
	\lesssim \lambda^{2\theta}+ \frac{\lambda^{-p}}{N} + \frac{m\lambda^{2\theta-p}}{N}\lesssim N^{-\frac{2\theta}{t}}(\log N)^{-\frac{4\theta}{t}}.
\end{align*}
We have completed the proof of Theorem \ref{theorem: theorem for comparision}.
\qed

\subsection{Mini-max Lower Rates}\label{subsection: lower rates}

In this subsection, we establish the lower bound in Theorem \ref{theorem:lower bound}. Before that, we present some crucial results used in our proof. Our analysis of lower bound bases on Fano's method, which provides lower bound in nonparametric estimation problem and was proposed by \cite{khas1979lower}. Fano's method has been a crucial method in minimax lower bound estimation problem since it was proposed, and has inspired many following studies (see e.g., \cite{yang1999information,guntuboyina2011lower,candes2013well}). The following lemma is a direct application of Fano's method (see, \cite{yang1999information}). To this end, recall that the Kullback-Leibler divergence (KL-divergence) of two probability measures $P,Q$ on a general space $(\Omega,\mathscr{F})$ is defined as
\[D_{kl}(P\|Q):= \int_{\Omega}\log\left(\frac{dP}{dQ}\right)dP,\]
if $P$ is absolutely continuous with respect to $Q$, and otherwise $D_{kl}(P\|Q):=\infty$. Recall that for $\beta \in \cL^2(\cT)$, $L^{1/2}_C\beta \in \mathrm{ran}T^{\theta}_*$ if $\beta$ satisfied the regularity condition \eqref{regularity condition}, i.e., 
\begin{equation*}
	L_C^{1/2}\beta = T_*^\theta(\gamma) \mbox{ with $0<\theta\leq 1/2$ and some $\gamma \in \mathcal{L}^2(\cal T)$}.
\end{equation*}

\begin{lemma}\label{lemma: lower bound 1} 
	Suppose that there exist constants $r,R>0$ and $\beta_1,\beta_2,\cdots,\beta_L\in \cL^2(\cT)$ for some integer $L\geq 2$, such that  
	\begin{equation}\label{hypothesis}
		L^{1/2}_C\beta_i \in \mathrm{ran}T^{\theta}_*,\ \left \|L_C^{1/2}(\beta_i-\beta_j)\right \|_{\cL^2}\geq 2r \mbox{  and  } D_{kl}(P_i\|P_j)\leq R, \ \forall 1 \leq i\neq j \leq L,
	\end{equation}where $P_i$ denotes the joint probability distribution of $(X, Y)$ with 
	\[Y=\int_{\mathcal{T}} \beta_i(t)X(t)dt + \epsilon.\] Here $\epsilon$ is independent of $X$ satisfying $\mathbb{E}[\epsilon]=0$ and $\mathbb{E}[\epsilon^2]\leq \sigma^2$. Then we have 
	\begin{equation}\label{equation: lower bound lemma1 final}
		\inf_{\hat{\beta}_S}\sup_{\beta_0}\mathbb{P}\left\{\mathcal{R}(\hat{\beta}_S)-\mathcal{R}(\beta_0)\geq r^2 \right\}\geq 1- \frac{NR+\log2}{\log L},
	\end{equation}
	where the supremum is taken over all $\beta_0 \in \mathcal{L}^2(\cal T)$ satisfying $L^{1/2}_C\beta_0 \in \mathrm{ran}T^{\theta}_*$  and the infimum is taken over all possible predictors $\hat{\beta}_S \in \mathcal{L}^2(\cal T)$ based on the training sample set $S=\{(X_i,Y_i)\}_{i=1}^N$ consisting of independent copies of $(X,Y)$ with 
	\[Y=\int_{\mathcal{T}} \beta_0(t)X(t)dt + \epsilon.\] 
\end{lemma}

In the followings, we first construct a family of $\{\beta_i\}_{i=1}^L$ satisfying \eqref{hypothesis} with suitable $r,R$ and $L$, and then apply Lemma \ref{lemma: lower bound 1} to establish the lower bound. Note that establishing a lower bound for a particular instance directly provides a lower bound for the general scenario. Consequently, it is adequate to examine the situation where $\epsilon$ represents a Gaussian random variable with zero mean, characterized by $\mathbb{E}[\epsilon^2]=\sigma^2$. The following lemma is from the formulation of KL-divergence of two Gaussian distribution (see, e.g., example 2.7 of \cite{duchi2016lecture}), which can further facilitate the calculation.
\begin{lemma}\label{lemma: lower bound 3}
	Suppose that $\epsilon$ is a zero-mean Gaussian random variable independent of $X$ satisfying $\mathbb{E}[\epsilon^2]=\sigma^2$>0. For $\beta_i \in \cL^2(\cT), i=1,2$, let $P_i$ denote the joint probability distribution of $(X, Y)$ with 
	\[Y=\int_{\mathcal{T}} \beta_i(t)X(t)dt + \epsilon.\] Then
	\begin{equation}\label{equation: lower bound lemma3}
		D_{kl}(P_1\|P_2)=\frac{1}{2\sigma^2}\left\|L_C^{1/2}(\beta_1-\beta_2)\right\|_{\cL^2}^2.
	\end{equation}
\end{lemma}

Our construction of $\{\beta_i\}_{i=1}^L$ relies on the following lemma which is known as Gilbert-Varshamov bound (see Lemma 7.5 in \cite{duchi2016lecture}).
\begin{lemma}\label{lemma: lower bound 2}
	Let $M\geq 8$. There exists a subset $\Lambda \subset \mathcal{H}_M = \{-1,1\}^M$ of size $|\Lambda| \geq \exp(M/8)$ such that 
	\[\begin{aligned}
		\left\|\iota - \iota'\right\|_1 = 2\sum_{i=1}^{M}\mathbb{I}_{\{\iota_j \neq \iota'_j\}}\geq M/2
	\end{aligned}\] for any $\iota \neq \iota'$ with  $\iota,\iota'\in \Lambda$.
\end{lemma}

Now we are in the position to prove Theorem \ref{theorem:lower bound}.

\noindent
{\bf Proof of Theorem \ref{theorem:lower bound}}. Recall that the eigenvalues of $T_*$ denoted by $\{\mu_k\}_{k\geq 1}$ are sorted in decreasing order with geometric multiplicities and satisfy $\mu_k\asymp k^{-1/p}$ for some $0< p \leq 1$, which implies there exists $c>0$ independent of $j$ such that
\begin{equation}\label{eiganvlaue}
	\mu_{k+1}\leq \mu_k \mbox{ and } ck^{-1/p}\leq\mu_k \leq \frac{1}{c} k^{-1/p}, \quad \forall k\geq 1.
\end{equation} We only consider the case that $\epsilon$ is from the Gaussian distribution $N(0,\sigma^2)$ and independent of $X$, then the Assumption \ref{assumption2} is satisfied with $\sigma>0$. 

For $L\geq 2$, we construct $\{\beta_i\}_{i=1}^L$ according to Lemma \ref{lemma: lower bound 2}. Take $M=\lceil a N^{\frac{p}{p+2\theta}} \rceil$, which denotes the smallest integer larger than $a N^{\frac{p}{p+2\theta}}$ for some constant $a>8$ to be specified later. Let $\iota^{(1)},\cdots,\iota^{(L)}\in \{-1,+1\}^M$ be given by Lemma \ref{lemma: lower bound 2} with $L\geq \exp(M/8)$. Given $0<\theta\leq 1/2$, define
\begin{equation}\label{equation: lower bound construction}
	L_C^{1/2}\beta_i = \sum_{k=M+1}^{2M}\frac{1}{\sqrt{M}}\mu_k^\theta \iota^{(i)}_{k-M} \varphi_k = T_*^\theta(\gamma_i),\quad i=1,\cdots,L,
\end{equation}
where $\{\varphi_k\}_{k\geq 1}$ are the eigenvectors (corresponding to eigenvalue $\mu_k$) of $T_*$ which constitutes the orthonormal bases of $L^2(\mathcal{T})$, and $\gamma_i=\sum_{k=M+1}^{2M}\frac{1}{\sqrt{M}}\iota^{(i)}_{k-M} \varphi_k$ satisfies $\|\gamma_i\|_{\cL^2}^2 = 1$. Then $L^{1/2}_C\beta_i \in \mathrm{ran}T^{\theta}_*$ with $0<\theta\leq 1/2$ for $i=1,\cdots,L$.

We next determine the positive constants $r$ and $R$ in \eqref{hypothesis} for $\{\beta_i\}_{i=1}^L$ defined above. For $1 \leq i,j\leq L$, we apply Lemma \ref{hypothesis} and \eqref{eiganvlaue} to obtain 
\[\begin{aligned}
	\left\|L_C^{1/2}(\beta_i - \beta_j)\right\|_{{\cal L}^2}^2 &= \sum_{k=M+1}^{2M}\frac{1}{M}\mu_k^{2\theta}\left(\iota^{(i)}_{k-M} - \iota^{(j)}_{k-M}\right)^2\\
	&\geq \mu_{2M}^{2\theta}\frac{4}{M}\sum_{k=M+1}^{2M}\mathbb{I}_{\{\iota^{(i)}_{k-M} \neq \iota^{(j)}_{k-M}\}}\\
	&\geq \mu_{2M}^{2\theta}\frac{4}{M}\frac{M}{4}\geq c^{2\theta} 2^{-\frac{2\theta}{p}}M^{-\frac{2\theta}{p}},
\end{aligned}\] where the last two inequalities are from \eqref{eiganvlaue}. Therefore, we can take $r=\frac{1}{2}\sqrt{c^{2\theta}2^{-\frac{2\theta}{p}}M^{-\frac{2\theta}{p}}}$. To determine $R$, we turn to bound $\mathcal{D}_{kl} (P_i \| P_j)$ where $\{P_i\}_{i=1}^L$ are the joint probability distributions of $(X,Y)$ with $Y=\langle X,\beta_i\rangle_{\cL^2}+\epsilon$ and $\epsilon \thicksim N(0,\sigma^2)$. Then, using lemma \ref{lemma: lower bound 3} and \eqref{equation: lower bound construction} yields
\begin{align}\nonumber
	\mathcal{D}_{kl} (P_i \| P_j)
	&=\frac{1}{2\sigma^2}\left\|L_C^{1/2}\left(\beta_i - \beta_j\right)\right\|_{{\cal L}^2}^2\nonumber\\
	&= \frac{1}{2\sigma^2}\sum_{k=M+1}^{2M}\frac{1}{M}\mu_k^{2\theta}(\iota^{(i)}_{k-M} - \iota^{(j)}_{k-M})^2\nonumber\\
	&\leq \frac{2}{\sigma^2}\mu_{M}^{2\theta}\leq \frac{2}{\sigma^2c^{2\theta}}M^{-\frac{2\theta}{p}},\nonumber
\end{align} where the last two inequalities are also due to \eqref{eiganvlaue}. Thus, we can take $R=\frac{2}{\sigma^2c^{2\theta}}M^{-\frac{2\theta}{p}}$.

Finally, let $r=\frac{1}{2}\sqrt{c^{2\theta}2^{-\frac{2\theta}{p}}M^{-\frac{2\theta}{p}}}$, $R=\frac{2}{\sigma^2c^{2\theta}}M^{-\frac{2\theta}{p}}$ in Lemma \ref{lemma: lower bound 1} with $L\geq \exp{(M/8)}$ and $M=\lceil a N^{\frac{p}{p+2\theta}} \rceil$. Then there holds
\[\begin{aligned}
	&\inf_{\hat{\beta}_S}\sup_{\beta_0}\mathbb{P}\left\{\mathcal{R}(\hat{\beta}_S)- \mathcal{R}(\beta_0)\geq  \frac{c^{2\theta}}{4}2^{-\frac{2\theta}{p}}a^{-\frac{2\theta}{p}}N^{-\frac{2\theta}{p+2\theta}}\right\}\\  
	&\geq 1-\frac{\frac{2N}{\sigma^2c^{2\theta}}M^{-\frac{2\theta}{p}}+\log 2}{M/8}\geq 1- \frac{\frac{2N}{\sigma^2c^{2\theta}}M^{-\frac{2\theta}{p}} + \log2}{M/8}\\
	&= 1- \frac{16}{\sigma^2c^{2\theta}}N M^{-\frac{2\theta + p}{p}} -  \frac{8\log2}{M}\\
	&\geq 1 - a^{-\frac{2\theta + p}{p}}\frac{16}{\sigma^2c^{2\theta}} N^{1-\frac{p}{2\theta + p}\cdot\frac{2\theta + p}{p}} - \frac{8\log2}{a N^{\frac{p}{p+2\theta}}}\\
	&= 1- a^{-\frac{2\theta + p}{p}}\frac{16}{\sigma^2c^{2\theta}} - \frac{8\log2}{a}N^{-\frac{p}{p+2\theta}}.
\end{aligned}\]
Therefore, we have
\[\mathop{\inf\lim}_{N\to \infty} \inf_{\hat{\beta}_{S}} \sup_{\beta_0} \mathbb{P} \left\{\mathcal{R}(\hat{\beta}_S)-\mathcal{R}(\beta_0)\geq \frac{c^{2\theta}}{4}2^{-\frac{2\theta}{p}}a^{-\frac{2\theta}{p}}N^{-\frac{2\theta}{p+2\theta}}\right\} = 1- a^{-\frac{2\theta + p}{p}}\frac{16}{\sigma^2c^{2\theta}}\]
and then
\[\lim_{a\to \infty}\mathop{\inf\lim}_{N\to \infty} \inf_{\hat{\beta}_{S}} \sup_{\beta_0} \mathbb{P} \left\{\mathcal{R}(\hat{\beta}_S)-\mathcal{R}(\beta_0)\geq \frac{c^{2\theta}}{4}2^{-\frac{2\theta}{p}}a^{-\frac{2\theta}{p}}N^{-\frac{2\theta}{p+2\theta}}\right\} = 1.\]
Taking $\gamma = \frac{c^{2\theta}}{4}2^{-\frac{2\theta}{p}}a^{-\frac{2\theta}{p}}$, we have
\[	\lim_{\gamma \to 0}\mathop{\inf\lim}_{N\to \infty} \inf_{\hat{\beta}_{S}} \sup_{\beta_0} \mathbb{P} \left\{\mathcal{R}(\hat{\beta}_S)-\mathcal{R}(\beta_0)\geq \gamma N^{-\frac{2\theta}{2\theta+p}}\right\} = 1.\]
This completes the proof of Theorem \ref{theorem:lower bound}.
\qed

\appendix

\section*{Appendix}\label{section:Appendix}

The lemma below provides Bernstein's inequality for the sum of self-adjoint random operators on a Hilbert space. The proof of this lemma is given in \cite{minsker2017some}.
\begin{lemma}\label{lemma: Bernstein's inequality for self-adjoint r.o.s}
	Consider a finite sequence $\{\eta_i\}_{i\geq 1}$ of independent random self-adjoint operators on a separable Hilbert space $H$. Assume that
	\[E[\eta_i] = 0 \quad and \quad ||\eta_i||\leq L \quad for \quad each \quad i\]
	Define the random operator $\overline{\eta}:= \sum_{i\geq 1}\eta_i$. Suppose there are constant $v,d>0$ such that $\left\|E[\overline{\eta}^2]\right\|\leq v$ and $trace\left(E[\overline{\eta}^2]\right)\leq vd$. Then for all $s\geq 0$,
	\[\PP(||\overline{\eta}||\geq s)\leq \left[1+6\left(\frac{v}{s^2}+\frac{L}{3s}\right)^2\right]d \exp\left(-\frac{s^2}{2(v+Ls/3)}\right)\]
\end{lemma}

We next give the proof of Lemma \ref{lemma: in comparison}.

\noindent
{\bf Proof of Lemma \ref{lemma: in comparison}}. Using the Courant-Fischer mini-max principle Theorem (see, for example, Theorem 4.2.7 in \cite{hsing2015theoretical}), there holds
\begin{align}\nonumber
	\rho_k(L_A^{1/2}L_BL_A^{1/2})& = \max_{v_1,\cdots,v_k\in \mathcal{H}}\min_{v\in span\{v_1,\cdots,v_k\}} \frac{\langle L_A^{1/2}L_BL_A^{1/2} v,v \rangle_\mathcal{H}}{\|v\|^2_\mathcal{H}}\nonumber\\
	&= \max_{v_1,\cdots,v_k\in \mathcal{H}}\min_{v\in span\{v_1,\cdots,v_k\}}\frac{\langle L_BL_A^{1/2} v,L_A^{1/2}v \rangle_\mathcal{H}}{\|L_A^{1/2}v\|^2_\mathcal{H}}\frac{\|L_A^{1/2}v\|_\mathcal{H}^2}{\|v\|^2_\mathcal{H}}\nonumber\\
	&\leq \max_{L_A^{1/2}v_1,\cdots,L_A^{1/2}v_k\in \mathcal{H}}\min_{L_A^{1/2}v\in span\{L_A^{1/2}v_1,\cdots,L_A^{1/2}v_k\}}\frac{\langle L_BL_A^{1/2} v,L_A^{1/2}v \rangle_\mathcal{H}}{\|L_A^{1/2}v\|^2_\mathcal{H}}\|L_A\|\nonumber\\
	&\overset{(\dagger)}{\leq} \max_{e_1,\cdots,e_k\in \mathcal{H}}\min_{e\in span\{e_1,\cdots,e_k\}}\frac{\langle L_Be,e \rangle_\mathcal{H}}{\|e\|^2_\mathcal{H}}\|L_A\|= \rho_k(L_B)\|L_A\|,\nonumber
\end{align}
where $\{v_1,\cdots,v_k\}$ and $\{e_1,\cdots,e_k\}$ are two groups of $k$ linearly independent elements in $\mathcal{H}$. Inequality $(\dagger)$ uses the fact that  $L_A^{1/2}v_1,\cdots,L_A^{1/2}v_k$ are linearly independent which is deduced from the assumption $\overline{\mathrm{ran}(L_A^{1/2})}= H$. The proof is then finished.
\qed

\end{document}